\documentclass{article}

\usepackage[table,dvipsnames]{xcolor}

\usepackage[final,nonatbib]{neurips_2020}
\usepackage[numbers]{natbib}

\usepackage[utf8]{inputenc} %
\usepackage[T1]{fontenc}    %
\usepackage{url}            %
\usepackage{microtype}
\usepackage{graphicx}
\usepackage{caption}
\usepackage{subcaption} %
\usepackage{longtable}

\usepackage{siunitx}

\newcommand{\acronym}{aPU}

\newcommand{\sourceCode}{https://github.com/ZaydH/arbitrary\_pu}

\newcommand{\sci}[2]{${{#1}\text{E}{#2}}$}
\newcommand{\oneP}{1\phantom{.0}}

\newcommand{\PNarrow}[1]{#1~~}
\newcommand{\PNdown}{\PNarrow{$\downarrow$}}
\newcommand{\PNup}{\PNarrow{$\uparrow$}}

\newcommand{\NResSS}[2]{\underline{\NRes{#1}{#2}}}  %
\newcommand{\NResTop}[2]{\underline{\textbf{\NRes{#1}{#2}}}}
\newcommand{\NResT}[2]{\textbf{\NRes{#1}{#2}}}
\newcommand{\NDef}[1]{\multirow{3}{*}{\shortstack[l]{#1}}}

\newcommand{\params}{\theta}
\newcommand{\learner}{\mathcal{A}}

\newcommand{\grad}{\nabla_{\params}}
\newcommand{\learningRate}{\eta}
\newcommand{\weightDecay}{\lambda}

\newcommand{\eye}[1]{\mathbf{I}_{#1}}

\newcommand{\eqsmall}[1]{{\small #1}}

\newcommand{\gradAtten}{\gamma}
\newcommand{\pnutrade}{\rho}

\newcommand{\W}{w}
\newcommand{\Wx}{{\W(\X)}}

\newcommand{\PUc}{PUc}
\newcommand{\pusb}{PUSB}

\newcommand{\Opt}{\textbf{*}}

\newcommand{\bpu}{bPU}

\newcommand{\nnabs}{abs\=/}
\newcommand{\absPU}{\nnabs{}PU}

\newcommand{\nnPU}{nnPU}

\newcommand{\nnpurpl}{PURR}
\newcommand{\upurpl}{u\nnpurpl}

\newcommand{\wuu}{wUU}
\newcommand{\putwo}{PU2}
\newcommand{\punu}{\putwo \wuu}
\newcommand{\nnpuOpt}{\nnPU\Opt}
\newcommand{\nnpuTE}{\nnPU\textsubscript{\te}}
\newcommand{\nnpuAll}{\nnPU\textsubscript{\te $\cup$ \tr}}
\newcommand{\PNte}{PN\textsubscript{\te}}
\newcommand{\PNtr}{PN\textsubscript{\tr}}

\newcommand{\abspuOpt}{\absPU\Opt}

\newcommand{\pnu}{PNU}
\newcommand{\bpnu}{a\pnu}
\newcommand{\pupnu}{\putwo\bpnu}

\newcommand{\gsub}{\mathcal{D}}  %
\newcommand{\psub}{\textnormal{p}}
\newcommand{\nsub}{\textnormal{n}}
\newcommand{\usub}{\textnormal{u}}

\newcommand{\domSub}{\mathcal{T}}  %

\newcommand{\nuSub}{\nsub\textnormal{-}\usub}

\newcommand{\tr}{\textnormal{tr}}
\newcommand{\te}{\textnormal{te}}
\newcommand{\trsub}[1]{\tr\textnormal{-}#1}
\newcommand{\tesub}[1]{\te\textnormal{-}#1}

\newcommand{\batch}[1]{#1^{(i)}}

\newcommand{\nnwrap}[1]{\abs{#1}}

\newcommand{\train}{\mathcal{X}}
\newcommand{\trainbase}[2]{\train_{\ifempty{#1}{}{#1\textnormal{-}}#2}}
\newcommand{\uset}[1]{\trainbase{#1}{\usub}}
\newcommand{\usetTE}{\train_{\te\textnormal{-}\usub}}  %
\newcommand{\pset}[1]{\trainbase{#1}{\psub}}
\newcommand{\psetHeading}{\train_{\psub}}
\newcommand{\nset}[1]{\trainbase{#1}{\nsub}}
\newcommand{\nsetPS}[1]{\widetilde{\train}_{\ifempty{#1}{}{#1\textnormal{-}}\nsub}}

\newcommand{\sample}{\mathcal{S}}

\newcommand{\dimX}{d}
\newcommand{\xrv}{X}
\newcommand{\domainX}{\real^{\dimX}}  %
\newcommand{\yrv}{Y}
\newcommand{\domainY}{{\set{\pm1}}}   %
\newcommand{\srv}{S}
\newcommand{\domainS}{\domainY}       %
\newcommand{\X}{x}
\newcommand{\xbar}{\bar{\X}}

\newcommand{\y}{y}
\newcommand{\yn}{{{-}1}}
\newcommand{\yp}{{{+}1}}
\newcommand{\yhat}{\hat{y}}

\newcommand{\sizebase}[1]{n_{#1}}
\newcommand{\np}{\sizebase{\psub}}
\newcommand{\nn}{\sizebase{\nsub}}

\newcommand{\nU}{\sizebase{\usub}}
\newcommand{\nTrU}{\sizebase{\trsub\usub}}
\newcommand{\nTeU}{\sizebase{\tesub\usub}}

\newcommand{\nTest}{\sizebase{\textnormal{Test}}}

\newcommand{\cmed}{c_{\textnormal{med}}}
\newcommand{\trainLo}{\train_{\textnormal{lo}}}
\newcommand{\trainHi}{\train_{\textnormal{hi}}}

\newcommand{\dec}{g}

\newcommand{\decx}{\dec(\X)}
\newcommand{\decxI}{\dec(\X_{i})}
\newcommand{\decxIsup}[1]{\dec(\X_{i}^{#1})}
\newcommand{\decRV}{\dec(\xrv)}
\newcommand{\decRVI}{\dec(\xrv_i)}

\newcommand{\decCls}{\mathcal{\MakeUppercase{\dec}}}

\newcommand{\cdcBaseSymbol}[1]{#1{\sigma}}
\newcommand{\cdcCls}{\widehat{\mathit{\Sigma}}}
\newcommand{\cdcBase}[4]{{#1\ifempty{#2}{}{_{#2}}(#3\ifempty{#4}{}{_{#4}})}}

\newcommand{\cdc}{\cdcBaseSymbol{\hat}}
\newcommand{\cdcx}{\cdcBase{\cdc}{}{\X}{}}
\newcommand{\cdcxI}{\cdcBase{\cdc}{}{\X}{i}}

\newcommand{\cdcxAlt}[1]{\cdcBase{\cdc}{#1}{\X}{}}
\newcommand{\cdcxSoft}{\cdcxAlt{\text{soft}}}
\newcommand{\cdcxHard}{\cdcxAlt{\text{hard}}}
\newcommand{\cdcxTopK}{\cdcxAlt{\text{top-k}}}

\newcommand{\cdcRV}{\cdcBase{\cdc}{}{\xrv}{}}
\newcommand{\cdcRVI}{\cdcBase{\cdc}{}{\xrv_{i}}{}}

\newcommand{\var}[1]{{\text{Var}(#1)}}

\newcommand{\loss}{\ell}
\newcommand{\lbase}[2]{{\loss \mathopen{} \left(#2 #1\right) \mathclose{}}}

\newcommand{\ldecRV}[1]{\lbase{\decRV}{#1}}

\newcommand{\ldec}[1]{\lbase{\decx}{#1}}
\newcommand{\ldecI}[1]{\lbase{\decxI}{#1}}

\newcommand{\normal}[2]{\mathcal{N}\mathopen{}\left(#1,#2\right)\mathclose{}}

\newcommand{\p}{p}
\newcommand{\pbase}[2]{{\p\ifempty{#1}{}{_{#1}}(#2)}}
\newcommand{\pbaseX}[1]{\pbase{#1}{\X}}

\newcommand{\pr}{\pi}
\newcommand{\prd}[1]{\pr_{#1}}
\newcommand{\pprior}[1]{\pbase{#1}{\yrv\!=\!\yp}}
\newcommand{\nprior}[1]{\pbase{#1}{\yrv\!=\!\yn}}

\newcommand{\marg}[1]{\pbase{\ifempty{#1}{}{#1\textnormal{-}}\usub}{\X}}
\newcommand{\condfull}[2]{\pbase{#1}{\X \vert \yrv\!=\!#2}}  %

\newcommand{\ppostGeneral}[1]{\pbase{#1}{\y \vert \X}}
\newcommand{\ppost}[2]{\pbase{#1}{\yrv\!=\!#2 \vert \X}}
\newcommand{\ppostRV}[2]{\pbase{#1}{\yrv\!=\!#2 \vert \xrv}}

\newcommand{\likeBaseFull}[2]{\pbase{#1}{\X \vert \yrv\!=\!#2}}
\newcommand{\likebase}[2]{\pbase{\ifempty{#1}{}{#1\textnormal{-}}#2}{\X}}
\newcommand{\plike}[1]{\likebase{#1}{\psub}}
\newcommand{\nlike}[1]{\likebase{#1}{\nsub}}

\newcommand{\pbaseNIPS}[1]{{\p_{#1}\mathopen{}\left(\X\right)\mathclose{}}}
\newcommand{\plikeTR}{\pbaseNIPS{\tr\textnormal{-}\psub}}
\newcommand{\nlikeTR}{\pbaseNIPS{\tr\textnormal{-}\nsub}}

\newcommand{\pjointbase}[1]{{\p_{#1}(\X, \y)}}
\newcommand{\pjointsimp}{\pjointbase{}}

\newcommand{\risk}[1]{R\ifempty{#1}{}{_{#1}}(\dec)}
\newcommand{\lblRiskComplete}[3]{{#1{R}_{#2}^{#3}(\dec)}}  %

\newcommand{\lblrisk}[2]{\lblRiskComplete{}{#1}{#2}}

\newcommand{\ersk}[1]{\widehat{R}\ifempty{#1}{}{_{#1}}(\dec)}
\newcommand{\elblrsk}[2]{\lblRiskComplete{\widehat}{#1}{#2}}

\newcommand{\lblAbsR}[2]{\ddot{R}\ifempty{#1}{}{_{#1}}\ifempty{#2}{}{^{#2}}(\dec)}
\newcommand{\trsk}[2]{\lblRiskComplete{\widetilde}{#1}{#2}}

\newcommand{\rskWUU}{\ersk{\textnormal{\wuu}}}
\newcommand{\rskNnWUU}{\ersk{\textnormal{nn-\wuu}}}
\newcommand{\rskBPNU}{\ersk{\textnormal{\bpnu}}}
\newcommand{\rskNnBPNU}{\ersk{\textnormal{nn-\bpnu}}}

\newcommand{\rskPURPL}{\ersk{\textnormal{\nnpurpl}}}
\newcommand{\rskNnPURPL}{\ersk{\textnormal{nn-\nnpurpl}}}

\newcommand{\rskUPURPL}{\ersk{\textnormal{\upurpl}}}

\newcommand{\SynVect}[2]{\begin{bmatrix} #1 & #2 \end{bmatrix}}
\newcommand{\SynFrac}[2]{\frac{#1}{#2}}
\newcommand{\SynVariance}{\eye{2}}

\newcommand{\SynFigHeight}{3.55cm}
\newcommand{\SynFigWidthSep}{4.45cm}
\newcommand{\SynFigWidthNonsep}{5.4cm}
\newcommand{\SynFigWidthSameLike}{4.05cm}

\newcommand{\SynXMinSep}{-5}
\newcommand{\SynXMinNonsep}{-9}
\newcommand{\SynXMax}{9}
\newcommand{\SynXMaxSameWidth}{5}
\newcommand{\SynYMin}{-4}
\newcommand{\SynYMax}{4}

\newcommand{\SynPBase}[1]{\plike{\tr\text{-(} #1 \text{)}}}
\newcommand{\SynPLike}[1]{\SynPBase{\text{\subref{fig:App:AddRes:#1}}}}

\newcommand{\SynLineWidth}{3}

\newcommand{\pDsTrain}{P\textsubscript{train}}
\newcommand{\pDsTest}{P\textsubscript{test}}

\newcommand{\NRes}[2]{#1 (#2)}
\newcommand{\NResMain}[2]{#1}
\newcommand{\NResPU}[2]{#1 (#2)}
\newcommand{\NResPUT}[2]{\textbf{\NResPU{#1}{#2}}}
\newcommand{\NResPUDiff}[4]{\ifempty{#1}{\phantom{--}}{--}\ifempty{#2}{0\phantom{.0}}{#2} (\ifempty{#3}{\phantom{--}}{--}\ifempty{#4}{0\phantom{.0}}{#4})}

\newcommand{\RETURN}{\textbf{return}~}

\newcommand{\UnlabeledDist}{\usub}
\newcommand{\CallDist}{A}
\newcommand{\CallComplement}{B}
\newcommand{\NewPrior}{\alpha}

\newcommand{\lossBound}{C_{\loss}}
\newcommand{\decBound}{C_{\dec}}
\newcommand{\varBound}{C_{\text{var}}}

\newcommand{\err}{\text{err}}

\newcommand{\MyMax}[1]{\sbrack{#1}_{+}}

\newcommand{\PTestDef}[1]{\NDef{#1}}

\newcommand{\SamePTest}{\pDsTest}

\newcommand{\TwoStepDim}{4.5cm}

\newcommand{\neuripsParagraph}[1]{\vspace{1pt}\noindent\textbf{#1}~~}

\newcommand{\ourMethodKey}{\textsuperscript{\textdagger}}
\newcommand{\ourMethodText}{~(ours)}

\newcommand{\BarLineWidth}{0.4pt}
\newcommand{\BarWidth}{3.1pt}
\newcommand{\NonoverlapYMax}{45}

\newcommand{\PUcBarWidth}{5pt}
\newcommand{\PUcYMax}{55}

\newcommand{\SpamYMax}{45}

\usepackage{pgfplots}
\pgfplotsset{compat=1.13,
  /pgfplots/ybar legend/.style={
    /pgfplots/legend image code/.code={%
       \draw[##1,/tikz/.cd,bar width=6pt,yshift=-0.2em,bar shift=0pt]
       plot coordinates {(0cm,0.6em)};
    },
  },
}

\usepackage{pgf}
\usepackage{collcell}
\usepackage{booktabs}

\newcommand*{\MinNumber}{0}%
\newcommand*{\MaxNumber}{40}%

\newcommand{\ApplyGradient}[1]{%
  \pgfmathsetmacro{\PercentColor}{100.0*(#1-\MinNumber)/(\MaxNumber-\MinNumber)}%
  \edef\x{\noexpand\cellcolor{red!\PercentColor}}\x\textcolor{black}{#1}%
}
\newcolumntype{R}{>{\collectcell\ApplyGradient}{r}<{\endcollectcell}}



\newcommand{\KSym}{\text{k}}
\newcommand{\K}{\=/k}
\newcommand{\usetTrK}{\trainbase{\trsub\usub}{\KSym}}

\newcommand{\TSAlgHeader}[1]{\multicolumn{3}{c}{#1}}
\newcommand{\TSResHeader}{Soft & Diff.\ Hard & Diff.\ Top\K{}}
\newcommand{\TSResHeaderAlt}{Soft & Hard & Top\K{}}

\newcommand{\ResTSB}[2]{\textbf{#1}~(\textbf{#2})}
\newcommand{\ResTSN}[2]{#1~(#2)}
\newcommand{\ResTSDBB}[2]{{\color{ForestGreen} --#1~({--}#2)}}
\newcommand{\ResTSDBW}[2]{{\color{ForestGreen} --#1~(\phantom{--}#2)}}
\newcommand{\ResTSDWW}[2]{{\color{BrickRed} #1~(\phantom{--}#2)}}
\newcommand{\ResTSDWB}[2]{{\color{BrickRed} #1~({--}#2)}}
\newcommand{\ZR}{0\phantom{.0}}
\newcommand{\ResTSDSW}[1]{{\ZR~(\phantom{--}#1)}}
\newcommand{\ResTSDSB}[1]{{\ZR~(--#1)}}

\newcommand{\ResTSAcc}[2]{#1~(#2)}

\newcommand{\abridge}[1]{ }

\newcommand{\SupplementaryMaterialsTitle}{%
  \vbox{
    \hrule height 4pt
    \vskip 0.05in
    \vskip -\parskip%
    \begin{center}
      {\LARGE\bf \titletext{} \par}

      \vspace{8pt}
      {\LARGE Supplemental Materials \par}
    \end{center}
    \vskip 0.15in
    \vskip -\parskip
    \hrule height 1pt
    \vskip 0.09in%
  }
}
\usepackage{standalone}

\usepackage[shortcuts]{extdash}

\newcommand{\colortext}[2]{{\color{#1} #2}}

\newcommand{\blue}[1]{\colortext{blue}{#1}}

\usepackage{amsmath}

\usepackage{amsfonts}  %
\usepackage{amssymb}

\usepackage{mathtools} %
\usepackage{etoolbox}
\newcommand\swapifbranches[3]{#1{#3}{#2}}
\makeatletter
\MHInternalSyntaxOn
\patchcmd{\DeclarePairedDelimiter}{\@ifstar}{\swapifbranches\@ifstar}{}{}
\MHInternalSyntaxOff
\makeatother
\DeclarePairedDelimiter{\sbrack}{\lbrack}{\rbrack}

\DeclarePairedDelimiter{\abs}{\lvert}{\rvert}
\DeclarePairedDelimiter{\round}{\lfloor}{\rceil}
\DeclarePairedDelimiter{\norm}{\lVert}{\rVert}
\usepackage{bm}
\DeclarePairedDelimiterX\set[1]\lbrace\rbrace{#1}
\DeclarePairedDelimiterX\setbuild[2]\lbrace\rbrace{#1 \bm: #2}

\newcommand{\func}[3]{{#1:#2\rightarrow#3}}
\newcommand{\defeq}{\coloneqq}
\newcommand{\fedeq}{\eqqcolon}

\newcommand{\expectS}[2]{\mathbb{E}_{#1}\sbrack{#2}}

\newcommand{\intsNN}{\mathbb{Z}_{+}}

\newcommand{\real}{\mathbb{R}}
\newcommand{\realnn}{\real_{{\geq}0}}  %

\newcommand{\iidsim}{\stackrel{\mathclap{\mbox{\tiny \textnormal{i.i.d.}}}}{\sim}}

\usepackage{array}  %
\usepackage{arydshln}  %
\usepackage{bigdelim}
\usepackage{booktabs}
\usepackage{multirow}
\usepackage{makecell}  %

\usepackage{amsthm}
\newtheorem{theorem}{Theorem}
\usepackage{tikz}
\usetikzlibrary{arrows,decorations.markings,shadows,positioning,calc,backgrounds,shapes}

\usepackage{pgfplots}
\pgfplotsset{compat=1.13}
\usepackage{pgfplotstable}
\usepackage{xifthen}
\newcommand{\ifempty}[3]{%
  \ifthenelse{\isempty{#1}}{#2}{#3}%
}
\newcommand{\addShiftLinePlot}[6]{%
\addplot[%
  line width=1.0,
  legend style={font=\scriptsize},
  color={#3},
  mark=#4,
  #5,
  mark options={fill=#3,scale=0.8,solid,line width=0.5},
]
table[x={x},y={#2}] {\thedata};
\ifx\relax#6\relax\addlegendentry{#1}\else\fi
}

\pgfplotsset{ignore legend/.style={every axis legend/.code={\let\addlegendentry\relax}}}
\newcommand{\ShiftLegendOff}{ignore legend,}

\newcommand{\shiftPlot}[5]{%
\begin{tikzpicture}
    \pgfplotstableread[col sep=comma] {plots/data/#1.csv}{\thedata}
    \begin{axis}
        [
         xmin=0.5,
         xmax=1.0,
         ymin=0,
         ymax=#4,
         ymajorgrids,  %
         xlabel={${\Pr\sbrack{\text{#2} \vert \yrv = #3}}$},
         ylabel={Misclassification Rate},
         width=7.6cm,
         height=5cm,
         axis x line*=bottom,  %
         axis y line*=left,    %
         xtick pos=left, %
         xtick align=outside,
         ytick pos=left, %
         xtick distance = {0.1},
         ytick distance = {10},
         minor tick num = {3},
         legend pos=outer north east,
         legend columns=1,
         legend cell align={left},              %
         legend style={font=\footnotesize},
         label style={font=\tiny},
         tick label style={font=\tiny},
        ]
        \addShiftLinePlot{\nnpurpl\ourMethodText{}}{purpl}{Orange}{pentagon*}{}{#5}
        \addShiftLinePlot{\pupnu\ourMethodText{}}{pnu}{Cyan}{o}{}{#5}
        \addShiftLinePlot{\punu\ourMethodText{}}{wuu}{NavyBlue}{square}{}{#5}
        \addShiftLinePlot{\PUc{}}{puc}{red}{diamond*}{dotted}{#5}
        \addShiftLinePlot{\nnpuOpt{}}{nnpu}{ForestGreen}{triangle*}{dashdotted}{#5}
        \addShiftLinePlot{\PNte{}}{pnte}{blue}{square*}{}{#5}
        \addShiftLinePlot{\PNtr{}}{pntr}{Magenta}{*}{dashed}{#5}
        \ShiftLegendOff%
    \end{axis}%
\end{tikzpicture}%
}
\newcommand{\disableTicks}{yticklabels={,,},}
\newcommand{\BarGraphBase}[4]{%
  \pgfplotstableread[col sep=comma]{plots/data/#2.csv}\datatable%
  \begin{tikzpicture}%
    \begin{axis}[%
        axis lines*=left,%
        ymajorgrids,  %
        bar width=\BarWidth,%
        height=3.6cm,%
        width=4.5cm,%
        ymin=0,%
        ymax=\NonoverlapYMax,%
        ybar={\BarLineWidth},%
        ytick distance={10},%
        minor y tick num={3},%
        y tick label style={font=\scriptsize},%
        #4%
        ylabel={\scriptsize #3},
        ylabel shift={-3pt},  %
        title={\footnotesize #1},%
        title style={yshift=-4pt,},
        xtick=data,%
        x tick label style={font=\scriptsize,align=center},%
        every tick/.style={color=black, line width=\BarLineWidth},%
        xticklabels from table={\datatable}{Dataset},%
        enlarge x limits=0.23,%
        legend pos=outer north east,%
        legend cell align={left},              %
        legend style={font=\small},
        title style={text depth=0.5ex}       %
    ]%
    \addplot table [x expr=\coordindex, y index=1] {\datatable};%
    \addplot table [x expr=\coordindex, y index=2] {\datatable};%
    \addplot table [x expr=\coordindex, y index=3] {\datatable};%
    \addplot table [x expr=\coordindex, y index=4] {\datatable};%
    \addplot table [x expr=\coordindex, y index=5] {\datatable};%
    #4%
    \end{axis}%
  \end{tikzpicture}%
}
\newcommand{\PriorTableCaptionText}[1]{Combined heat map and table showing the effect of incorrectly specified priors~$\prd{\tr}$ and~$\prd{\te}$ on~\underline{#1}'s inductive misclassification rate~(\%).  Each result is the average of 100~trials.}

\usepackage{hyperref}

\usepackage{algorithm}
\usepackage[noend]{algorithmic}

\usetikzlibrary{external}

\newcommand{\titletext}{Learning from Positive and Unlabeled Data \\ with Arbitrary Positive Shift}
\title{\titletext}

\author{%
  Zayd Hammoudeh \quad\quad Daniel Lowd \\
  Department of Computer \& Information Science \\
  University of Oregon \\
  Eugene,~OR, USA \\
  \texttt{\{zayd,~lowd\}@cs.uoregon.edu}
}

\begin{document}

\maketitle

\begin{abstract}
\textit{Positive\=/unlabeled}~(PU) \textit{learning} trains a binary classifier using only positive and unlabeled data. A common simplifying assumption is that the positive data is representative of the target positive class. This assumption rarely holds in practice due to temporal drift, domain shift, and/or adversarial manipulation.  This paper shows that PU~learning is possible even with arbitrarily non\=/representative positive data given unlabeled data from the source and target distributions. Our key insight is that only the negative class's distribution need be fixed.  We integrate this into two statistically consistent methods to address arbitrary positive bias -- one approach combines \textit{negative\=/unlabeled learning} with \textit{unlabeled\=/unlabeled learning} while the other uses a novel, recursive risk estimator. Experimental results demonstrate our methods' effectiveness across numerous real-world datasets and forms of positive bias, including disjoint positive class-conditional supports. Additionally, we propose a general, simplified approach to address PU~risk estimation overfitting.
 \end{abstract}

\section{Introduction}\label{sec:Introduction}

\emph{Positive\=/negative}~(PN) \emph{learning} (i.e.,~ordinary supervised classification) trains a binary classifier using positive and negative labeled datasets.  In practice, good labeled data are often unavailable for one class. High negative-class diversity may make constructing a representative labeled set prohibitively difficult~\citep{duPlessis:2015}, or negative data may not be systematically recorded in some domains~\citep{Bekker:2018:Prior}.

\emph{Positive\=/unlabeled}~(PU) \emph{learning} addresses this problem by constructing classifiers using only labeled\=/positive and unlabeled data.  PU~learning has been applied to numerous real-world domains including: opinion spam detection~\citep{Hernandez:2013}, disease-gene identification~\citep{Yang:2012}, land-cover classification~\citep{Li:2011}, and protein similarity prediction~\citep{Elkan:2008}.  The related task of \emph{negative-unlabeled} (NU)~learning is functionally identical to PU~learning but with labeled data drawn from the negative class.

Most PU~learning methods assume the labeled set is \emph{selected completely at random}~(SCAR) from the target distribution~\citep{duPlessis:2015,Elkan:2008,Hou:2017,Kiryo:2017,Gong:2018,Zhang:2019,Hsieh:2019}. External factors like temporal drift, domain shift, and adversarial concept drift often cause the labeled-positive and target distributions to diverge.

\emph{Biased\=/positive, unlabeled}~(\bpu) \emph{learning} algorithms relax SCAR by modeling \emph{sample selection bias} for the labeled data~\citep{Bekker:2019,Kato:2019} or a \emph{covariate shift} between the training and target distributions~\cite{Sakai:2019}.

This paper generalizes \bpu{}~learning to the more challenging \emph{arbitrary-positive, unlabeled}~(\acronym) learning setting, where the labeled (positive) data may be \emph{arbitrarily different} from the target distribution's positive class.  Solving this problem would eliminate the need to spend time and money labeling new data whenever the positive class drifts.

Devoid of some assumption, \acronym~learning is impossible~\citep{Elkan:2008}.  As a first step to address \acronym~learning, our key insight is that given a labeled-positive set and two unlabeled sets as proposed by \citet{Sakai:2019}, \acronym~learning is possible when \emph{all negative examples are generated from a single distribution}.  The labeled and target-positive distributions' \emph{supports} (sets of examples with non-zero probability) may even be disjoint.  Many real-world PU~learning tasks feature a shifting positive class but (largely) fixed negative class including:
\begin{enumerate}
  \setlength{\itemsep}{0pt}
  \item \textbf{Land-Cover Classification}: Cross-border land-cover datasets often do not exist due to differing national technological standards or insufficient financial resources by one country~\citep{Galen:2016}. This limits research into natural processes at broad geographic scales.  However, cross-border geographic terrains often follow a similar distribution differing primarily in man\=/made objects (e.g.,~roads) due to local construction materials and regulations~\cite{Li:2011}.
    \item \textbf{Adversarial \acronym{}~Learning}: Malicious adversaries (email spammers, malware authors) rapidly adapt their attacks to bypass automated detection.  The benign class changes much more slowly but may be too diverse to construct a representative labeled set~\citep{Hernandez:2013,Li:2014,Zhang:2017,Zhang:2019:Malware}.
\end{enumerate}

Our paper's four primary contributions are enumerated below. Note that most experiments and all proofs are in the supplemental materials.
\begin{enumerate}
  \setlength{\itemsep}{0pt}
  \item We propose~\absPU{} -- a simplified, statistically \emph{consistent} approach to correct general PU~risk estimation overfitting.  Our \acronym~methods leverage \absPU{} to streamline their optimization.
  \item We address our \acronym~learning task via a two-step formulation; the first step applies standard PU~learning and the second uses unlabeled\=/unlabeled~(UU) learning.
  \item We separately propose \nnpurpl~-- a novel, recursive, consistent \acronym~risk estimator.
  \item We evaluate our methods on a wide range of benchmarks, demonstrating our algorithms' effectiveness over the state of the art in PU~and \bpu~learning. Our empirical evaluation includes an adversarial \acronym{}~learning case study using public spam email datasets.
\end{enumerate}
\section{Ordinary Positive-Unlabeled Learning}\label{sec:ProblemSetting}\label{sec:PreviousWork:Vanilla}

We begin with an overview of PU~learning without distributional shifts, including definitions and notation. %
Consider two random variables, covariate ${\xrv \in \domainX}$ and label ${\yrv \in \domainY}$, with joint distribution $\pjointbase{}$. Marginal distribution~$\marg{}$ is composed from the positive prior~${\prd{} \defeq \pprior{}}$, positive class\=/conditional~${\plike{} \defeq \likeBaseFull{}{\yp}}$, and negative class\=/conditional~${\nlike{} \defeq \likeBaseFull{}{\yn}}$.

\neuripsParagraph{Risk} Let $\func{\dec}{\domainX}{\real}$ be an arbitrary \emph{decision function} parameterized by~$\params$, and let $\func{\loss}{\real}{\realnn}$ be the \emph{loss function}.  Risk ${\risk{} \defeq \expectS{(\xrv,\yrv) \sim \pjointbase{}}{\ldecRV{\yrv}}}$ quantifies $\dec$'s~expected loss over~$\pjointbase{}$. It decomposes via the product rule to ${\risk{} =} {\prd{}\lblrisk{\psub}{+}} {+ (1-\prd{})\lblrisk{\nsub}{-}}$, where the \emph{labeled risk} is
\begin{equation}\label{eq:PreviousWork:LabeledRisk}
  \lblrisk{\gsub}{\yhat} \defeq \expectS{\xrv \sim \pbase{\gsub}{\X}}{\ldecRV{\yhat}}
\end{equation}%
\noindent
for predicted label ${\yhat \in \set{\pm1}}$ and ${\gsub \in \set{\psub,\nsub,\usub}}$ denoting the positive class-conditional, negative class-conditional or marginal distribution respectively, as defined above.

Since~$\pjointbase{}$ is unknown, \emph{empirical risk} is used in practice. We consider the \emph{case-control scenario}~\citep{Niu:2016} where each dataset is i.i.d.\ sampled from its associated distribution.  PN~learning has two labeled datasets: positive set \eqsmall{${\pset{} \defeq \set{\X_{i}^{\psub}}_{i = 1}^{\np} \iidsim \plike{}}$} and negative set \eqsmall{${\nset{} \defeq \set{\X_{i}^{\nsub}}_{i=1}^{\nn} \iidsim \nlike{}}$}. These are used to calculate empirical labeled risks \eqsmall{${\elblrsk{\psub}{+} = \frac{1}{\np}\sum_{i = 1}^{\np} \lbase{\decxIsup{\psub}}{}}$} and \eqsmall{${\elblrsk{\nsub}{-} = \frac{1}{\nn} \sum_{i=1}^{\nn} \lbase{\decxIsup{\nsub}}{-}}$}. We denote the empirical positive\=/negative risk
\begin{equation}\label{eq:PreviousWork:PN}
  \ersk{\text{PN}} \defeq \prd{}\elblrsk{\psub}{+} + (1 - \prd{})\elblrsk{\nsub}{-}\text{.}
\end{equation}

PU~learning cannot directly estimate~$\lblrisk{\nsub}{\yhat}$ since there is no negative (labeled) data (i.e.,~${\nset{} = \emptyset}$).  Let \eqsmall{${\uset{} \defeq \set{\X_{i}^{\usub}}_{i=1}^{\nU} \iidsim \marg{}}$} be an unlabeled set with empirical labeled risk \eqsmall{${\elblrsk{\usub}{\yhat} = \frac{1}{\nU}\sum_{i = 1}^{\nU} \lbase{\decxIsup{\usub}}{\yhat}}$}.  \citet{duPlessis:2014} make a foundational contribution that,
\begin{equation}\label{eq:PreviousWork:NegRisk}
  (1-\prd{})\lblrisk{\nsub}{\yhat} = \lblrisk{\usub}{\yhat} - \prd{} \lblrisk{\psub}{\yhat}\text{.}
\end{equation}%
\noindent%
Their unbiased~PU~(uPU) risk estimator is therefore \eqsmall{${\ersk{\text{uPU}} \defeq \prd{}\elblrsk{\psub}{+} + \elblrsk{\usub}{-} - \prd{}\elblrsk{\psub}{-}}$}.
\citet{Kiryo:2017} observe that highly expressive models (e.g.,~neural networks) often overfit~$\pset{}$ causing uPU to estimate that \eqsmall{${\elblrsk{\usub}{-} - \prd{}\elblrsk{\psub}{-} < 0}$}.

Since negative-valued risk is impossible, \citeauthor{Kiryo:2017}'s non\=/negative PU~(nnPU) risk estimator ignores negative estimates of risk via a $\max$~term:
\begin{equation}\label{eq:PreviousWork:nnPU}
  \ersk{\text{nnPU}} \defeq \prd{}\elblrsk{\psub}{+} + \max\{0,\elblrsk{\usub}{-} - \prd{}\elblrsk{\psub}{-}\}\text{.}
\end{equation}
\noindent%
When \citeauthor{Kiryo:2017}'s customized empirical risk minimization~(ERM) framework detects overfitting (i.e.,~\eqsmall{${\elblrsk{\usub}{-} - \prd{}\elblrsk{\psub}{-} < 0}$}), their framework ``defits'' $\dec$ using negated gradient \eqsmall{${-\gradAtten \grad( \elblrsk{\usub}{-} - \prd{} \elblrsk{\psub}{-} )}$}, where hyperparameter ${\gradAtten \in (0,1]}$ attenuates the learning rate to throttle ``defitting.'' Observe that positive-labeled risk,~\eqsmall{$\elblrsk{\psub}{+}$}, is excluded from nnPU's negated gradient.
\section{Simplifying Non-Negativity Correction}\label{sec:AbsValCorrection}

Rather than enforcing the non-negative risk constraint with two combined techniques (a $\max$~term and ``defitting'') like~\citeauthor{Kiryo:2017}, we propose a simpler approach, inspired by Lagrange multipliers, that directly puts the non-negativity constraint into the risk estimator.  Our \textit{absolute\=/value correction},
\begin{equation}\label{eq:AbsVal:Def}
  (1 - \prd{}) \lblAbsR{\nsub}{\yhat} \defeq \big\lvert{\elblrsk{\usub}{\yhat} - \prd{} \elblrsk{\psub}{\yhat}}\big\rvert \text{,}  %
\end{equation}%
\noindent
replaces \nnPU{}'s~$\max$ with absolute value to prevent the optimizer overfitting an implausible risk estimate by explicitly penalizing those risk estimates for being negative. This penalty ``defits'' the learner automatically, eliminating the need for hyperparameter~$\gradAtten$ and \nnPU{}'s custom ERM~algorithm.

  \begin{theorem}\label{thm:AbsCorrection}
    Let $\func{\dec}{\domainX}{\real}$ be an arbitrary decision function and let $\func{\loss}{\real}{\realnn}$ be a loss function bounded\footnote{Each theorem's definition of ``bounded'' loss appears in the associated proof. See the supplemental materials.} w.r.t.~$\dec$ then \eqsmall{$\lblAbsR{\nsub}{\yhat}$} is a consistent estimator of \eqsmall{$\elblrsk{\nsub}{\yhat}$}.
  \end{theorem}

  We integrate absolute value correction into our \textit{\absPU{}~risk estimator},
\begin{equation}\label{eq:AbsPU}
  \ersk{\text{\absPU}} \defeq \prd{}\elblrsk{\psub}{+} + \big\lvert{\elblrsk{\usub}{-} - \prd{} \elblrsk{\psub}{-}}\big\rvert \text{,}  %
\end{equation}%
\noindent
which by Theorem~\ref{thm:AbsCorrection} is consistent like \nnPU{}\@.  When \eqsmall{${\elblrsk{\usub}{-} - \prd{} \elblrsk{\psub}{-} < 0}$}, \absPU{}'s update gradient,
  \eqsmall{${%
      \grad( \prd{}\elblrsk{\psub}{+} - \elblrsk{\usub}{-} + \prd{} \elblrsk{\psub}{-} ) \text{,}
  }$} %
  includes~\eqsmall{$\elblrsk{\psub}{+}$}.  Hence, \absPU{} spends comparatively more time optimizing the positive-labeled risk than \nnPU{}. Also, by penalizing implausible risk, \absPU{} estimates validation performance (i.e.,~risk) differently than \nnPU{}.

  Empirically we observed that \absPU{} yields models of similar or slightly better accuracy than \nnPU{} albeit with a simpler, more efficient optimization. The following builds on \absPU{} with a full comparison to \nnPU{} in supplemental Section~\ref{sec:App:AddRes:AbspuVsNnpu}.
\section{Arbitrary-Positive, Unlabeled Learning}\label{sec:APULearning}

\emph{Arbitrary-positive unlabeled} (\acronym)~learning~--- the focus of this work --- is one of three problem settings proposed by \citet{Sakai:2019}.  We generalize their original definition below.

Consider two joint distributions: train~$\pjointbase{\tr}$ and test~$\pjointbase{\te}$. Notation $\pbaseX{\trsub\gsub}$ where ${\gsub \in \set{\psub,\nsub,\usub}}$ refers to the training positive class-conditional, negative class-conditional, and marginal distributions respectively.  $\pbaseX{\tesub\gsub}$~denotes the corresponding test distributions.

No assumption is made about the label's conditional probability, i.e.,~$\ppostGeneral{\tr}$ and~$\ppostGeneral{\te}$, nor about positive class-conditionals $\plike{\tr}$ and $\plike{\te}$. We only assume a fixed negative class-conditional%
\begin{equation}\label{eq:PreviousWork:NegAssume}
  \nlike{} = \nlike{\tr} = \nlike{\te}\text{.}
\end{equation}%
\noindent

Both the train and test positive\=/class priors,~$\prd{\tr}$ and~$\prd{\te}$ respectively, are treated as known throughout this work. In practice, they may be known \textit{a priori} through domain-specific knowledge. Techniques also exist to estimate them from data~\citep{Bekker:2018:Prior,Ramaswamy:2016,duPlessis:2017,Zeiberg:2020}. Theorem~\ref{thm:App:TwoStep:PriorSample} in the supplemental materials provides an algorithm to estimate~$\prd{\te}$ by training an additional classifier.

As shown in Figure~\ref{fig:TwoStep:aPUExample}, the available datasets are: labeled (positive) set \eqsmall{${\pset{} \iidsim \plike{\tr}}$} as well as unlabeled set\underline{s} \eqsmall{${\uset{\tr} \defeq \set{\X_i}_{i=1}^{\nTrU} \iidsim \marg{\tr}}$} and \eqsmall{${\uset{\te} \defeq \set{\X_i}_{i=1}^{\nTeU} \iidsim \marg{\te}}$} with their empirical risks defined as before.  An optimal classifier minimizes the \emph{test} risk/expected loss: \eqsmall{${\expectS{(\xrv,\yrv) \sim \pjointbase{\te}}{\ldecRV{\yrv}}}$}.

\subsection{Relating \acronym~Learning and Covariate Shift Adaptation Methods}\label{sec:PreviousWork:PUc}

Covariate shift~\citep{Shimodaira:2000} is a common technique to address differences between~$\pjointbase{\tr}$ and~$\pjointbase{\te}$. Unlike \acronym~learning, covariate shift restrictively assumes a \emph{consistent input-output relation}, i.e., ${\ppostGeneral{\tr} = \ppostGeneral{\te}}$. Define the \emph{importance function} as ${\Wx \defeq \frac{\marg{\te}}{\marg{\tr}}}$. When $\ppostGeneral{}$~is fixed, it is easy to show that {${\Wx \pjointbase{\tr} = \pjointbase{\te}}$}.

\citet{Sakai:2019} exploit this relationship in their \PUc~risk estimator.  $\Wx$~is approximated via direct density\=/ratio estimation~\citep{Sugiyama:2012} -- specifically the RuLSIF~algorithm \citep{Yamada:2013} over $\uset{\tr}$ and~$\uset{\te}$.  Their \PUc{}~risk %
adds importance weighting to~uPU\@, with the labeled risks still estimated from~$\pset{}$ and~$\uset{\tr}$. \citeauthor{Sakai:2019}'s formulation specifies linear-in-parameter models to enforce convexity. They improve tractability via a simplified version of \citet{duPlessis:2015}'s surrogate squared loss for~$\loss$.

Selection bias \bpu{}~methods~\citep{Bekker:2019,Kato:2019} need the positive\=/labeled data to meet specific conditions that arbitrary-positive data will not satisfy making a comparison to those methods infeasible. \PUc{}~serves as the primary baseline here since as a covariate shift \bpu~method, it places no requirements on the positive data beyond that the training distribution's support be a superset of the target positive class.

\subsection{Comparing Variations of the \acronym~Learning Problem}\label{sec:aPU:SimilarProblems}

\citet{Sakai:2019}~show that PU~learning with a fixed positive class and arbitrary negative shift is much simpler than \acronym~learning. In fact, provided a positive-labeled set and two unlabeled sets as above, they show that arbitrary negative shift is trivially equivalent to ordinary PU~learning over~$\pset{}$ and~$\uset{\te}$ (since $\pset{}$ being drawn from~$\plike{\te}$ renders $\uset{\tr}$ unnecessary). When both the positive and negative classes shift arbitrarily, learning is impossible without additional data and/or assumptions.  \acronym{}~learning's complexity sits between these two extremes.
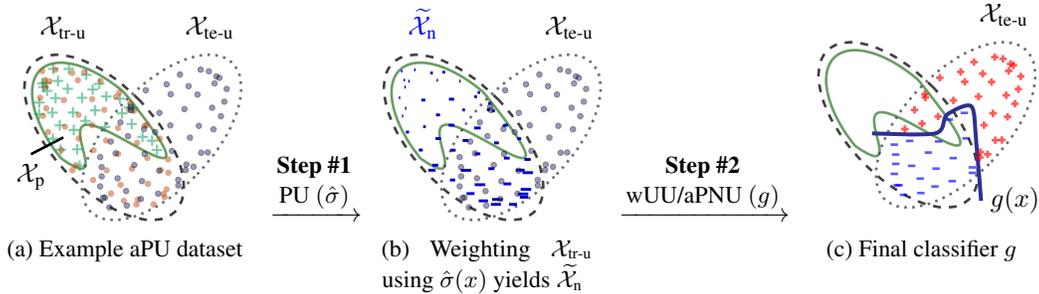
\begin{figure*}[t]
  \centering
  \newcommand{\wuupad}[1]{\hspace{0.27cm}$\xrightarrow{\shortstack{#1}}$\hspace{0.27cm}}
  \begin{subfigure}[t]{.225\textwidth}
    \centering
\begin{tikzpicture}
    \pgfplotstableread[col sep=comma] {plots/data/wuu_raw.csv}\thedata
    \pgfplotstableread[col sep=comma] {plots/data/raw-u-tr_border.csv}\trborderdata
    \pgfplotstableread[col sep=comma] {plots/data/raw-p_border.csv}\pborderdata
    \pgfplotstableread[col sep=comma] {plots/data/raw-u-te_border.csv}\teborderdata
    \begin{axis}
        [
          ymin=-5,
          ymax=8,
          xmin=-5,
          xmax=8.5,
          point meta=explicit,
          width=\TwoStepDim,
          height=\TwoStepDim,
          axis line style={draw=none},
          ticks=none,
          clip mode=individual,                %
          scatter/classes={
            0={mark=+,color=ForestGreen,mark size=2.0pt,line width=0.8pt},
            1={mark=*,color=Bittersweet,mark size=1.0pt},
            2={mark=*,color=Blue,draw=black,mark size=1.0pt}
          },
        ]
       \addplot[
         opacity=0.75,
         smooth,
         dashed,
         draw=black,
         color=black,
         line width=1pt,
       ] table[x index=0,y index=1] {\trborderdata};
       \node[] at (axis cs: -2.5,6.5) {$\uset{\tr}$};
       \addplot[
         opacity=0.75,
         smooth,
         draw=OliveGreen,
         line width=1pt,
       ] table[x index=0,y index=1] {\pborderdata};
       \addplot[black,line width=1pt,no markers] coordinates { (-2.5,0) (-4.5,-1) };
       \node[black,below] at (axis cs:-4.5,-1){$\pset{}$};
       \addplot[
         opacity=0.75,
         smooth,
         dotted,
         color=darkgray,
         line width=1pt,
       ] table[x index=0,y index=1] {\teborderdata};
       \node[] at (axis cs: 6.5,6.5) {$\uset{\te}$};
        \addplot[
          scatter,
          only marks,
          opacity=0.5,
          line width=0.2,
          mark size=0.7pt,
       ]
       table[x index=0,y index=1, meta index=2] {\thedata};
       \label[0]{leg:TwoStep:Raw:Pos}
       \label[1]{leg:TwoStep:Raw:UTr}
       \label[2]{leg:TwoStep:Raw:UTe}
    \end{axis}
\end{tikzpicture}
     \subcaption{Example \acronym~dataset}\label{fig:TwoStep:aPUExample}
  \end{subfigure}
  \wuupad{\textbf{\footnotesize Step~\#1}\\{\footnotesize \text{PU}~$(\cdc)$}}%
  \begin{subfigure}[t]{.20\textwidth}
    \centering
\begin{tikzpicture}
    \pgfplotstableread[col sep=comma] {plots/data/wuu_step1_u-tr.csv}\utrdata
    \pgfplotstableread[col sep=comma] {plots/data/wuu_step1_u-te.csv}\utedata
    \pgfplotstableread[col sep=comma] {plots/data/raw-p_border.csv}\pborderdata
    \pgfplotstableread[col sep=comma] {plots/data/raw-u-tr_border.csv}\trborderdata
    \pgfplotstableread[col sep=comma] {plots/data/raw-u-te_border.csv}\teborderdata
    \begin{axis}
        [
          ymin=-5,
          ymax=8,
          xmin=-5,
          xmax=8.5,
          point meta=explicit,
          width=\TwoStepDim,
          height=\TwoStepDim,
          axis line style={draw=none},
          ticks=none,
          clip mode=individual,                %
        ]
       \addplot[
         opacity=0.75,
         smooth,
         dashed,
         draw=black,
         color=black,
         line width=1pt,
       ] table[x index=0,y index=1] {\trborderdata};
     \node[] at (axis cs: -2.5,6.75) {\colortext{blue}{$\nsetPS{}$}};
       \addplot[
         opacity=0.75,
         smooth,
         draw=OliveGreen,
         line width=1pt,
       ] table[x index=0,y index=1] {\pborderdata};
       \addplot[
         opacity=0.75,
         smooth,
         dotted,
         color=darkgray,
         line width=1pt,
       ] table[x index=0,y index=1] {\teborderdata};
        \addplot[
            scatter,
            only marks,
            scatter src=explicit,
            visualization depends on=\thisrow{w2}\as\wtwo,
            scatter/@pre marker code/.append style={
              /tikz/mark size=\wtwo
            },
            mark=-,
            line width=1.1pt,
            color=blue,
            ] table[x=x,y=y,meta=meta]{plots/data/wuu_step1_u-tr.csv};\label{leg:TwoStep:Step1:UTr}
        \addplot[
          scatter,
          only marks,
          opacity=0.5,
          line width=0.2,
          mark size=0.7pt,
          scatter/classes={
             2={mark=*,color=Blue,draw=black,mark size=1.0pt}
          },
       ] table[x index=0,y index=1, meta index=2] {\utedata};
       \node[] at (axis cs: 6.5,6.5) {$\uset{\te}$};
    \end{axis}
\end{tikzpicture}
     \subcaption{Weighting~$\uset{\tr}$ \linebreak using~$\cdcx$ yields~\eqsmall{$\nsetPS{}$}}\label{fig:TwoStep:Step1}
  \end{subfigure}
  \wuupad{\textbf{\footnotesize Step~\#2}\\{\small \text{\wuu/\bpnu}~$(\dec)$}}
  \begin{subfigure}[t]{.20\textwidth}
    \centering
\begin{tikzpicture}
    \pgfplotstableread[col sep=comma] {plots/data/wuu_step2_u-te.csv}\utedata
    \pgfplotstableread[col sep=comma] {plots/data/raw-p_border.csv}\pborderdata
    \pgfplotstableread[col sep=comma] {plots/data/raw-u-tr_border.csv}\trborderdata
    \pgfplotstableread[col sep=comma] {plots/data/raw-u-te_border.csv}\teborderdata
    \begin{axis}
        [
          ymin=-5,
          ymax=8,
          xmin=-5,
          xmax=8.5,
          point meta=explicit,
          width=\TwoStepDim,
          height=\TwoStepDim,
          axis line style={draw=none},
          ticks=none,
          clip mode=individual,                %
          scatter/classes={
             2={mark=-,color=blue,draw=blue,mark size=1.4pt,line width=1pt},
             3={mark=+,color=red,mark size=1.4pt,line width=1pt}
          },
        ]
       \addplot[
         opacity=0.75,
         smooth,
         dashed,
         draw=black,
         color=black,
         line width=1pt,
       ] table[x index=0,y index=1] {\trborderdata};
       \addplot[
         opacity=0.75,
         smooth,
         draw=OliveGreen,
         line width=1pt,
       ] table[x index=0,y index=1] {\pborderdata};
       \addplot[
         opacity=0.75,
         smooth,
         dotted,
         color=darkgray,
         line width=1pt,
       ] table[x index=0,y index=1] {\teborderdata};
       \addplot[color=Blue,line width=1.5pt,no markers,smooth] coordinates { (-1.6,-0.5) (2.3,-0.5) (3.0,0.6) (4.5,1.1) (4.8,-1.) (5.1,-4.5) };
       \node[below right] at (axis cs: 5.1,-3.1) {$\decx$};
        \addplot[
          scatter,
          only marks,
          opacity=0.65,
          line width=0.2,
          mark size=1.0pt,
       ] table[x index=0,y index=1, meta index=2] {\utedata};
       \label[2]{leg:TwoStep:Step2:Neg}
       \label[3]{leg:TwoStep:Step2:Pos}
       \node[] at (axis cs: 6.5,6.5) {$\uset{\te}$};
    \end{axis}
\end{tikzpicture}
     \subcaption{Final classifier~$\dec$}\label{fig:TwoStep:Step2}
  \end{subfigure}
  \caption{Two-step \acronym\ learning. Fig.~\ref{fig:TwoStep:aPUExample}~shows a toy \acronym\ dataset with
  (\ref{leg:TwoStep:Raw:Pos})~representing a labeled positive example, (\ref{leg:TwoStep:Raw:UTr})~an unlabeled train sample, and (\ref{leg:TwoStep:Raw:UTe})~an unlabeled test sample.  Borders surround each set for clarity. After learning probabilistic classifier~$\cdc$ in Step~\#1, Fig.~\ref{fig:TwoStep:Step1}~visualizes $\cdc$'s~predicted negative-posterior probability using marker~(\ref{leg:TwoStep:Step1:UTr}) size. Fig.~\ref{fig:TwoStep:Step2}~shows the final decision boundary with~(\ref{leg:TwoStep:Step2:Neg}) and~(\ref{leg:TwoStep:Step2:Pos}) representing $\usetTE$~examples classified negative and positive respectively.}\label{fig:TwoStep:Flow}
\end{figure*}

\section{\acronym~Learning via Unlabeled-Unlabeled Learning}\label{sec:WUU}

  To build an intuition for solving the \acronym~learning problem, consider the ideal case where a perfect classifier correctly labels~$\uset{\tr}$. Let~$\nset{\tr}$ be $\uset{\tr}$'s~negative examples. $\nset{\tr}$~is SCAR w.r.t.~${\nlike{\tr}}$ and by Eq.~\eqref{eq:PreviousWork:NegAssume}'s assumption also~$\nlike{\te}$. Multiple options exist to then train the second classifier,~$\dec$, e.g.,~NU~learning with $\nset{\tr}$ and~$\uset{\te}$.

  A perfect classifier is unrealistic.  Is there an alternative? Our key insight is that by weighting~$\uset{\tr}$ (similar to covariate shift's importance function) it can be transformed into a representative negative set. From there, we consider two methods to fit the second classifier~$\dec$: one a variant of NU~learning we call weighted\=/unlabeled, unlabeled (\wuu)~learning and the other a semi-supervised method we call arbitrary\=/positive, negative, unlabeled (\bpnu)~learning.  We refer to the complete algorithms as \punu\ and \pupnu, respectively.

Figure~\ref{fig:TwoStep:Flow} visualizes our two-step approach, with a formal presentation in Algorithm~\ref{alg:TwoStep}. Below is a detailed description and theoretical analysis.

\begin{algorithm}[t]
  \caption{Two-step unlabeled-unlabeled \acronym{}~learning}\label{alg:TwoStep}
\textbf{Input}: Labeled\=/positive set~$\pset{}$ and unlabeled sets $\uset{\tr}, \uset{\te}$

\textbf{Output}: $\dec$'s model parameters~$\params$

\begin{algorithmic}[1]
  \STATE Train probabilistic classifier~$\cdc$ using $\pset{}$ and $\uset{\tr}$
  \STATE Use $\cdc$ to transform~$\uset{\tr}$ into surrogate negative set~$\nsetPS{}$
  \STATE Train final classifier,~$\dec$, using ERM with~$\rskWUU$ or~$\rskBPNU$
\end{algorithmic}%
 \end{algorithm}

\subsection*{Step~\#1: Create\texorpdfstring{}{ a} Surrogate Negative Set\texorpdfstring{~$\nsetPS{}$ from~$\uset{\tr}$}{}}\label{sec:WUU:Step1}

  This step's goal is to learn the training distribution's negative class-posterior,~$\ppost{\tr}{\yn}$. We achieve this by training PU~probabilistic classifier $\func{\cdc}{\domainX}{[0,1]}$ using~$\pset{}$ and~$\uset{\tr}$. In principle, any probabilistic PU~method can be used; we focused on ERM-based PU~methods so the logistic loss served as surrogate,~$\loss$. Sigmoid activation is applied to the model's output to bound its range to~${(0,1)}$.

  \begin{theorem}\label{thm:TwoStep:Unbiased}
    Let $\func{\dec}{\domainX}{\real}$ be an arbitrary decision function and $\func{\loss}{\real}{\realnn}$ be a loss function bounded w.r.t.~$\dec$. Let ${\yhat \in \domainY}$ be a predicted label.  Define \eqsmall{${\uset{\tr} \defeq \set{\X_i}_{i=1}^{\nTrU} \iidsim \marg{\tr}}$}, and restrict ${\prd{\tr} \in [0,1)}$. Define \eqsmall{${\trsk{\nuSub}{\yhat} \defeq \frac{1}{\nTrU} \sum_{\X_i \in \uset{\tr}} \frac{\cdcxI \lbase{\decxI}{\yhat}}{1-\prd{\tr}}}$}. Let $\func{\cdc}{\domainX}{\sbrack{0,1}}$ be in hypothesis set~\eqsmall{$\cdcCls$}. When ${\cdcx = \ppost{\tr}{\yn}}$, \eqsmall{$\trsk{\nuSub}{\yhat}$}~is an unbiased estimator of~\eqsmall{$\lblrisk{\nsub}{\yhat}$}. When the concept class of functions that defines~$\ppost{\tr}{\yn}$ is probably approximately correct (PAC)~learnable by some PAC\=/learning algorithm~$\learner$ that selects ${\cdc \in \cdcCls}$, then \eqsmall{$\trsk{\nuSub}{\yhat}$} is a consistent estimator of~\eqsmall{$\lblrisk{\nsub}{\yhat}$}.
  \end{theorem}

  From Theorem~\ref{thm:TwoStep:Unbiased}, we see that \emph{soft} weighting each unlabeled instance in~$\uset{\tr}$ by~$\cdc$ yields a \emph{surrogate negative set}~\eqsmall{$\nsetPS{}$} that can be used to estimate the train/test negative labeled risk. %
We form~\eqsmall{$\nsetPS{}$} transductively, but inductive learning is an option. Since $\uset{\tr}$ contains positive examples, $\cdc$~may overfit and memorize random positive example variation. This is usually detectable via an implausible validation loss given $\prd{\tr}$, $\np$, and~$\nTrU$. Care should be shown to tune $\cdc$'s capacity and regularization.

Supplemental Section~\ref{sec:App:AddRes:AltStep1} proposes and empirically evaluates two additional methods to construct~\eqsmall{$\nsetPS{}$}. While these other methods are not statistically consistent, they may outperform soft weighting.

  \neuripsParagraph{What if $\psetHeading$ is not SCAR?} Our \acronym{}~learning setting, detailed in Section~\ref{sec:APULearning}, specifies that~$\pset{}$ is representative of $\uset{\tr}$'s~positive examples.  In scenarios where $\pset{}$~is biased w.r.t.~$\uset{\tr}$, any bPU~method (e.g.,~\citep{Bekker:2019,Kato:2019}) can be used in step~\#1 to (hard) label~$\uset{\tr}$ thereby constructing~\eqsmall{$\nsetPS{}$}.

  \subsection*{Step~\#2: Train the Test Distribution Classifier\texorpdfstring{~$\dec$}{}}\label{sec:TwoStep:Step2}

  Negative-unlabeled~(NU) learning is functionally the same as PU~learning.  \citet{Sakai:2017} formalize an unbiased~NU risk estimator,
\eqsmall{${\ersk{\text{NU}} \defeq \big\lvert  \elblrsk{\usub}{+} - (1 - \prd{}) \elblrsk{\nsub}{+} \big\rvert +  (1 - \prd{}) \elblrsk{\nsub}{-}}$}
(defined here with our absolute-value correction). Our \emph{weighted-unlabeled, unlabeled}\footnote{``Unlabeled\=/unlabeled learning'' denotes the two unlabeled sets and is different from UU~learning in~\citep{Menon:2015,Lu:2019}.}~(\wuu) estimator,
  \begin{equation}\label{eq:TwoStep:wUU}
      \rskWUU \defeq \nnwrap{\elblrsk{\tesub\usub}{+} - (1 - \prd{\te}) \trsk{\nuSub}{+}} +  (1 - \prd{\te}) \trsk{\nuSub}{-} \text{,}
  \end{equation}%
  \noindent%
  modifies \citeauthor{Sakai:2017}'s definition to use~\eqsmall{$\nsetPS{}$} and~$\uset{\te}$. Observe that \eqsmall{$\rskWUU$}~uses only data that was originally unlabeled. When \eqsmall{$\trsk{\nuSub}{\yhat}$}~is consistent, \wuu~is also consistent just like nnPU/\absPU\@.

  \neuripsParagraph{Risk Estimation with Positive Data Reuse} When~$\plike{\tr}$'s and~$\plike{\te}$'s supports intersect, $\pset{}$~may contain useful information about the target distribution given limited data. In such settings, a semi-supervised approach leveraging~$\pset{}$, surrogate~\eqsmall{$\nsetPS{}$}, and~$\uset{\te}$ may perform better than~\wuu{}.

  \citet{Sakai:2017} propose the \pnu~risk estimator, \eqsmall{${\ersk{\text{\pnu}} \defeq (1 - \pnutrade) \ersk{\text{PN}} + \pnutrade \ersk{\text{NU}}}$}, where hyperparameter~${\pnutrade \in (0,1)}$ weights the~PN and~NU estimators. Our \emph{arbitrary-positive, negative, unlabeled}~(\bpnu) risk estimator in Eq.~\eqref{eq:TwoStep:PNU:Empirical} modifies \pnu\ to use~\eqsmall{$\nsetPS{}$} and our absolute-value correction.
  \begin{equation}\label{eq:TwoStep:PNU:Empirical}
    \rskBPNU = (1 - \pnutrade) \prd{\te} \elblrsk{\psub}{+} + (1 - \prd{\te}) \trsk{\nuSub}{-} + \pnutrade \nnwrap{\elblrsk{\tesub{\usub}}{+} - (1 - \prd{\te}) \trsk{\nuSub{}}{+}}
 \end{equation}%
 \noindent%
 If ${\pnutrade = 0}$, \bpnu{} ignores the test distribution (i.e.,~$\uset{\te}$) entirely.  If ${\pnutrade = 1}$, \bpnu~is simply~\wuu\@. When a large positive shift is expected (e.g.,~by domain-specific knowledge), $\pset{}$~is of limited value so set $\pnutrade$~closer to~1. For small expected positive shifts, set $\pnutrade$~closer to~0.  A midpoint value of ${\pnutrade = 0.5}$ empirically performed well when no knowledge about the positive shift was assumed.

\neuripsParagraph{ERM Framework} Both \eqsmall{$\rskWUU$} and \eqsmall{$\rskBPNU$} integrate into a standard ERM framework since they use our absolute-value correction.  For completeness, supplemental materials Section~\ref{sec:App:ERM:TwoStep} details their custom ERM~algorithm if \citet{Kiryo:2017}'s~non\=/negativity correction is used instead.

\neuripsParagraph{Heterogeneous Classifiers} Two\=/step learners enable different learner architectures in each step (e.g.,~random forest for step~\#1 and a neural network for step~\#2). Our experiments leverage this flexibility where $\cdc$'s~neural network may have fewer hidden layers or different hyperparameters than~$\dec$ in step~\#2.
\section{Positive-Unlabeled Recursive Risk Estimation}\label{sec:PURPL}

Two-step methods --- both ours and \PUc\ --- solve a challenging problem by decomposing it into sequential (easier) subproblems. Serial decision making's disadvantage is that earlier errors propagate and can be amplified when subsequent decisions are made on top of those errors.

Can our \acronym\ problem setting be learned in a single \textit{joint} method?  \citeauthor{Sakai:2019} leave it as an open question.  We show in this section the answer is yes.
To understand why this is possible, it helps to simplify our perspective of unbiased PU~and NU~learning. When estimating a labeled risk,~\eqsmall{$\elblrsk{\gsub}{\yhat}$} (where ${\gsub \in \set{\psub,\nsub}}$), the ideal case is to use SCAR~data from class-conditional distribution~$\likebase{}{\gsub}$.  When such labeled data is unavailable, the risk \textit{decomposes} via the simple linear transformation,
\begin{equation}\label{eq:PURPL:RiskDecomposition}
  (1 - \NewPrior) \elblrsk{\CallDist}{\yhat} = \elblrsk{\UnlabeledDist}{\yhat} - \NewPrior \elblrsk{\CallComplement}{\yhat}
\end{equation}
\noindent
where ${A=\nsub}$ and ${B=\psub}$ for PU~learning or vice versa for NU~learning. $\NewPrior$~is the positive (negative) prior for PU (NU)~learning.

In standard PU~and NU~learning, either~\eqsmall{$\elblrsk{A}{\yhat}$} or \eqsmall{$\elblrsk{B}{\yhat}$} can always be estimated from labeled data.  If that were not true, can this decomposition be applied recursively (i.e.,~nested)?  The answer is again yes.  Below we apply recursive risk decomposition to our \acronym~learning task.

\subsection*{Applying Recursive Risk to \acronym~learning}

Our \underline{p}ositive-\underline{u}nlabeled \underline{r}ecursive \underline{r}isk~(\nnpurpl) estimator quantifies our \acronym~setting's empirical risk and integrates into a standard ERM~framework.  \nnpurpl's~top-level definition is simply the test risk:
\begin{equation}\label{eq:PURR:PN}
  \rskPURPL = \prd{\te} \elblrsk{\tesub\psub}{+} + (1-\prd{\te}) \elblrsk{\tesub\nsub}{-} \text{.}
\end{equation}
\noindent
Since only unlabeled data is drawn from the test distribution, both terms in Eq.~\eqref{eq:PURR:PN} require risk decomposition. First, for \eqsmall{$\elblrsk{\tesub\nsub}{-}$}, we consider
its more general form \eqsmall{$\elblrsk{\tesub\nsub}{\yhat}$} below since \eqsmall{$\elblrsk{\tesub\nsub}{+}$} will be needed as well.  Using Eq.~\eqref{eq:PreviousWork:NegAssume}'s~assumption, \eqsmall{$\elblrsk{\tesub\nsub}{\yhat}$} can be estimated directly from the training distribution. Combining Eq.~\eqref{eq:PreviousWork:NegRisk} with absolute-value correction, we see that
\begin{equation}\label{eq:PURR:NegLabel}
  \elblrsk{\tesub\nsub}{\yhat} =\elblrsk{\trsub\nsub}{\yhat} = \frac{1}{1 - \prd{\tr}} \nnwrap{\elblrsk{\trsub\usub}{\yhat} - \prd{\tr}\elblrsk{\trsub\psub}{\yhat}} \text{.}
\end{equation}

Next, \eqsmall{$\elblrsk{\tesub\psub}{+}$}, as a positive risk, undergoes NU~decomposition so (with absolute\=/value correction):
\begin{equation}\label{eq:PURR:PosLabel:Flat}
  \prd{\te}\elblrsk{\tesub\psub}{+} = \nnwrap{\elblrsk{\tesub\usub}{+} - (1 - \prd{\te})\elblrsk{\tesub\nsub}{+}} \text{.}
\end{equation}
\noindent
Eq.~\eqref{eq:PURR:NegLabel} with ${\yhat=\yp}$ substitutes for~\eqsmall{$\elblrsk{\tesub\nsub}{+}$} in Eq.~\eqref{eq:PURR:PosLabel:Flat} yielding  \eqsmall{$\rskPURPL$}'s complete definition:
{\small
  \begin{equation}\label{eq:PURPL:Complete}
    \rskPURPL = \Bigg\lvert\underbrace{\elblrsk{\tesub\usub}{+} - (1 - \prd{\te}) \bigg\lvert \underbrace{\frac{\elblrsk{\trsub\usub}{+} - \prd{\tr}\elblrsk{\trsub\psub}{+}}{1 - \prd{\tr}}}_{\elblrsk{\tesub\nsub}{+}} \bigg\rvert}_{\prd{\te}\elblrsk{\tesub\psub}{+}}\Bigg\lvert
               + (1 - \prd{\te}) \bigg\lvert \underbrace{\frac{\elblrsk{\trsub\usub}{-} - \prd{\tr}\elblrsk{\trsub\psub}{-}}{1 - \prd{\tr}}}_{\elblrsk{\tesub\nsub}{-}} \bigg\rvert \text{.}
  \end{equation}
}

\begin{theorem}\label{thm:PURPL:BiasedConsistent}
  Fix decision function~${\dec \in \decCls}$.   If $\loss$~is bounded over ${\decx}$'s~image and \eqsmall{${\elblrsk{\tesub\nsub}{\yhat},\elblrsk{\tesub\psub}{+} > 0}$} for ${\yhat \in \domainY}$, then \eqsmall{$\rskPURPL$} is a consistent estimator.  \eqsmall{$\rskPURPL$}~is a biased estimator unless for all \eqsmall{${\uset{\tr} \iidsim \marg{\tr}}$}, \eqsmall{${\uset{\te} \iidsim \marg{\te}}$}, and \eqsmall{${\pset{} \iidsim \plike{\tr}}$} it holds that \eqsmall{${\Pr[\elblrsk{\trsub\usub}{\yhat} - (1 - \prd{\te})\elblrsk{\trsub\psub}{\yhat} < 0] = 0}$} and \eqsmall{${\Pr[\elblrsk{\tesub\usub}{+} - (1 - \prd{\te})\elblrsk{\tesub\nsub}{+} < 0] =0}$}.
\end{theorem}

\neuripsParagraph{Optimization} \nnpurpl{}~with absolute-value correction integrates into a standard ERM~framework. If non\=/negativity is used instead, \nnpurpl{}'s optimization scheme becomes significantly more complicated as it must consider four candidate gradients per update; see suppl.\ Section~\ref{sec:App:ERM:PURPL} for more details.
\section{Experimental Results}\label{sec:ExpRes}
We empirically studied the effectiveness of our methods -- \nnpurpl, \punu, and \pupnu{} -- using synthetic and real\=/world data.\footnote{Our implementation is publicly available at: \blue{\href{\sourceCode}{\sourceCode}}.} Limited space allows us to discuss only two experiment sets here.  Suppl.\ Section~\ref{sec:App:AddRes} details experiments on: synthetic data, 10~LIBSVM datasets~\citep{LIBSVM} under a totally different positive-bias condition, and a study of our methods' robustness to negative-class shift.

\subsection{Experimental Setup}

Supplemental Section~\ref{sec:App:ExperimentSetup} enumerates our complete experimental setup with a brief summary below.

\neuripsParagraph{Baselines} \PUc~\citep{Sakai:2019} with a linear-in-parameter model and Gaussian kernel basis is the primary baseline.\footnote{The \PUc~implementation was provided by~\citet{Sakai:2019} via personal correspondence.}
Ordinary nnPU is the performance floor. To ensure the strongest baseline, we separately trained nnPU with unlabeled set~$\uset{\te}$ as well as with the combined~${\uset{\tr} \cup \uset{\te}}$ (using the true, composite prior) and report each experiment's best performing configuration, denoted~\nnpuOpt. PN\=/test (trained on labeled~$\uset{\te}$) provides a reference for the performance ceiling. All methods saw identical training/test data splits and where applicable used the same initial weights.

\neuripsParagraph{Datasets} Section~\ref{sec:ExpRes:NonidenticalSupports} considers the MNIST~\citep{LeCun:1998}, CIFAR10~\citep{CIFAR10}, and 20~Newsgroups~\citep{20newsgroups} datasets with binary classes formed by partitioning each dataset's labels.  Section~\ref{sec:ExpRes:Spam} uses two different TREC~\citep{TREC} spam email datasets to demonstrate our methods' performance under real-world adversarial concept drift.  Further details on all datasets are in the supplemental materials.

\neuripsParagraph{Learner Architecture} We focus on training neural networks~(NNs) via stochastic optimization (i.e.,~AdamW~\citep{AdamW} with AMSGrad~\citep{AMSGrad}).  Probabilistic classifier,~$\cdc{}$, used our \absPU~risk estimator with logistic loss. All other learners used sigmoid loss for~$\loss$.
Since \PUc{}~is limited to linear models with Gaussian kernels, we limited our NNs to at most three fully-connected layers of 300~neurons. For MNIST, our NNs~were trained from scratch.  Pretrained deep networks encoded the CIFAR10, 20~Newsgroups, and TREC spam datasets into static representations all learners used. Specifically, the 20~Newsgroups documents and TREC emails were encoded into 9,216~dimensional vectors using ELMo~\citep{Peters:2018}. This encoding scheme was used by~\citet{Hsieh:2019} and is based on~\citep{Ruckle:2018}. DenseNet\=/121~\citep{DenseNet} encoded each CIFAR10 image into a 1,024~dimensional vector.

\neuripsParagraph{Hyperparameters} Our only individually tuned hyperparameters are learning rate and weight decay. We assume the worst case of no \textit{a priori} knowledge about the positive shift so midpoint value ${\pnutrade = 0.5}$ was used.  \PUc{}'s~hyperparameters were tuned via importance-weighted cross validation~\citep{Sugiyama:Krauledat:2007}. For the complete hyperparameter details, see supplemental materials Section~\ref{sec:App:ExpSetup:Hyperparameters}.

\begin{table*}[t]
  \centering
  \caption{Mean inductive misclassification rate~(\%) over 100~trials for MNIST, 20~Newsgroups, \& CIFAR10 for different positive \& negative class definitions.  Bold denotes a \emph{shifted task}'s best performing method. For \emph{all} shifted tasks, our three methods -- denoted with~\ourMethodKey{} -- statistically outperformed~\PUc{} and~\nnpuOpt{} based on a paired t\=/test $(p < 0.01)$. Each dataset's first three experiments have identical negative~(N) \& positive\=/test~(\pDsTest) class definitions. Positive train~(\pDsTrain) specified as ``\pDsTest'' denotes no bias. Additional shifted tasks (with result standard deviations) are in the supplemental materials.}\label{tab:ExpRes:Nonoverlap}
\newcommand{\DsName}[1]{\multirow{4}{*}{\rotatebox[origin=c]{90}{#1}}}

\newcommand{\PNpartial}[2]{\NResMain{#1}{#2}}

\newcommand{\NResMainUL}[2]{\NResMain{#1}{#2}}
\newcommand{\NResMainBest}[2]{\textbf{\NResMain{#1}{#2}}}

\renewcommand{\arraystretch}{1.2}
\setlength{\dashlinedash}{0.4pt}
\setlength{\dashlinegap}{1.5pt}
\setlength{\arrayrulewidth}{0.3pt}

\setlength{\tabcolsep}{5.4pt}

{\centering
  \small
  \begin{tabular}{@{}llllllrrrrrr@{}}
  \toprule
  \multirow{2}{*}{} & \multirow{2}{*}{N} & \multirow{2}{*}{\pDsTest} & \multirow{2}{*}{\pDsTrain} &  \multirow{2}{*}{$\prd{\tr}$} & \multirow{2}{*}{$\prd{\te}$} &    & \multicolumn{2}{c}{\footnotesize Two\=/Step (\putwo)}  & \multicolumn{2}{c}{\footnotesize Baselines}  & \multicolumn{1}{c@{}}{\footnotesize Ref.} \\\cmidrule(lr){8-9}\cmidrule(lr){10-11}\cmidrule(l){12-12}
                    &   & & &  & & \hspace{-6pt}\nnpurpl\ourMethodKey{} & \bpnu\ourMethodKey{} & \wuu\ourMethodKey{} & \PUc{} & \hspace{-3.8pt}\nnpuOpt{} & \multicolumn{1}{c@{}}{\PNte} \\\midrule
  \DsName{MNIST} & \NDef{0, 2, 4,\\ 6, 8} & \PTestDef{1, 3, 5,\\ 7, 9} & \SamePTest & 0.5 & 0.5 & \NResMain{10.0}{1.3} & \NResMain{10.0}{1.2} & \NResMain{11.6}{1.6} & \NResMain{8.6}{0.8} & \NResMain{5.5}{0.5} & \PNup{} \\\cdashline{4-11}
    &       &         & \multirow{2}{*}{7, 9}    & 0.5  & 0.5 & \NResMainUL{9.4}{1.5} & \NResMainBest{7.1}{0.9} & \NResMainUL{8.3}{1.5} & \NResMain{26.8}{2.4} & \NResMain{35.1}{2.5} & \PNpartial{2.8}{0.2}  \\\cdashline{5-11}
    &       &         &     & 0.29 & 0.5 & \NResMainUL{6.8}{0.8} & \NResMainBest{5.3}{0.6} & \NResMainUL{6.0}{0.7} & \NResMain{29.2}{2.1} & \NResMain{36.7}{2.7} & \PNdown{} \\\cdashline{2-12}
      & {0, 2}  & {5, 7}  & {1, 3}  & 0.5  & 0.5 & \NResMainUL{4.0}{0.8} & \NResMainUL{3.6}{0.9} & \NResMainBest{3.1}{0.7} & \NResMain{17.1}{4.6} & \NResMain{30.9}{5.3} & \PNpartial{1.1}{0.2} \\\midrule
    \DsName{20~News.} & \NDef{sci, soc,\\talk} & \PTestDef{alt, comp,\\ misc, rec} & \SamePTest & 0.56 & 0.56 & \NResMain{15.4}{1.3} & \NResMain{14.9}{1.2} & \NResMain{16.7}{2.3} & \NResMain{14.9}{1.0} & \NResMain{14.1}{0.8} & \PNup{} \\\cdashline{4-11}
  & &  & \multirow{2}{*}{misc, rec}   & 0.56 & 0.56 & \NResMainUL{17.5}{2.1} & \NResMainBest{13.5}{0.8} & \NResMainUL{15.1}{1.3} & \NResMain{23.9}{3.0} & \NResMain{28.8}{1.7} & \PNpartial{10.5}{0.4} \\\cdashline{5-11}
  & &  &    & 0.37 & 0.56 & \NResMainUL{13.9}{0.7} & \NResMainBest{12.8}{0.6} & \NResMainUL{14.3}{0.9} & \NResMain{28.9}{1.8} & \NResMain{28.8}{1.3} & \PNdown{} \\\cdashline{2-12}
    & {misc, rec} & {soc, talk} & {alt, comp} & 0.55 & 0.46 & \NResMainUL{5.9}{1.0} & \NResMainUL{7.1}{1.1} & \NResMainBest{5.6}{1.7} & \NResMain{18.5}{4.3} & \NResMain{35.3}{5.2} & \PNpartial{2.1}{0.3}  \\\midrule
    \DsName{CIFAR10}  & \NDef{Bird, Cat,\\Deer, Dog,\\Frog, Horse} & \PTestDef{Plane, \\ Auto, Ship, \\ Truck} & \SamePTest & 0.4 & 0.4 & \NResMain{14.1}{0.9} & \NResMain{14.2}{1.3} & \NResMain{15.5}{1.6} & \NResMain{13.8}{0.8} & \NResMain{12.3}{0.6} & \PNup{} \\\cdashline{4-11}
  & & & \multirow{2}{*}{Plane} & 0.4 & 0.4 & \NResMainBest{13.8}{0.9} & \NResMainUL{14.5}{1.4} & \NResMainUL{15.1}{1.6} & \NResMain{20.6}{1.5} & \NResMain{27.4}{1.0} & \PNpartial{9.8}{0.6} \\\cdashline{5-11}
  & & &  & 0.14 & 0.4 & \NResMainUL{12.1}{0.7} & \NResMainBest{11.9}{0.7} & \NResMainUL{12.4}{0.9} & \NResMain{26.7}{1.4} & \NResMain{26.7}{1.0} & \PNdown{} \\\cdashline{2-12}
    & {Deer, Horse} & {Plane, Auto} & {Cat, Dog} & 0.5 & 0.5 & \NResMainUL{14.1}{0.9} & \NResMainUL{14.9}{1.5}  & \NResMainBest{11.2}{0.8} & \NResMain{33.1}{2.7} & \NResMain{47.5}{2.0} & \PNpartial{7.7}{0.4} \\
  \bottomrule
\end{tabular}
}
\end{table*}

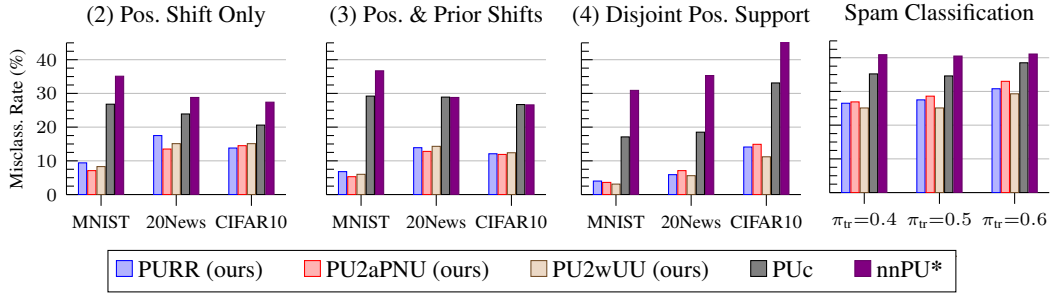
\begin{figure*}[tb]
  \centering
  \hspace{-10pt}
  \begin{subfigure}[b]{0.27\textwidth}%
    \BarGraphBase{(2)~Pos.\ Shift Only}{bar-shift-pos-only}{Misclass.\ Rate~(\%)}{}%
  \end{subfigure}%
  ~~%
  \begin{subfigure}[b]{0.23\textwidth}%
    \BarGraphBase{(3)~Pos.\ \& Prior Shifts}{bar-shift-pos-prior}{}{\disableTicks}%
  \end{subfigure}%
  ~~%
  \begin{subfigure}[b]{0.23\textwidth}%
    \BarGraphBase{(4)~Disjoint Pos.\ Support}{bar-disjoint}{}{\disableTicks}%
  \end{subfigure}%
  ~%
  \begin{subfigure}[b]{0.23\textwidth}%
    \BarGraphBase{Spam Classification}{spam-bar}{}{\disableTicks}%
  \end{subfigure}%
  \begin{center}
\newcommand{\legendSpacer}{\mbox{\hspace{0.3cm}}}
\begin{tikzpicture}
    \begin{axis}[%
      ybar,
      hide axis,  %
      xmin=0,  %
      xmax=1,
      ymin=0,
      ymax=1,
      scale only axis,width=1mm, %
      legend cell align={left},              %
      legend style={font=\footnotesize},
      legend columns=5,
    ]
    \pgfplotsinvokeforeach{\nnpurpl{} (ours),\pupnu{} (ours),\punu{} (ours),\PUc}{%
      \addplot coordinates {(0,0)};
      \addlegendentry{#1~~~~~~~}
    }
    \addplot coordinates {(0,0)};
    \addlegendentry{\nnpuOpt};
    \end{axis}
\end{tikzpicture}
   \end{center}
  \caption{Mean inductive misclassification rate over 100~trials on the MNIST, 20~News., CIFAR10, \& TREC spam datasets for our methods \& baselines.  Each numbered plot (i.e.,~2--4) corresponds to one experimental shift task in Table~\ref{tab:ExpRes:Nonoverlap}. Spam classification experiments are detailed in Section~\ref{sec:ExpRes:Spam}.}\label{fig:ExpRes:Nonoverlap}
\end{figure*}

\subsection{Partially and Fully Disjoint Positive Class-Conditional Supports}\label{sec:ExpRes:NonidenticalSupports}

Here we replicate scenarios where positive subclasses exist only in the test distribution (e.g.,~adversarial zero\=/day attacks). These experiments are modeled after \citet{Hsieh:2019}'s~experiments for positive, unlabeled, biased-negative (PUbN)~learning.

Table~\ref{tab:ExpRes:Nonoverlap} lists the experiments' positive train/test and negative class definitions. Datasets are sampled u.a.r.\ from their constituent sublabels.  Each dataset has four experimental conditions (ordered by row number): (1)~\mbox{\pDsTrain\ = \pDsTest}, i.e.,~no bias, (2~\& 3~resp.) partially disjoint positive supports without and with prior shift, and (4)~disjoint positive class definitions. $\prd{\te}$~equals \pDsTest's true prior w.r.t.\ ${\text{\pDsTest} \sqcup \text{N}}$.  By default~${\prd{\tr} = \prd{\te}}$; in the prior shift and disjoint support experiments~(rows~3 and~4), $\prd{\tr}$~equals \pDsTrain's true prior w.r.t.\ ${\text{\pDsTrain} \sqcup \text{N}}$.

\neuripsParagraph{Analysis} Results are shown in Table~\ref{tab:ExpRes:Nonoverlap} and Figure~\ref{fig:ExpRes:Nonoverlap}. On \emph{unshifted} data (row~1 for each dataset), baselines \PUc{} and \nnpuOpt{} slightly outperformed our methods, which shows that \PUc{}'s architecture is sufficiently expressive. In contrast, on \emph{shifted} data (rows~2--4 for each dataset), our methods' performance generally improved while both \PUc{}'s and \nnpuOpt{}'s performance always degraded.  This performance divergence demonstrates our methods' algorithmic advantage. In fact for all shifted tasks, our methods always outperformed \PUc\ and~\nnpuOpt\ according to a paired t\=/test (${p < 0.01}$).
For partially disjoint positive supports (rows~2 and~3 for each dataset), \pupnu\ was the top performer for five of six setups (\nnpurpl{} was top on the other). This pattern reversed for fully disjoint supports (row~4) where \pupnu\ always lagged \punu{}; this is expected as explained in Section~\ref{sec:TwoStep:Step2}.

Reducing~$\prd{\tr}$ always improved our algorithms' performance and degraded \PUc's. A smaller prior enables easier identification of $\uset{\tr}$'s~negative examples and in turn a more accurate estimation of $\uset{\te}$'s~negative risk. In contrast, importance weighting is most accurate in the absence of bias (see row~1 for each dataset).  Any shift increases density estimation's (and by extension \PUc's) inaccuracy.

\nnpuOpt\ outperformed both \PUc\ and our methods when there was no bias. This is expected. If an algorithm searches for non-existent phenomena, any additional patterns found will not generalize.

\subsection{Case Study: Arbitrary Adversarial Concept Drift}\label{sec:ExpRes:Spam}

\begin{table*}[t]
  \centering
  \caption{Mean inductive misclassification rate~(\%) over 100~trials for spam adversarial drift. Our methods -- \nnpurpl, \punu, and \pupnu{} -- outperformed \PUc{} \& \nnpuOpt{} based on a 1\%~paired t-test.  Each result's standard deviation appears in supplemental Table~\ref{tab:App:AddRes:Spam}.}\label{tab:ExpRes:Spam}
\renewcommand{\arraystretch}{1.2}
\setlength{\dashlinedash}{0.4pt}
\setlength{\dashlinegap}{1.5pt}
\setlength{\arrayrulewidth}{0.3pt}

\setlength{\tabcolsep}{5.4pt}

\newcommand{\NResSpam}[2]{#1}
\newcommand{\NResSpamTop}[2]{\textbf{#1}}

{\centering
  \footnotesize
  \begin{tabular}{@{}llllllrrrrrr@{}}
    \toprule
    \multicolumn{2}{c}{{\footnotesize Train Set}} & \multicolumn{2}{c}{{\footnotesize Test Set}} &  \multirow{2}{*}{$\prd{\tr}$} & \multirow{2}{*}{$\prd{\te}$} &    & \multicolumn{2}{c}{{\footnotesize Two\=/Step (\putwo)}}  & \multicolumn{2}{c}{{\footnotesize Baselines}}  & \multicolumn{1}{c@{}}{{\footnotesize Ref.}} \\\cmidrule(r){1-2}\cmidrule(lr){3-4}\cmidrule(lr){8-9}\cmidrule(lr){10-11}\cmidrule(l){12-12}
    Pos. & Neg.  & Pos.    &  Neg.     &      &     & \hspace{-0pt}\nnpurpl{}  & \bpnu{}  & \wuu{}  & \PUc{}            & \nnpuOpt{} & \PNte{} \\\midrule
     \NDef{2005 \\ Spam} & \NDef{2005 \\ Ham} & \NDef{2007 \\ Spam}   & \NDef{2007 \\ Ham}
                                          & 0.4  & 0.5 & \NResSpam{26.5}{2.6} & \NResSpam{26.9}{3.1}  & \NResSpamTop{25.1}{3.1}  & \NResSpam{35.2}{11.3} & \NResSpam{40.9}{3.1} & \multicolumn{1}{c@{}}{$\uparrow$} \\\cdashline{5-11}
                        &     &     &     & 0.5  & 0.5 & \NResSpam{27.5}{3.4} & \NResSpam{28.6}{4.5}  & \NResSpamTop{25.1}{3.3}  & \NResSpam{34.6}{10.2} & \NResSpam{40.5}{2.7} & \NResSpam{0.6}{0.3} \\\cdashline{5-11}
                        &     &     &     & 0.6  & 0.5 & \NResSpam{30.8}{4.2} & \NResSpam{33.0}{5.7}  & \NResSpamTop{29.3}{6.5}  & \NResSpam{38.5}{10.8} & \NResSpam{41.1}{2.9} & \multicolumn{1}{c@{}}{$\downarrow$} \\
  \bottomrule
\end{tabular}
}
 \end{table*}

PU~learning has been applied to multiple adversarial domains including opinion spam~\citep{Hernandez:2013,Li:2014,Zhang:2017,Zhang:2019:Malware}.  We use spam classification as a vehicle to test our methods in an adversarial setting by considering two different TREC email spam datasets -- training on TREC05 and evaluating on TREC07. Spam -- the positive class -- evolves quickly over time, but the two datasets' ham emails are also quite different: TREC05 relies on Enron emails while TREC07 contains mostly  emails from a university server. Thus, this represents a more challenging, realistic setting where Eq.~\eqref{eq:PreviousWork:NegAssume}'s assumption does not hold.

Table~\ref{tab:ExpRes:Spam} and Figure~\ref{fig:ExpRes:Nonoverlap} show that our methods outperformed \PUc{} and \nnpuOpt{} according to a 1\%~paired t\=/test across three training priors~($\prd{\tr}$).  \punu{}~was the top-performer as $\cdc$~accurately labeled~$\uset{\tr}$, yielding a strong surrogate negative set. \pupnu{}~performed slightly worse than~\punu{} as the significant adversarial concept drift greatly limited $\pset{}$'s~value.
Overall, these experiments show that our \acronym{}~setting arises in real-world domains. All of our methods handled large positive shifts better than prior work, even in realistic cases where the negative class also shifts.

\subsection{Discussion}

Our two-step methods assume asymptotic consistency for~$\eqsmall{\nsetPS{}}$ in step~\#1, but finite training data ensures a non\=/consistent evaluation setting. Nonetheless, either \pupnu{} or~\punu{} was the top performer in all but one experiment in this section.\footnote{Supplemental Sections~\ref{sec:App:AddRes:Nonoverlap} and~\ref{sec:App:AddRes:MarginalBias} enumerate multiple empirical setups where \nnpurpl{} is the top performer.} Supplemental Section~\ref{sec:App:AddRes:AltStep1} includes additional experiments where we further stress our two-step methods by forcing $\cdc$~away from our posterior estimate.  Even under those deleterious step~\#1 conditions, our two-step learners are robust.

Conventional wisdom suggests that joint method~\nnpurpl{} should outperform pipeline approaches. This intuition breaks down in our case because \nnpurpl{}, with its three risk decompositions, is strictly harder to optimize than \wuu, \bpnu, \absPU, and~\nnPU\@ -- all of which have a single decomposition. This harder optimization can lead to worse accuracy compared to the two-step methods, especially on easier problems (e.g.,~MNIST), where each step can be solved accurately on its own.

For completeness, suppl.\ Section~\ref{sec:App:AddRes:PUSB} compares our methods to \bpu{} selection bias method \pusb{}~\citep{Kato:2019}. Our algorithms generally outperformed \pusb{} on data specifically tuned for their method even after accounting for the differing unlabeled sets. Those experiments indicate that \pusb{}'s underlying assumption entails only a small data shift and further point to potential \pusb{} learning brittleness.
\section{Conclusions}

We examined arbitrary-positive, unlabeled (\acronym)~learning, where the labeled-positive and target-positive distributions may be arbitrarily different. A (nearly) fixed negative class-distribution allows us to train accurate classifiers without any labeled data from the target distribution (i.e.,~disjoint positive supports). Empirical results on real-world data above and in the supplementals show that our methods are still robust in the realistic case of some negative shift.  Future work seeks a less restrictive yet statistically-sound replacement assumption of a fixed negative class\=/conditional distribution.
\section{Broader Impact}

The algorithms proposed in this work are general and could be applied to many different applications. Forecasting the broader impact of work like this is challenging and generally inaccurate. With that caveat, we discuss potential impacts based on possible applications.

The case study on email spam suggests that our methods may be useful in adversarial domains, such as the detection of fraud, malware, network intrusion, distributed denial of service (DDoS) attacks, and many types of spam. In these settings, one class (e.g., spam) evolves quickly as attackers try to evade detection. For many of these domains, improved classifiers would benefit society by reducing spam and fraud. However, for domains such as facial recognition, improved robustness could lead to reduced privacy and other societal harms. See \citet{Albert:2020} for an extensive discussion of the politics of adversarial machine learning.

In other domains, such as epidemiological analysis and land-cover classification, our work may lead to new or better models by reducing the need for labeled data and relaxing the SCAR assumption. As detailed in Section~\ref{sec:Introduction}, only recently has the PU~SCAR barrier been broken~\citep{Bekker:2019,Kato:2019,Sakai:2019}. \acronym{}~learning pushes PU~learning's positive-shift boundary to a new extreme. We hope this paper will enable PU~learning to be applied in domains where existing \bpu{}\textbackslash{}PU~methods are impractical. This could also benefit society if used responsibly, with experts performing proper model validation and vetting risks. Careful model validation is especially important when labeled data is limited and biased.

\begin{ack}
  This work was supported by a grant from the Air Force Research Laboratory and the Defense Advanced Research Projects Agency~(DARPA) -- agreement number FA8750\=/16\=/C\=/0166, subcontract K001892\=/00\=/S05.

  This work benefited from access to the University of Oregon high performance computer, Talapas.
\end{ack}
 {\small
  \bibliography{bib/ref.bib}
}
\bibliographystyle{unsrtnat}
\clearpage
\newgeometry{top=1in, bottom=1in, left=1in, right=1in}
\pagenumbering{arabic}%
\renewcommand*{\thepage}{A\arabic{page}}
\appendix
\SupplementaryMaterialsTitle{}

\section{Nomenclature}\label{sec:App:Nomenclature}

\begin{longtable}{lp{5.7in}}
  \caption{\acronym\ nomenclature reference}\label{tab:App:Nomenclature}
  \\\toprule
  \endfirsthead
  \caption{\acronym\ nomenclature reference (continued)}%
  \\\toprule
  \endhead
  PN              & Positive-negative learning, i.e., ordinary supervised classification \\
  PU              & Positive-unlabeled learning \\
  NU              & Negative-unlabeled learning \\
  uPU             & Unbiased Positive-Unlabeled risk estimator from~\citep{duPlessis:2014}. See Section~\ref{sec:PreviousWork:Vanilla}  \\
  \nnPU{}         & Non\=/negative Positive-Unlabeled risk estimator from~\citep{Kiryo:2017}. See Section~\ref{sec:PreviousWork:Vanilla} \\
  \absPU{}        & Our Absolute\=/value Positive\=/Unlabeled risk estimator. See Section~\ref{sec:AbsValCorrection} \\
  \bpu{}          & Biased-positive, unlabeled learning where the labeled-positive set is not representative of the target positive class. \bpu{} algorithm categories include sample selection bias~\citep{Bekker:2019,Kato:2019} and covariate shift methods~\cite{Sakai:2019} \\
  \acronym{}      & Proposed in this work, arbitrary-positive, unlabeled learning generalizes \bpu{}~learning where the positive training data may be arbitrarily different from the target application's positive-class distribution \\
  \PUc{}          & Positive-Unlabeled Covariate shift algorithm from~\citep{Sakai:2019}. See Section~\ref{sec:PreviousWork:PUc}   \\
  \punu{}         & Our Positive-Unlabeled to Weighted Unlabeled-Unlabeled (two-step) \acronym~learner. See Section~\ref{sec:WUU}\\
  \pupnu{}        & Our Positive-Unlabeled to Arbitrary-Positive, Negative, Unlabeled (two-step) \acronym~learner. See Section~\ref{sec:WUU}\\
  \nnpurpl{}      & Our Positive-Unlabeled Recursive Risk (one-step) \acronym~estimator. See Section~\ref{sec:PURPL} \\
  \absPU{}        & Our Absolute\=/value Positive\=/Unlabeled risk estimator. See Section~\ref{sec:AbsValCorrection} \\
  \nnpuOpt{}      & Version of \nnPU{} used as an empirical baseline.  \nnpuOpt{} considers two classifiers -- one trained with $\uset{\te}$~as the unlabeled set and the other trained with ${\uset{\tr} \cup \uset{\te}}$ as the unlabeled set -- and reports whichever configuration performed better. See Section~\ref{sec:ExpRes} \\
  \abspuOpt{}     & Baseline equivalent of \nnpuOpt{} except risk estimator~$\ersk{\text{\absPU}}$ is used instead of~$\ersk{\text{\nnPU}}$.  See Section~\ref{sec:App:AbsPuVsNnpu:WithShift}  \\
  $\xrv$          & Covariate where ${\xrv \in \domainX}$ \\
  $\yrv$          & Dependent random variable, i.e.,~label, where ${\yrv \in \domainY}$ \\
  $\yhat$         & Predicted label ${\yhat \in \domainY}$ \\
  $\dec$          & Decision function,~$\func{\dec}{\domainX}{\real}$ \\
  $\params$       & Parameter(s) of decision function~$\dec$ \\
  $\decCls$       & Real-valued decision function hypothesis class, i.e.,~${\dec \in \decCls}$ \\
  $\loss$         & Loss function,~$\func{\loss}{\real}{\realnn}$ \\
  $\pjointbase{\domSub}$  & Joint distribution, where ${\domSub \in \set{\tr,\te}}$ for train and test resp.  \\
  $\prd{\domSub}$   & Positive-class prior probability, ${\prd{\domSub} \defeq \pprior{\domSub}}$ where ${\domSub \in \set{\tr,\te}}$ for train \& test resp. \\
  $\plike{\domSub}$ & Positive class-conditional $\plike{\domSub} \defeq \likeBaseFull{\domSub}{\yp}$ where ${\domSub \in \set{\tr,\te}}$ for train \& test resp. \\
  $\nlike{\domSub}$ & Negative class-conditional $\nlike{\domSub} \defeq \likeBaseFull{\domSub}{\yn}$ where ${\domSub \in \set{\tr,\te}}$ for train \& test resp. \\
  $\marg{\domSub}$  & Marginal distribution where ${\marg{\domSub} \defeq \pbase{\domSub}{\X}}$ where ${\domSub \in \set{\tr,\te}}$ for train and test resp. \\
  $\pset{}$       & Labeled (positive) dataset, i.e., ${\pset{} \iidsim \plike{\tr}}$ \\
  $\uset{\tr}$    & Unlabeled dataset sampled from the \textit{training} marginal distribution, i.e.,~${\uset{\tr} \iidsim \marg{\tr}}$ \\
  $\uset{\te}$    & Unlabeled dataset sampled from the \textit{test} marginal distribution, i.e.,~${\uset{\te} \iidsim \marg{\te}}$ \\
  $\cdc$          & Probabilistic classifier,~$\func{\cdc}{\domainX}{[0,1]}$ that approximates~$\ppost{\tr}{\yn}$ \\
  $\cdcCls$       & Function class containing~$\cdc$ \\
  $\nset{}$       & Labeled negative dataset. In PU~learning, ${\nset{} = \emptyset}$ \\
  $\nsetPS{}$     & Surrogate negative set formed by reweighting $\uset{\tr}$ by $\cdc$ \\
  $\risk{}$       & Risk, i.e.,~expected loss, for decision function~$\dec$ and loss~$\loss$, i.e.,~${\risk{} \defeq \expectS{(\xrv,\yrv) \sim \pjointbase{}}{\ldecRV{\yrv}}}$ \\
  $\ersk{}$       & Empirical estimate of risk~$\risk{}$ \\
  $\elblrsk{\gsub}{\yhat}$ & Empirical risk when predicting label ${\yhat \in \domainY}$ on data sampled from some distribution,~$\pbase{\gsub}{\X}$.  See Section~\ref{sec:PreviousWork:Vanilla} \\
  $\lblAbsR{\nsub}{\yhat}$ & Labeled negative risk with absolute\=/value correction.  See Eq.~\eqref{eq:AbsVal:Def} in Section~\ref{sec:AbsValCorrection} \\
  $\trsk{\nuSub}{\yhat}$ & Surrogate negative risk formed by weighting unlabeled set~$\uset{\tr}$ by probabilistic classifier~$\cdc$ where~${\trsk{\nuSub}{\yhat} \defeq \frac{1}{\nTrU} \sum_{\X_i \in \uset{\tr}} \frac{\cdcxI \lbase{\decxI}{\yhat}}{1-\prd{\tr}}}$ \\
  $\Wx$           & Covariate shift importance function based on density-ratio estimation where~${\Wx \defeq \frac{\marg{\te}}{\marg{\tr}}}$ \\
  $\np$           & Size of the labeled (positive) dataset, i.e., ${\np \defeq \abs{\pset{}}}$ \\
  $\nTrU$         & Size of the unlabeled \textit{training} dataset, i.e., ${\nTrU \defeq \abs{\uset{\tr}}}$ \\
  $\nTeU$         & Size of the unlabeled \textit{test} dataset, i.e., ${\nTeU \defeq \abs{\uset{\te}}}$ \\
  $\nTest$        & Size of the inductive test set \\
  $\learner$      & Learning or optimization algorithm \\
  $\learningRate$ & Learning rate hyperparameter, ${\learningRate > 0}$ \\
  $\weightDecay$  & Weight decay hyperparameter, ${\weightDecay \geq 0}$  \\
  $\gradAtten$    & Non-negative gradient attenuator hyperparameter ${\gradAtten \in (0,1]}$.  This hyperparameter is ignored when absolute\=/value correction is used. \\
  $\normal{\boldsymbol{\mu}}{\eye{m}}$ & Multivariate Gaussian (normal) distribution with mean~$\boldsymbol{\mu}$ and $m$\=/dimensional identity covariance.  See Section~\ref{sec:App:AddRes:Synthetic} \\
  $\MyMax{a}$     & ${\defeq \max\{0,a\}}$.  See Section~\ref{sec:App:ERM:PURPL} \\
  $\round{a}$     & Rounds ${a \in \real}$ to the nearest integer. See Section~\ref{sec:App:AddRes:AltStep1} \\
  \bottomrule
\end{longtable}
\section{Proofs}\label{sec:App:Proofs}

\subsection{Proof of Theorem~\ref{thm:AbsCorrection}}

\begin{proof}
  Mild assumptions are made about the behavior of the loss and decision functions; the following conditions match those assumed by~\citet{Kiryo:2017}. Define loss function~$\loss$ as \textit{bounded} over some class of real-valued functions~$\decCls$ (where ${\dec \in \decCls}$) when the following conditions both hold:
  \begin{enumerate}
    \item ${\exists \decBound > 0}$ such that ${\sup_{\dec \in \decCls}\norm{\dec}_{\infty} \leq \decBound}$
    \item ${\exists \lossBound > 0}$ such that ${\sup_{\abs{t} \leq \decBound} \max_{\yhat \in \set{\pm1}} \lbase{t}{\yhat} \leq \lossBound}$ \text{.}
  \end{enumerate}

  \citet{duPlessis:2014} show that
  \begin{equation}
    (1-\prd{})\lblrisk{\nsub}{\yhat} = \lblrisk{\usub}{\yhat} - \prd{} \lblrisk{\psub}{\yhat} \text{.}
  \end{equation}
  \noindent
  Consider the labeled negative\=/valued risk estimator with absolute\=/value correction
  \begin{equation}\label{eq:App:thm:AbsCorrection}
    \lblAbsR{\nsub}{\yhat} = \nnwrap{\elblrsk{\nsub}{\yhat}} \text{.}
  \end{equation}

  An estimator,~$\hat{\theta}_n$, over $n$~samples is consistent w.r.t.\ parameter~$\theta$ if for all ${\epsilon > 0}$ it holds that
  \begin{equation*}
    \lim_{n \rightarrow \infty} \Pr\sbrack{\abs{\hat{\theta}_{n} - \theta} \geq \epsilon} = 0\text{.}
  \end{equation*}
  \noindent
  Let estimator ${\hat{Y} = \sum_{i=1}^{k} \beta_i \hat{\theta}_{(i)}}$ be the weighted sum of $k$~consistent estimators with each constant~${\beta_i \ne 0}$.  Let ${\epsilon > 0}$ be an arbitrary positive constant. If each $\hat{\theta}_{(i)}$ converges to within~${\frac{\epsilon}{k\abs{\beta_i}} > 0}$ of~${\theta_{(i)} \geq 0}$, then ${\hat{Y}}$ converges to within~$\epsilon$ of ${\sum_{i=1}^{k} \beta_i \theta_{(i)}}$. Therefore, to prove the consistency of~$\lblAbsR{\nsub}{\yhat}$ in Eq.~\eqref{eq:App:thm:AbsCorrection}, it is sufficient to show that each of its individual terms is consistent.

  Both $\elblrsk{\psub}{\yhat}$ and $\elblrsk{\usub}{\yhat}$ are empirically estimated directly from a training data set.  Let ${\gsub \in \set{\psub,\usub}}$ and ${\train_{\gsub} \iidsim \pbaseX{\gsub}}$.  For each (independent) ${\xrv \sim \pbaseX{\gsub}}$, $\ldecRV{\yhat}$ is an unbiased estimate of~$\lblrisk{\gsub}{\yhat}$. In addition, ${\ldecRV{\yhat} < \lossBound < \infty}$ implies that ${\var{\ldecRV{\yhat}} <\infty}$.  By Chebyshev's Inequality, $\elblrsk{\gsub}{\yhat}$~is consistent as
  \begin{equation*}
    \lim_{\abs{\train} \rightarrow \infty}\Pr\sbrack{\abs{\frac{1}{\abs{\train}}\sum_{\X_i \in \train} \bigg( \ldecI{\yhat} \bigg) - \lblrisk{\gsub}{\yhat}} \geq \epsilon} < \frac{\var{\ldecRV{\yhat}}}{\abs{\train}\epsilon^2} = 0\text{.}
  \end{equation*}
  \noindent
  Since $\elblrsk{\nsub}{\yhat}$~is the weighted sum of consistent estimators, it is consistent as ${\sizebase{} = \min\set{\np, \nU} \rightarrow \infty}$.

  To show $\lblAbsR{\nsub}{\yhat}$~is consistent, it suffices to show that
  \begin{equation*}
    \lim_{\sizebase{} \rightarrow \infty} \Pr\sbrack{\abs{\lblAbsR{\nsub}{\yhat} - \lblrisk{\nsub}{\yhat}} \geq \epsilon} = 0 \text{.}
  \end{equation*}
  \noindent
  Because $\elblrsk{\nsub}{\yhat}$~is consistent, then as ${\sizebase{} \rightarrow \infty}$ it holds that ${\elblrsk{\nsub}{\yhat} - \epsilon \leq \lblrisk{\nsub}{\yhat} \leq \elblrsk{\nsub}{\yhat} + \epsilon}$.  When ${\elblrsk{\nsub}{\yhat} \geq \lblrisk{\nsub}{\yhat} \geq 0}$,  then ${\lblAbsR{\nsub}{\yhat} = \elblrsk{\nsub}{\yhat}}$ (i.e.,~absolute value has no effect) so
  \begin{equation*}
    0 \leq \lblAbsR{\nsub}{\yhat} - \lblrisk{\nsub}{\yhat} \leq \epsilon \text{.}
  \end{equation*}

  Consider the alternate possibility where ${\elblrsk{\nsub}{\yhat} < \lblrisk{\nsub}{\yhat}}$.  If ${\elblrsk{\nsub}{\yhat} \geq 0}$ or ${\lblrisk{\nsub}{\yhat} = 0}$, then absolute\=/value correction again has no effect on the estimation error (i.e.,~remains~${{\leq}\epsilon}$).  Lastly, when ${\elblrsk{\nsub}{\yhat} < 0}$ and ${\lblrisk{\nsub}{\yhat} > 0}$, the estimation error strictly decreases as
  \begin{align*}
    \err_{\hat{R}} &= \abs{\elblrsk{\nsub}{\yhat} - \lblrisk{\nsub}{\yhat}} \\
                   &= -\elblrsk{\nsub}{\yhat} + \lblrisk{\nsub}{\yhat}  &  \text{Since } \elblrsk{\nsub}{\yhat} < 0 \text{ and } {\lblrisk{\nsub}{\yhat} > 0}\\
                   &= \abs{\elblrsk{\nsub}{\yhat}} + \lblrisk{\nsub}{\yhat} & \text{Again since } \elblrsk{\nsub}{\yhat} < 0\\
                  &= \lblAbsR{\nsub}{\yhat} + \lblrisk{\nsub}{\yhat} < \epsilon
  \end{align*}
  \noindent
  so
  \begin{align}
    \err_{\ddot{R}} &= \abs{\abs{\elblrsk{\nsub}{\yhat}} - \lblrisk{\nsub}{\yhat}} \nonumber\\
                    &\fedeq \abs{\lblAbsR{\nsub}{\yhat} - \lblrisk{\nsub}{\yhat}} \nonumber\\
                    &< \lblAbsR{\nsub}{\yhat} + \lblrisk{\nsub}{\yhat} < \epsilon &  \text{Since } \elblrsk{\nsub}{\yhat} < 0 \text{ and } {\lblrisk{\nsub}{\yhat} > 0} \label{eq:App:Thm:AbsCorrection} \text{.}
  \end{align}

  The above shows that as ${\sizebase{} \rightarrow \infty}$, it always holds that ${\abs{\lblAbsR{\nsub}{\yhat} - \lblrisk{\nsub}{\yhat}} \leq \epsilon}$  for arbitrary ${\epsilon > 0}$ making $\lblAbsR{\nsub}{\yhat}$~consistent.
\end{proof}

\subsection{Proof of Theorem~\ref{thm:TwoStep:Unbiased}}

\begin{proof}
  Consider first the case that ${\cdcx = \ppost{\tr}{\yn}}$:
  \begin{align*}
    \expectS{\uset{\tr} \iidsim \marg{\tr}}{\trsk{\nuSub}{\yhat}} &= \expectS{\uset{\tr} \iidsim \marg{\tr}}{\frac{1}{\nTrU} \sum_{\xrv_i \in \uset{\tr}} \frac{\lbase{\decRVI}{\yhat} \cdcRVI}{1-\prd{\tr}}} \\
                                                         &= \frac{1}{\nTrU} \sum_{i=1}^{\nTrU} \expectS{\xrv \sim \marg{\tr}}{\frac{\ldecRV{\yhat} \cdcRV}{1-\prd{\tr}}} & \text{Linearity of expectation} \\
                                                         &= \expectS{\xrv \sim \marg{\tr}}{\frac{\ldecRV{\yhat} \cdcRV}{1-\prd{\tr}}} \\
                                                         &= \expectS{\xrv \sim \marg{\tr}}{\frac{\ldecRV{\yhat} \ppostRV{\tr}{\yn}}{\nprior{\tr}}} \\
                                                         &= \int_{\X} \ldec{\yhat} \frac{\ppost{\tr}{\yn} \marg{\tr}}{\nprior{\tr}} \\
                                                         &= \expectS{\xrv \sim \nlike{\tr}}{\ldecRV{\yhat}} & \text{Bayes' Rule} \\
                                                         &\fedeq \lblrisk{\trsub\nsub}{\yhat},
  \end{align*}
  \noindent
  satisfying the definition of unbiased.

  Next we consider whether $\trsk{\nuSub}{\yhat}$~is a consistent estimator of~$\lblrisk{\nsub}{\yhat}$.  For the complete definition of PAC~learnability that we use here, see~\citep{FoundationsML}.  We provide a brief sketch of the definition below.

  We assume that true posterior distribution,~$\ppost{\tr}{\yn}$ is in some concept class~$\mathcal{C}$ of functions --- i.e.,~\textit{concepts} --- mapping~${\domainX}$ to~${[0,1]}$.  Let ${\cdc_{\sample} \in \cdcCls}$ be the hypothesis selected by learning algorithm~$\learner$ after being provided a training sample~$\sample$ of size ${\sizebase{} = \min\left\{\np, \nTrU\right\}}$.\footnote{No restrictions are placed on~$\learner$ other than its existence and that selected hypothesis~$\cdc_{\sample}$ satisfies Eq.~\eqref{eq:App:Proof:TwoStep:Convergence}.}  Consider the \textit{realizable} setting so $\mathcal{C}$'s~PAC~learnability entails that for all ${\epsilon,\delta > 0}$, there exists an~$\sizebase{}'$ such that for all ${\sizebase{} > \sizebase{}'}$,
  \begin{equation}\label{eq:App:Proof:TwoStep:Convergence}
    \Pr\bigg[\expectS{\xrv \sim \marg{\tr}}{\abs{\cdc_{\sample}(\xrv) - \pbase{\tr}{\yrv = \yn \vert \xrv}}} > \epsilon\bigg] < \delta.
  \end{equation}
  \noindent
  Therefore, as ${\sizebase{} \rightarrow \infty}$, $\cdc$'s~expected (absolute) error w.r.t.~$\ppost{\tr}{\yn}$ decreases to~0 making $\trsk{\nuSub}{\yhat}$ asymptotically unbiased.  To demonstrate consistency, it is necessary to show that for all ${\epsilon > 0}$:
  \begin{equation*}
    \lim_{\sizebase{} \rightarrow \infty}\Pr\sbrack{\abs{\trsk{\nuSub}{\yhat} - \lblrisk{\trsub\nsub}{\yhat}} > \epsilon} = 0\text{.}
  \end{equation*}
  Let ${\sup_{\abs{t} \leq \norm{\dec}_{\infty}}\lbase{t}{\yhat} \leq \lossBound}$, where $\norm{\dec}_{\infty}$ is the Chebyshev norm of~$\dec$ for~${\X \in \domainX}$. Bounding the loss's magnitude bounds the variance when estimating the surrogate negative risk of ${\xrv \sim \marg{\tr}}$ such that~${\frac{1}{(1 - \prd{\tr})^2}\var{\cdcRV \ldecRV{\yhat}} \leq \varBound}$ where~${\varBound \in \realnn}$ and~${\prd{\tr} \in [0,1)}$.

  Since $\trsk{\nuSub}{\yhat}$ is asymptotically unbiased, then from Chebyshev's inequality for ${\epsilon > 0}$:
  \begin{align*}
    \lim_{\sizebase{} \rightarrow \infty}\Pr\sbrack{\abs{\trsk{\nuSub}{\yhat} - \lblrisk{\trsub\nsub}{\yhat}} \geq \epsilon} &\leq \frac{\var{\trsk{\nuSub}{\yhat}}}{\epsilon^2} \\
                                                                                                                             &= \frac{1}{(1-\prd{\tr})^2\epsilon^2} \sum_{i=1}^{\nTrU} \text{Var}\left( \frac{\cdcRV \lbase{\decRV}{\yhat}}{\nTrU} \right) & \text{Linearity of independent r.v.\ var.} \\
                                                                                        &\leq \frac{\nTrU\varBound}{\nTrU^2 \epsilon^2} \\
                                                                                        &= 0 & \text{L'H\^{o}pital's Rule}\text{.}
  \end{align*}
\end{proof}

\subsection{Proof Regarding Estimating \texorpdfstring{$\prd{\te}$}{the Test Distribution's Positive Class Prior}}

We are not aware of an existing technique to directly estimate the test distribution's positive prior~$\prd{\te}$ given only~$\pset{}$,~$\uset{\tr}$, and~$\uset{\te}$.  We propose the following that uses an additional classifier.

\begin{theorem}\label{thm:App:TwoStep:PriorSample}
  Define ${\uset{} \defeq \set{\X_i}_{i=1}^{\nU} \iidsim \marg{}}$. Let ${\nset{} = \setbuild{\X_i \in \uset{}}{Q_i = 1}}$ be a set where $Q_i$~is a Bernoulli random variable with probability of success~${q_i = \pbase{}{\yrv = \yn \vert \X_i}}$.  Then $\nset{}$ is a SCAR sample w.r.t.\ negative class\=/conditional distribution ${\nlike{} = \condfull{}{\yn}}$.
\end{theorem}

\begin{proof}
  By Bayes' Rule
  \begin{align*}
    \nlike{} &\propto \ppost{}{\yn} \marg{}\\
  \end{align*}
  Each ${\X_{i} \in \uset{}}$ is sampled from~$\marg{}$.   By including $\X_i$~in~$\nset{}$ only if ${Q_i = 1}$, then $\X_i$'s effective sampling probability is ${\pbase{}{\yrv = \yn \vert \X_i} \pbase{}{\X}}$. Bayes' Rule includes prior inverse~$\frac{1}{1 - \prd{}}$, where~${\prd{} = \pprior{}}$; this constant scalar can be ignored since it does not change whether~$\nset{}$ is unbiased, i.e.,~it does not affect relative probability.
\end{proof}

\paragraph{Commentary} Theorem~\ref{thm:App:TwoStep:PriorSample} states the property generally, but consider it over \acronym's~training distribution.  Probabilistic classifier~$\cdc$ is used as a surrogate for~$\ppost{\tr}{\yn}$.  Rather than \textit{soft} weighting the samples like in Theorem~\ref{thm:TwoStep:Unbiased}'s proof, sample inclusion in the negative set is a \textit{hard} ``in\=/or\=/out'' decision.  This does not change the sample's statistical properties, but it allows us to create an unweighted negative set, we denote~$\nset{\tr}$.

By Eq.~\eqref{eq:PreviousWork:NegAssume}'s assumption, $\nset{\tr}$~is representative of samples from the negative class\=/conditional distribution ${\nlike{} = \nlike{\tr} = \nlike{\te}}$.  Given a representative labeled set from the \textit{test distribution}, well-known positive-unlabeled prior estimation techniques~\citep{Ramaswamy:2016,duPlessis:2017} can be used without modification using~$\nset{\tr}$ and~$\uset{\te}$.  Be aware that these PU~prior estimation methods would return the negative-class's prior,~$\nprior{\te}$, while our risk estimators use the positive class's prior,~${\prd{\te} = 1 - \nprior{\te}}$.

We provide empirical results regarding the effect of inaccurate prior estimation's in Section~\ref{sec:App:AddRes:PriorEstimationError}.

\subsection{Proof of Theorem~\ref{thm:PURPL:BiasedConsistent}}

The definition of ``bounded loss'' is identical to the proof of Theorem~\ref{thm:AbsCorrection}.

\begin{proof}
  Consider first whether \nnpurpl{} is unbiased. \citet{duPlessis:2014} observe that the negative labeled risk can be found via decomposition where
  \begin{equation}\label{eq:App:ThmPURPL:NegDecomposition}
    (1-\prd{})\lblrisk{\nsub}{\yhat} = \lblrisk{\usub}{\yhat} - \prd{} \lblrisk{\psub}{\yhat}\text{.}
  \end{equation}
  \noindent
  The positive labeled risk similarly decomposes as
  \begin{equation}\label{eq:App:ThmPURPL:PosDecomposition}
    \prd{}\lblrisk{\psub}{\yhat} = \lblrisk{\usub}{\yhat} - (1 - \prd{}) \lblrisk{\nsub}{\yhat}\text{.}
  \end{equation}
  \noindent
  Applying these decompositions along with Eq.~\eqref{eq:PreviousWork:NegAssume}'s assumption yields an unbiased version of~\nnpurpl:
  \begin{equation}%
    \rskUPURPL = \underbrace{\elblrsk{\tesub\usub}{+} - (1 - \prd{\te})\underbrace{\frac{\elblrsk{\trsub\usub}{+} - \prd{\tr}\elblrsk{\trsub\psub}{+}}{1 - \prd{\tr}}}_{\elblrsk{\tesub\nsub}{+}}}_{\prd{\te}\elblrsk{\tesub\psub}{+}} + (1 - \prd{\te}) \underbrace{\frac{\elblrsk{\trsub\usub}{-} - \prd{\tr}\elblrsk{\trsub\psub}{-}}{1 - \prd{\tr}}}_{\elblrsk{\tesub\nsub}{-}} \text{.}
  \end{equation}

  Since ${\forall_{t} \lbase{t}{} \geq 0}$, it always holds that labeled risk~${\lblrisk{\gsub}{\yhat} \geq 0}$.  When using risk decomposition (i.e.,~Eqs.~\eqref{eq:App:ThmPURPL:NegDecomposition} and~\eqref{eq:App:ThmPURPL:PosDecomposition}) to empirically estimate a labeled risk, it can occur that~${\elblrsk{\gsub}{\yhat} < 0}$.  Absolute-value correction addresses these obviously implausible risk estimates. The unrolled definition of the \nnpurpl~risk estimator with absolute\=/value correction is:
  \begin{equation}\label{eq:App:Proofs:PURPL}
    \rskPURPL = \Bigg\lvert\underbrace{\elblrsk{\tesub\usub}{+} - (1 - \prd{\te}) \bigg\lvert \underbrace{\frac{\elblrsk{\trsub\usub}{+} - \prd{\tr}\elblrsk{\trsub\psub}{+}}{1 - \prd{\tr}}}_{\elblrsk{\tesub\nsub}{+}} \bigg\rvert}_{\prd{\te}\elblrsk{\tesub\psub}{+}}\Bigg\lvert
               + (1 - \prd{\te}) \Bigg\lvert \underbrace{\frac{\elblrsk{\trsub\usub}{-} - \prd{\tr}\elblrsk{\trsub\psub}{-}}{1 - \prd{\tr}}}_{\elblrsk{\tesub\nsub}{-}} \Bigg\rvert \text{.}
  \end{equation}

  Clearly, ${\rskPURPL \geq \rskUPURPL}$.  For $\rskPURPL$ to be unbiased, equality must strictly hold.  This only occurs if the absolute\=/value is never needed, i.e.,~has probability~0 of occurring.

  Next consider whether \nnpurpl{} is consistent.  Theorem~\ref{thm:AbsCorrection} showed that $\lblAbsR{\nsub}{\yhat}$ is consistent.  Following the same logic in Theorem~\ref{thm:AbsCorrection}'s proof, it is straightforward to see that when performing decomposition on~$\lblrisk{\psub}{\yhat}$, $\lblAbsR{\psub}{\yhat}$~is also consistent.

  It follows by induction that \nnpurpl{} (and any similarly\=/defined recursive risk estimator) is consistent.  Theorem~\ref{thm:AbsCorrection} shows the consistency of the base case where both composite terms (e.g.,~$\lblrisk{\usub}{\yhat}$ and~$\lblrisk{\CallComplement}{\yhat}$ in Eq.~\eqref{eq:PURPL:RiskDecomposition}) were estimated directly from training data.  By induction, it is again straightforward from Theorem~\ref{thm:AbsCorrection} that any decomposed term (e.g.,~$\lblrisk{\CallDist}{\yhat}$ in Eq.~\eqref{eq:PURPL:RiskDecomposition}) formed from the sum of consistent estimators must be itself consistent.

  Theorem~\ref{thm:AbsCorrection} further demonstrated that applying absolute\=/value correction does not affect the consistency of a risk estimator. Therefore, any recursive risk estimator with absolute-value correction is consistent.  \nnpurpl{}'s consistency is just a single, specific example of this general property.
\end{proof}
\section{Non-Negativity Correction Empirical Risk Minimization Algorithms}\label{sec:App:ErmAlgs}

\citet{Kiryo:2017}'s non\=/negativity correction algorithm uses the ${\max\set{0,\cdot}}$~term to ensure a plausible risk estimate.  Unlike our simpler absolute-value correction described in Section~\ref{sec:AbsValCorrection}, \citeauthor{Kiryo:2017}'s non\=/negativity correction requires a custom empirical risk minimization (ERM)~procedure.  This section presents the custom ERM~algorithms required if non\=/negativity correction is used for our two\=/step methods and \nnpurpl.

\subsection{Two-Step, Non-Negativity ERM Algorithm}\label{sec:App:ERM:TwoStep}

The weighted-unlabeled, unlabeled~(\wuu) risk estimator with non\=/negativity correction is defined as:
\begin{equation}
  \rskNnWUU \defeq \max\Big\{0, \elblrsk{\tesub\usub}{+} - (1 - \prd{\te}) \trsk{\nuSub}{+}\Big\} +  (1 - \prd{\te}) \trsk{\nuSub}{-} \text{.}
\end{equation}
\noindent
The arbitrary-positive, negative, unlabeled~(\bpnu) risk estimator with non\=/negativity correction is similarly defined as:
\begin{equation}
  \rskNnBPNU \defeq (1 - \pnutrade) \prd{\te} \lblrisk{\psub}{+} + (1 - \prd{\te}) \trsk{\nuSub}{-} + \pnutrade \max\Big\{0,\lblrisk{\tesub{\usub}}{+} - (1 - \prd{\te}) \trsk{\nuSub{}}{+}\Big\} \text{.}
\end{equation}
\noindent
Like their counterparts with absolute\=/value correction, both $\rskNnWUU$ and $\rskNnBPNU$~are consistent estimators.

Algorithm~\ref{alg:App:ERM:TwoStep} shows the custom ERM~framework for~$\rskNnWUU$ and~$\rskNnBPNU$ with integrated ``defitting.'' The algorithm learns parameters~$\params$ for decision function~$\dec$. The non-negativity correction occurs whenever \eqsmall{${\elblrsk{\tesub{\usub}}{+} - (1 - \prd{\te}) \trsk{\nuSub{}}{+} < 0}$} (see line~7). The basic algorithm is heavily influenced by the stochastic optimization algorithm proposed by~\citet{Kiryo:2017}.

\begin{algorithm}[h]
  \caption{\wuu{} and \bpnu{} with non-negativity correction custom ERM~procedure}\label{alg:App:ERM:TwoStep}
\textbf{Input}: Datasets $(\pset{}, \nsetPS{}, \uset{\te})$, hyperparameters~$(\gradAtten, \learningRate)$ and risk estimator~${\ersk{\text{TS}} \in \set{\rskNnWUU, \rskNnBPNU}}$

\textbf{Output}: Decision function $\dec$'s parameters~$\params$

\begin{algorithmic}[1]
  \STATE Select SGD\=/like optimization algorithm~$\learner$
  \WHILE{Stopping criteria not met}
    \STATE Shuffle $(\pset{},\nsetPS{},\uset{\te})$ into $N$~batches
    \FOR{\textbf{each} minibatch ${(\batch{\pset{}}, \batch{\nsetPS{}}, \batch{\uset{\te}})}$}
      \IF{$\elblrsk{\tesub{\usub}}{+} - (1 - \prd{\te}) \trsk{\nuSub{}}{+} < 0$}
        \STATE Set gradient ${-\grad\left(\elblrsk{\tesub{\usub}}{+} - (1 - \prd{\te}) \trsk{\nuSub{}}{+} \right)}$
        \STATE Update $\params$ by $\learner$ with attenuated learning rate~${\gradAtten \learningRate}$\label{alg:TwoStep:ERM:Defit}
      \ELSE
        \STATE Set gradient ${\grad \ersk{\text{TS}}}$
        \STATE Update $\params$ by $\learner$ with default learning rate~$\learningRate$
      \ENDIF
    \ENDFOR
  \ENDWHILE
  \STATE \RETURN $\params$ minimizing validation loss
\end{algorithmic}
 \end{algorithm}

Algorithm~\ref{alg:App:ERM:TwoStep} terminates after a fixed epoch count (see Table~\ref{tab:App:GeneralHyperparams} for the number of epochs used for each dataset). Although not shown in Algorithm~\ref{alg:App:ERM:TwoStep}, the validation loss is measured at the end of each epoch.  The algorithm returns the model parameters with the lowest validation loss.

\subsection{\nnpurpl{} Non-Negativity ERM Algorithm}\label{sec:App:ERM:PURPL}

For readability and compactness, let ${\MyMax{a} \defeq \max\set{0,a}}$.  \nnpurpl{}~with non\=/negativity correction is defined as
\begin{equation}
   \rskNnPURPL \defeq \Bigg[\underbrace{\elblrsk{\tesub\usub}{+} - (1 - \prd{\te})\Bigg[\underbrace{\frac{\elblrsk{\trsub\usub}{+} - \prd{\tr}\elblrsk{\trsub\psub}{+}}{1 - \prd{\tr}}}_{\elblrsk{\tesub\nsub}{+}}\Bigg]_{+}}_{\prd{\te}\elblrsk{\tesub\psub}{+}}\Bigg]_{+} + (1 - \prd{\te}) \Bigg[ \underbrace{\frac{\elblrsk{\trsub\usub}{-} - \prd{\tr}\elblrsk{\trsub\psub}{-}}{1 - \prd{\tr}}}_{\elblrsk{\tesub\nsub}{-}} \Bigg]_{+} \text{.}
\end{equation}
\noindent
Like \eqsmall{$\rskPURPL$}~from Section~\ref{sec:PURPL}, \eqsmall{$\rskNnPURPL$}~is a consistent estimator.

When a risk estimator only has a single term that can be negative (like~\nnPU{},~\eqsmall{$\rskNnWUU$}, and~\eqsmall{$\rskNnBPNU$}), the custom non\=/negativity ERM~framework is relatively straightforward as shown in Algorithm~\ref{alg:App:ERM:TwoStep}.  However, \eqsmall{$\rskNnPURPL$}~has three non\=/negativity corrections --- one of which is nested inside another non\=/negativity correction.

Algorithm~\ref{alg:PURPL:ERM} details \eqsmall{$\rskNnPURPL$}'s custom ERM~procedure with learning rate~$\learningRate$. Each non-negativity correction is individually checked with the ordering critical. The optimizer minimizes risk on positive set~$\pset{}$ by both decreasing \eqsmall{$\elblrsk{\psub}{+}$} and increasing~\eqsmall{$\elblrsk{\psub}{-}$}. In contrast, each unlabeled example’s minimizing risk is uncertain. This creates explicit tension and uncertainty for the optimizer. This enforced trade-off over the best unlabeled risk commonly delays or counteracts unlabeled set overfitting. As such, overfitting is most likely with labeled (positive) data.  When that occurs, \eqsmall{$\elblrsk{\trsub\psub}{-}$} increases significantly making \eqsmall{$\elblrsk{\tesub\nsub}{-}$} most likely to be negative so its non\=/negativity is checked first (line~5). Nested term~\eqsmall{$\elblrsk{\tesub\nsub}{+}$} receives second highest priority since whenever its value is implausible, any term depending on it, e.g.,~\eqsmall{$\elblrsk{\tesub\psub}{+}$}, is meaningless.  By elimination, \eqsmall{$\elblrsk{\tesub\psub}{+}$}~has lowest priority.

Algorithm~\ref{alg:PURPL:ERM} applies non-negativity correction by negating risk~\eqsmall{$\elblrsk{A}{\yhat}$}'s gradient~(see Eq.~\eqref{eq:PURPL:RiskDecomposition}).  This addresses overfitting by ``defitting''~$\dec$.  A large negative gradient can push~$\dec$ into a poor parameter space so hyperparameter~${\gradAtten \in (0,1]}$ limits the amount of correction by attenuating gradient magnitude.

\begin{algorithm}[H]
  \caption{\nnpurpl\ with non-negativity correction custom ERM procedure}\label{alg:PURPL:ERM}
\textbf{Input}: Datasets $(\pset{}, \uset{\tr}, \uset{\te})$  \& hyperparameters~$(\gradAtten,\learningRate)$

\textbf{Output}: Decision function $\dec$'s parameters~$\params$

\begin{algorithmic}[1]
  \STATE Select SGD\=/like optimization algorithm~$\learner$
  \WHILE{Stopping criteria not met}
    \STATE Shuffle $(\pset{},\uset{\tr},\uset{\te})$ into $N$~batches
    \FOR{\textbf{each} minibatch ${(\batch{\pset{}}, \batch{\uset{\tr}}, \batch{\uset{\te}})}$}
    \IF{$\elblrsk{\tesub\nsub}{-} < 0$}
        \STATE Use~$\learner$ to update~$\params$ with $-\gradAtten \learningRate \grad\elblrsk{\tesub\nsub}{-}$
    \ELSIF{$\elblrsk{\tesub\nsub}{+} < 0$}
        \STATE Use~$\learner$ to update~$\params$ with $-\gradAtten \learningRate \grad\elblrsk{\tesub\nsub}{+}$
    \ELSIF{$\elblrsk{\tesub\psub}{+} < 0$}
        \STATE Use~$\learner$ to update~$\params$ with $-\gradAtten \learningRate \grad\elblrsk{\tesub\psub}{+}$
      \ELSE
        \STATE Use~$\learner$ to update~$\params$ with $\learningRate \grad\rskNnPURPL$
      \ENDIF
    \ENDFOR
  \ENDWHILE
  \STATE \RETURN $\params$ minimizing validation loss
\end{algorithmic}
 \end{algorithm}
 
\suppressfloats
\section{Detailed Experimental Setup}\label{sec:App:ExperimentSetup}

This section details the experimental setup used to collect the results in Sections~\ref{sec:ExpRes} and~\ref{sec:App:AddRes}.

\subsection{Reproducing our Experiments}

Our implementation is written and tested in Python~3.6.5 and~3.7.1 using the \texttt{PyTorch}~\citep{PyTorch} neural network framework versions~1.3.1 and~1.4. The source code is available at: \href{\sourceCode}{\sourceCode}.  The repository includes file \texttt{requirements.txt} that details Python package dependency information.

To run the program, invoke:

\hspace{1cm}\texttt{python driver.py \underline{ConfigFile}}

\noindent
where \texttt{ConfigFile} is a \texttt{yaml}-format text file specifying the experimental setup.  Repository folder ``\texttt{src/configs}'' contains the configuration files for the experiments in Sections~\ref{sec:ExpRes},~\ref{sec:App:AddRes:Synthetic}, and~\ref{sec:App:AddRes:MarginalBias}.  Prior probability shifts can be made by modifying the configuration files (see \texttt{yaml} fields \texttt{train\_prior} and \texttt{test\_prior}).

\paragraph{Datasets} Our program automatically retrieves all necessary data.  Synthetic data is generated by the program itself.  Otherwise the dataset is downloaded automatically from the web.  If you have trouble downloading any datasets, please verify that your network/firewall ports are properly configured.

\subsection{Class Definitions}

\subsubsection{Partially and Fully Disjoint Positive Distribution Supports}\label{sec:App:ClassDef:DisjointSupports}

Section~\ref{sec:ExpRes:NonidenticalSupports}'s experimental setups are very similar to \citet{Hsieh:2019}'s experiments for positive, unlabeled, biased-negative learning. We even follow \citeauthor{Hsieh:2019}'s label partitions.  The basic rationale motivating the splits are:
\begin{itemize}
  \item \textbf{MNIST}: Odd~(positive class) vs.\ even~(negative class) digits. Each digit's frequency in the original dataset is approximately~0.1 making each class's target prior~${5 * 0.1 = 0.5}$.
  \item \textbf{20~Newsgroups}: As its name suggests, the 20~Newsgroups dataset consists of 20~disjoint labels. Categories are formed by partitioning those 20~labels into 7~groups based on the corresponding text documents' general theme. Our classes are formed by splitting the categories into two disjoint sets. Specifically, the positive-test class consists of documents with labels~0 to~10 in the original dataset.  The negative class is comprised of documents whose labels in the original dataset are~11\=/19.  This split's actual positive prior probability is approximately~0.56.\footnote{We used the latest version of the 20~Newsgroups dataset with duplicates and cross-posts removed.}

  \item \textbf{CIFAR10}: Inanimate objects~(positive class) vs.\ animals~(negative class).  CIFAR10 is a multiclass dataset with ten labels.  Each label is equally common in the training and test set, i.e.,~has prior~0.1.  Since CIFAR10's positive-test class has exactly four labels (e.g.,~plane, automobile, truck, and ship), the positive-test prior is~${4 * 0.1 = 0.4}$.
\end{itemize}
\noindent
For this experiment set, the distribution shift between train and test is premised on new subclasses emerging in the test distribution (e.g.,~due to novel adversarial attacks or systematic failure to collect data on a positive subpopulation in the original dataset).

\subsubsection{TREC Spam Classification}\label{sec:App:SpamDatasets}

As noted previously, PU~learning has been applied to multiple adversarial domains including opinion spam~\citep{Hernandez:2013,Li:2014,Zhang:2017,Zhang:2019:Malware}.  We use spam classification as a vehicle for testing our method in an adversarial domain.

Clearly, email spam classification is not a scenario where PU~learning would generally be applied.  Labeled data for both classes is generally plentiful (especially at the corporate level), and for most modern email systems, spam classification is a solved problem.  For our purposes, spam email provides a good avenue for demonstrating our methods' performance in an adversarial setting for multiple reasons, including:
\begin{itemize}
  \item The positive class (i.e.,~spam) evolves significantly faster than the negative class (i.e.,~not spam or ``ham'').
  \item Our fixed negative class\=/conditional distribution assumption (i.e.,~Eq.~\eqref{eq:PreviousWork:NegAssume}) will not explicitly hold.  This more closely represents what will be encountered ``in-the-wild.''
  \item Public spam/ham datasets exist eliminating the need to use our own proprietary adversarial learning dataset.
  \item Email dates provide a realistic criteria for partitioning the training and test datasets.
\end{itemize}
\noindent
To be clear, what we propose here is not intended as a plausible, deployable spam classifier.  Rather, we show that our methods apply to real-world adversarial domains.

\paragraph{Dataset Construction} The \underline{T}ext \underline{RE}trieval \underline{C}onference~(TREC) is organized annually be the United States' National Institute of Standards and Technology~(NIST) to support information retrieval research~\citep{TREC}.  In~2005, 2006, and~2007, TREC arranged \href{https://trec.nist.gov/data/spam.html}{annual spam classifier competitions} where they released corpuses of spam and ham (i.e.,~not spam) emails.

As detailed in Table~\ref{tab:app:TrecClasses}, the training set consisted of the TREC~2005 (TREC05)~email dataset\footnote{The raw TREC05~emails can be downloaded from~\url{https://plg.uwaterloo.ca/~gvcormac/treccorpus/}.} while the test set was the TREC~2007 (TREC07)~email dataset\footnote{The raw TREC07 emails can be downloaded from~\url{https://plg.uwaterloo.ca/~gvcormac/treccorpus07/}.}.  Basic statistics for the two datasets appear in Table~\ref{tab:app:TrecStats}.

The two sets of emails come from different domains.  TREC05's ham emails derive largely from the Enron dataset.  In contrast, TREC07's emails were received by a particular server between April and July~2007. Many of the ham emails were received by the University of Waterloo where the datasets were curated.

Due to the extended time required to encode all emails using the ELMo~embedder (see Section~\ref{sec:App:TrecRepresentation}), we consider the first 10,000~emails from each dataset as defined by the dataset's \mbox{\texttt{full/index}} file.

\begin{table}[ht]
  \centering
  \caption{TREC05 \& TREC07 dataset statistics}\label{tab:app:TrecStats}
  \begin{tabular}{@{}lrr@{}}
    \toprule
                     & TREC05                & TREC07 \\\midrule
    Dataset Size     & 92,189                & 75,419 \\
    Approx. \% Spam  & \textasciitilde 57\%  & \textasciitilde 66\%  \\
    \bottomrule
  \end{tabular}
\end{table}

\begin{table}[ht]
  \centering
  \caption{TREC spam email classification datasets}\label{tab:app:TrecClasses}
  \begin{tabular}{@{}ll@{}}
    \toprule
    Class       & Definition \\\midrule
    Pos.\ Train & TREC05 Spam \\
    Neg.\ Train & TREC05 Ham \\
    Pos.\ Test  & TREC07 Spam \\
    Neg.\ Test  & TREC07 Ham \\
    \bottomrule
  \end{tabular}
\end{table}

\subsubsection{Identical Positive Supports with Bias}

Table~\ref{tab:App:ClassDef:LIBSVM} defines the positive and negative classes for the 10~LIBSVM datasets used in Section~\ref{sec:App:AddRes:MarginalBias}.  Label~``${+1}$'' always corresponded to the positive class. In two-class (binary) datasets, the other label was the negative class.  For multiclass datasets (e.g.,~connect4), whichever other class had the most examples was used as the negative class.

\begin{table}[ht]
  \centering
  \caption{Positive \& negative class definitions for the LIBSVM datasets in Section~\ref{sec:App:AddRes:MarginalBias}}\label{tab:App:ClassDef:LIBSVM}
  \begin{tabular}{@{}lrrr@{}}
    \toprule
    \multirow{2}{*}{Dataset} & \multirow{2}{*}{$d$} & \multirow{2}{*}{\shortstack[l]{Pos. \\ Class}} & \multirow{2}{*}{\shortstack[l]{Neg. \\ Class}} \\
                &         &      & \\
    \midrule
    banana      & 2       &  +1  &   2  \\
    cod-rna     & 8       &  +1  & --1  \\
    susy        & 18      &  +1  &   0  \\
    ijcnn1      & 22      &  +1  & --1  \\
    covtype.b   & 54      &  +1  &   2  \\
    phishing    & 68      &  +1  &   0  \\
    a9a         & 123     &  +1  & --1  \\
    connect4    & 126     &  +1  & --1  \\
    w8a         & 300     &  +1  & --1  \\
    epsilon     & 2,000   &  +1  & --1  \\
    \bottomrule
  \end{tabular}
\end{table}

\subsection{\bpu~Selection Bias Invariance of Order}

Section~\ref{sec:App:AddRes:PUSB}'s experiments follow the invariance of order assumption as proposed and implemented by \citet{Kato:2019}.  Their original experiments considered the MNIST dataset.  For completeness, we expand our comparison to their method to also include the MNIST~variants, FashionMNIST~\citep{FashionMNIST} and KMNIST~\citep{KMNIST}. Like MNIST, both FashionMNIST and KMNIST are multiclass datasets consisting of 10~disjoint labels.  As described in Section~\ref{sec:App:ClassDef:DisjointSupports}, binary classes are formed by partitioning the original set of labels.

As before, MNIST splits the labels between odds~(positive class) and evens~(negative class).  For consistency, we used the same odd/even label partition for FashionMNIST and KMNIST. Note that such a partitioning lacks a corresponding semantic meaning for those two datasets.

\subsection{Training, Validation, and~Test Set Sizes}\label{app:DatasetSizes}

Table~\ref{tab:App:DatasetSizes:Default} lists the default size of each dataset's positive~($\pset{}$), unlabeled train~($\uset{\tr}$), unlabeled test~($\uset{\te}$), and inductive test sets. All LIBSVM datasets (e.g.,~susy, a9a, etc.\ in Section~\ref{sec:App:AddRes:MarginalBias}) used the dataset sizes defined by~\citet{Sakai:2019}.  The separate validation set -- made up of only positive and unlabeled examples -- was one-fifth Table~\ref{tab:App:DatasetSizes:Default}'s training set sizes.  Each learner observed identical dataset splits in each trial.

\begin{table}[t]
  \caption{Each dataset's default training set sizes.  LIBSVM denotes all datasets downloaded directly from~\citep{LIBSVM} and used in Section~\ref{sec:App:AddRes:MarginalBias}. All quantities in the table do \textit{not} include the validation set.}\label{tab:App:DatasetSizes:Default}
  \centering
\setlength{\tabcolsep}{9.5pt}

\newcommand{\pusbarrows}{\multicolumn{2}{c}{\small $\leftarrow$~See Sec.~\ref{sec:App:AddRes:PUSB}~~$\rightarrow$}}

\begin{tabular}{@{}lrrrr@{}}
  \toprule
  Dataset        & $\np$      & $\nTrU$    & $\nTeU$    & $\nTest$ \\\midrule
  MNIST          & 1,000      & 5,000      & 5,000      & 5,000   \\
  20 Newsgroups  & 500        & 2,500      & 2,500      & 5,000   \\
  CIFAR10        & 1,000      & 5,000      & 5,000      & 3,000   \\\hdashline
  TREC Spam      & 500        & 1,250      & 1,250      & 1,000   \\\hdashline
  Synthetic      & 1,000      & 1,000      & 1,000      & N/A     \\\hdashline
  LIBSVM         & 250        & 583        & 583        & 2,000   \\\hdashline
  FashionMNIST   & 833        & \pusbarrows             & 5,000   \\
  KMNIST         & 833        & \pusbarrows             & 5,000   \\
  \bottomrule
\end{tabular}
 \end{table}

Special inductive test set sizes were needed for two of Section~\ref{sec:ExpRes:NonidenticalSupports}'s disjoint positive-support experiments.  To understand why, consider the MNIST disjoint-support experiment (i.e.,~the fourth MNIST row in Table~\ref{tab:ExpRes:Nonoverlap}) where the negative class~(N) is comprised of labels~$\set{0,2}$ and the positive-test class~(\pDsTest) is composed of labels~$\set{5,7}$.  Each label has approximately 1,000~examples in the dedicated test set meaning there are approximately 4,000~total test examples between the negative and positive classes.  However, MNIST's default inductive test set size~($\nTest$) is~5,000 (see Table~\ref{tab:App:DatasetSizes:Default}).  Rather than duplicating test set examples, we reduced MNIST's~$\nTest$ to~1,500 \textit{for the disjoint positive-support experiments only}. 20~Newsgroups has the same issue so its disjoint-positive support~$\nTest$ was also reduced as specified in Table~\ref{tab:App:DatasetSizes:DisjointSupport}.  To be clear, for all other datasets and experimental setups in Sections~\ref{sec:ExpRes:NonidenticalSupports},~\ref{sec:ExpRes:Spam}, \ref{sec:App:AddRes:Synthetic}, and~\ref{sec:App:AddRes:MarginalBias}, Table~\ref{tab:App:DatasetSizes:Default} applies.

\begin{table}[ht]
  \centering
  \caption{Smaller MNIST and 20~Newsgroups inductive test set sizes, i.e.,~$\nTest$, used in the disjoint-support experiments.}\label{tab:App:DatasetSizes:DisjointSupport}
  \begin{tabular}{@{}lr@{}}
    \toprule
    Dataset        & $\nTest$ \\\midrule
    MNIST          & 3,000    \\
    20 Newsgroups  & 1,500    \\
    \bottomrule
  \end{tabular}
\end{table}

MNIST, 20~Newsgroups, and CIFAR10 have predefined test sets, which we exclusively used to collect the inductive results.  They were not used for training or validation.  Only some LIBSVM datasets have dedicated test sets, and for those that do, \citet{Sakai:2019} do not specify whether the test set was held out in their experiments. When applicable, we merge the LIBSVM train and test datasets together as if there was only a single monolithic training set.  $\pset{}$,~$\uset{\tr}$,~$\uset{\te}$~and the inductive test set are independently sampled at random from this monolithic set without replacement.

Since the \PUc~formulation is convex, \citeauthor{Sakai:2019} train their final model on the combined training and validation set.

\subsection{CIFAR10 Image Representation}\label{sec:App:ExpSetup:Representation:CIFAR}

Each CIFAR10~\citep{CIFAR10} image is 32~pixels by 32~pixels with three~(RGB) color channels (3,072~dimensions total). \PUc{} specifies a convex model so it cannot be used to train (non\=/convex) deep convolutional networks directly.  To ensure a meaningful comparison, we leveraged the DenseNet\=/121 deep convolutional network architecture pretrained on 1.2~million images from ImageNet~\citep{DenseNet}. The network's (linear) classification layer was removed, and the experiments used the 1,024-dimension feature vector output by DenseNet's convolutional backbone.

\subsection{20~Newsgroups Document Representation}\label{sec:App:20NewsRepresentation}

The 20~Newsgroups dataset is a collection of internet discussion board posts. The original dataset consisted of 20,000 documents~\citep{20newsgroups}; it was pruned to 18,828~documents in~2007 after removal of duplicates and cross-posts~\citep{Rennie:2007}. This latest dataset has a predefined split of 11,314~train and 7,532~test documents. Similar to~CIFAR10, we use transfer learning to create a richer representation of each document.

Classic word embedding models like GloVe and Word2Vec yield token representations that are independent of context. Proposed by \citet{Peters:2018}, ELMo (\underline{e}mbeddings for \underline{l}anguage \underline{mo}dels) enhances classic word embeddings by making the token representations context dependent.  We use ELMo to encode each 20~Newsgroup document as described below.

ELMo's embedder consists of three sequential layers --- first a character convolutional neural network (CNN) provides subword information and improves unknown word robustness. The CNN's output is then fed into a two-layer, bidirectional LSTM\@.  The output from each of ELMo's layers is a 1,024\=/dimension vector.  For a token stream of length~$m$, the output of ELMo's embedder would be a tensor of size~${\langle \textnormal{\#Layers} \times d_{\textnormal{layer}} \times \textnormal{\#Tokens} \rangle}$ --- in this case~${\langle 3 \times 1024 \times m \rangle}$.

Like~\citet{Hsieh:2019} who used this encoding scheme for positive, unlabeled, biased-negative~(PUbN)~(PUbN) learning, we used \citet{Ruckle:2018}'s sentence representation encoding scheme, which takes the minimum, maximum, and average value along each ELMo~layer's output dimension.  The dimension of the resulting document encoding is:
\begin{equation*}
  \abs{\set{\max,\min,\textnormal{avg}}} \cdot \textnormal{\#Layers} \cdot d_{\textnormal{layer}} = 3 \cdot 3 \cdot 1024 = 9,216\text{.}
\end{equation*}

When documents are encoded serially, each document implicitly contains information about all preceding documents. Put simply, the order documents are processed affects each document's final encoding.  For consistency, all 20~Newsgroups experiments used a single identical encoding for all learners.

The Allen Institute for Artificial Intelligence has published multiple pretrained ELMo~models.  We used the ELMo model trained on a 5.5~billion token corpus --- 1.9~billion from Wikipedia and 3.6~billion from a news crawl.  We chose this version because ELMo's developers report that it was the best performing.

\subsection{TREC Email Representation}\label{sec:App:TrecRepresentation}

The TREC05 and TREC07 emails are encoded using the ELMo embedder identical to 20~Newsgroups. See Section~\ref{sec:App:20NewsRepresentation} above for the details.

\subsection{Models and Hyperparameters}\label{sec:App:ExpSetup:Hyperparameters}

This section reviews the experiments' hyperparameter methodology.

As specified by its authors, \PUc's hyperparameters were tuned via importance-weighted cross validation (IWCV)~\citep{Sugiyama:Krauledat:2007}. \PUc's author\=/supplied implementation includes a built-in hyperparameter tuning architecture that we used without modification.

Our hyperparameters and best\=/epoch weights were selected using the validation loss (using the associated risk estimation) on a validation set. Our experiments' hyperparameters can be grouped into two categories.  First, some hyperparameters~(e.g.,~number of epochs) apply to most/all learners (excluding \PUc).  The second category's hyperparameters are individualized to each learner and were used for all of that learner's experiments on the corresponding dataset.

Table~\ref{tab:App:GeneralHyperparams} enumerates the general hyperparameter settings that applied to most/all learners.  Batch sizes were selected based on the dataset sizes (see Tables~\ref{tab:App:DatasetSizes:Default} and~\ref{tab:App:DatasetSizes:DisjointSupport}) while the epoch count was determined after monitoring the typical time required for the best validation loss to stop (meaningfully) changing. A grid search was used to select each dataset's layer count; we specifically searched set~$\set{1, 2,3}$ for $\dec$ and $\set{0,1,2}$ for~$\cdc$. With the exception of the output layer, each linear layer used ReLU~activation and batch normalization~\citep{Ioffe:2015}. The selected layer count minimized the median validation loss across all learners.

Tables~\ref{tab:App:Hyperparams:Ours},~\ref{tab:App:Hyperparams:nnPU}, and~\ref{tab:App:Hyperparams:PN} enumerate the final hyperparameter settings for our models, nnPU, and the positive-negative (PN)~learners respectively.  The selected hyperparameter setting had the best average validation loss across 10~independent trials. We also used a grid search for these parameters.  The search space was: learning rate ${\learningRate \in \set{10^{-5}, 10^{-4}, 10^{-3}}}$, weight decay ${\weightDecay \in \set{10^{-4}, 10^{-3}, 5\cdot10^{-3}, 10^{-2}, 10^{-1}}}$, and (where applicable) gradient attenuator ${\gradAtten \in \set{0.1, 0.5, 1.0}}$\footnote{Hyperparameter~$\gradAtten$ only applies when using \citet{Kiryo:2017}'s non\=/negativity correction.  $\gradAtten$~is not considered by our absolute\=/value correction.}.

By monitoring the (implausible) validation loss during Step~\#1, we observed overfitting when using the rich ELMo~representations for the 20~Newsgroups and TREC email datasets.  To address this, we added a dropout layer (with probability ${p = 0.5}$)  before the input to each linear (i.e.,~fully-connected) layer.  It is uncommon to use dropout even on the input dimension. However, we deliberately made this choice to still allow dropout even if we use a strictly linear\=/in\=/parameter model.  Dropout was not used for any other dataset.

\begin{table}[ht]
  \centering
  \caption{General hyperparameter settings}\label{tab:App:GeneralHyperparams}
\setlength{\tabcolsep}{9.5pt}

\newcommand{\batchsizearrow}[1]{\multicolumn{3}{c}{$\longleftarrow$~~#1~~$\longrightarrow$}}
\newcommand{\secPUSBArrow}{\batchsizearrow{See Section~\ref{sec:App:AddRes:PUSB}}}

\begin{tabular}{@{}lrrrrrrl@{}}
  \toprule
  \multirow{2}{*}{Dataset} & \multirow{2}{*}{\#Epoch} & \multicolumn{2}{c}{Layer Count} & \multicolumn{3}{c}{Batch Size} & \multirow{2}{*}{Dropout?} \\\cmidrule(lr){3-4}\cmidrule(lr){5-7}
                 &     & $\decx$ & $\cdcx$ & $\decx$ & $\cdcx$ & \PNte{} \\\midrule
  MNIST          & 200 & 3   & 1   & 5,000   & 5,000  & 4,000 &              \\
  20~Newsgroups  & 200 & 1   & 1   & 5,000   & 2,500  & 2,000 & $\checkmark$ \\
  CIFAR10        & 200 & 2   & 1   & 10,000  & 2,500  & 1,500 &              \\\hdashline
  TREC Spam      & 200 & 1   & 0   & 1,000   & 1,000  & 1,000 & $\checkmark$ \\\hdashline
  Synthetic      & 100 & N/A & N/A & 2,000   & 750    & 500   &              \\\hdashline
  banana         & 500 & 3   & 2   & 500     & 750    & 500   &              \\
  cod-rna        & 500 & 2   & 1   & 500     & 750    & 500   &              \\
  susy           & 500 & 2   & 2   & 500     & 750    & 500   &              \\
  ijcnn1         & 500 & 2   & 2   & 500     & 750    & 500   &              \\
  covtype.b      & 500 & 3   & 1   & 500     & 750    & 500   &              \\
  phishing       & 500 & 2   & 2   & 500     & 750    & 500   &              \\
  a9a            & 500 & 2   & 2   & 500     & 750    & 500   &              \\
  connect4       & 500 & 2   & 1   & 500     & 750    & 500   &              \\
  w8a            & 500 & 2   & 1   & 500     & 750    & 500   &              \\
  epsilon        & 500 & 1   & 0   & 500     & 750    & 500   &              \\\hdashline
  FashionMNIST   & 200 & 3   & 1   & \secPUSBArrow            &              \\
  KMNIST         & 200 & 3   & 1   & \secPUSBArrow            &              \\
  \bottomrule
\end{tabular}
 \end{table}

\begin{table}[ht]
  \centering
  \caption{Dataset-specific hyperparameter settings for our \acronym{}~learners.  Hyperparameter~$\gradAtten*$ only applies when using \citet{Kiryo:2017}'s non\=/negativity correction instead of our absolute\=/value correction.}\label{tab:App:Hyperparams:Ours}
  {\small
\setlength{\tabcolsep}{8.5pt}

\begin{tabular}{@{}lrrrrrrrrrrrr@{}}
\toprule
  \multirow{2}{*}{Dataset}  &  \multicolumn{3}{c}{\nnpurpl}  &  \multicolumn{3}{c}{$\cdc$}  &  \multicolumn{3}{c}{\bpnu}  &  \multicolumn{3}{c}{\wuu}\\\cmidrule(lr){2-4}\cmidrule(lr){5-7}\cmidrule(lr){8-10}\cmidrule(l){11-13}
                & $\learningRate$ & $\weightDecay$  & $\gradAtten^{*}$ & $\learningRate$ & $\weightDecay$  & $\gradAtten^{*}$ & $\learningRate$ & $\weightDecay$  & $\gradAtten^{*}$ & $\learningRate$ & $\weightDecay$ & $\gradAtten^{*}$ \\ \midrule
  MNIST         & \sci{1}{-3} & \sci{1}{-3} & \oneP{}   & \sci{1}{-3} & \sci{5}{-3} & \oneP{}  & \sci{1}{-3} & \sci{1}{-3} & \oneP{} & \sci{1}{-4} & \sci{5}{-3} & \oneP{}  \\
  20~Newsgroups & \sci{1}{-4} & \sci{1}{-4} & 0.5       & \sci{1}{-3} & \sci{5}{-3} & \oneP{}  & \sci{1}{-4} & \sci{1}{-4} & 0.5     & \sci{1}{-4} & \sci{1}{-4} & 0.5      \\
  CIFAR10       & \sci{1}{-3} & \sci{1}{-3} & \oneP{}   & \sci{1}{-3} & \sci{5}{-3} & \oneP{}  & \sci{1}{-3} & \sci{1}{-4} & 0.5     & \sci{1}{-3} & \sci{1}{-2} & 0.5      \\\hdashline
  TREC Spam     & \sci{1}{-3} & \sci{1}{-2} & \oneP{}   & \sci{1}{-3} & \sci{1}{-1} & \oneP{}  & \sci{1}{-3} & \sci{1}{-3} & 0.5     & \sci{1}{-3} & \sci{1}{-2} & 0.5      \\\hdashline
  Synthetic     & \sci{1}{-2} &           0 & \oneP{}   & \sci{1}{-2} &           0 & \oneP{}  & \sci{1}{-2} &           0 & \oneP{} & \sci{1}{-2} &           0 & \oneP{}  \\\hdashline
  banana        & \sci{1}{-4} & \sci{1}{-3} & 0.1       & \sci{1}{-4} & \sci{5}{-3} & \oneP{}  & \sci{1}{-5} & \sci{1}{-3} & 0.5     & \sci{1}{-3} & \sci{1}{-3} & 0.1      \\
  cod\_rna      & \sci{1}{-4} & \sci{1}{-3} & 0.5       & \sci{1}{-3} & \sci{1}{-4} & \oneP{}  & \sci{1}{-3} & \sci{1}{-3} & 0.1     & \sci{1}{-4} & \sci{1}{-3} & 0.5      \\
  susy          & \sci{1}{-5} & \sci{1}{-2} & 0.5       & \sci{1}{-4} & \sci{5}{-3} & \oneP{}  & \sci{1}{-5} & \sci{1}{-3} & 0.1     & \sci{1}{-5} & \sci{1}{-4} & 0.5      \\
  ijcnn1        & \sci{1}{-4} & \sci{1}{-3} & 0.5       & \sci{1}{-4} & \sci{5}{-3} & \oneP{}  & \sci{1}{-4} & \sci{1}{-2} & 0.5     & \sci{1}{-4} & \sci{1}{-2} & 0.5      \\
  covtype.b     & \sci{1}{-5} & \sci{1}{-3} & \oneP{}   & \sci{1}{-3} & \sci{1}{-4} & \oneP{}  & \sci{1}{-5} & \sci{1}{-3} & 0.1     & \sci{1}{-4} & \sci{1}{-3} & \oneP{}  \\
  phishing      & \sci{1}{-5} & \sci{1}{-3} & 0.5       & \sci{1}{-3} & \sci{1}{-4} & \oneP{}  & \sci{1}{-5} & \sci{1}{-3} & 0.5     & \sci{1}{-5} & \sci{1}{-3} & 0.5      \\
  a9a           & \sci{1}{-5} & \sci{1}{-4} & \oneP{}   & \sci{1}{-4} & \sci{5}{-3} & \oneP{}  & \sci{1}{-5} & \sci{1}{-4} & 0.5     & \sci{1}{-4} & \sci{1}{-3} & 0.5      \\
  connect4      & \sci{1}{-4} & \sci{1}{-3} & 0.5       & \sci{1}{-3} & \sci{1}{-4} & \oneP{}  & \sci{1}{-4} & \sci{1}{-4} & 0.5     & \sci{1}{-3} & \sci{1}{-2} & 0.5      \\
  w8a           & \sci{1}{-5} & \sci{1}{-4} & 0.5       & \sci{1}{-3} & \sci{1}{-4} & \oneP{}  & \sci{1}{-5} & \sci{1}{-3} & 0.5     & \sci{1}{-5} & \sci{1}{-2} & 0.5      \\
  epsilon       & \sci{1}{-5} & \sci{1}{-2} & 0.1       & \sci{1}{-3} & \sci{1}{-4} & \oneP{}  & \sci{1}{-5} & \sci{1}{-2} & 0.1     & \sci{1}{-4} & \sci{1}{-2} & 0.1      \\\hdashline
  FashionMNIST  & \sci{1}{-3} & \sci{1}{-3} & \oneP{}   & \sci{1}{-3} & \sci{5}{-3} & \oneP{}  & \sci{1}{-3} & \sci{1}{-3} & \oneP{} & \sci{1}{-4} & \sci{5}{-3} & \oneP{}  \\
  KMNIST        & \sci{1}{-3} & \sci{1}{-3} & \oneP{}   & \sci{1}{-3} & \sci{5}{-3} & \oneP{}  & \sci{1}{-3} & \sci{1}{-3} & \oneP{} & \sci{1}{-4} & \sci{5}{-3} & \oneP{}  \\
  \bottomrule
\end{tabular}
   }
\end{table}

\begin{table}[ht]
  \centering
  \caption{Dataset-specific hyperparameter settings for \nnPU{}.}\label{tab:App:Hyperparams:nnPU}
  {\small
\setlength{\tabcolsep}{8.5pt}

\begin{tabular}{@{}lrrrrrr@{}}
\toprule
  \multirow{2}{*}{Dataset}  &  \multicolumn{3}{c}{\nnpuAll}  &  \multicolumn{3}{c}{\nnpuTE}\\\cmidrule(lr){2-4}\cmidrule(l){5-7}
  & $\learningRate$ & $\weightDecay$ & $\gradAtten$ & $\learningRate$ & $\weightDecay$ & $\gradAtten$   \\ \midrule
  MNIST         & \sci{1}{-3} & \sci{1}{-3} & 0.5       & \sci{1}{-3} & \sci{1}{-3} & 0.5       \\
  20~Newsgroups & \sci{1}{-3} & \sci{1}{-3} & 0.5       & \sci{1}{-3} & \sci{1}{-2} & 0.5       \\
  CIFAR10       & \sci{1}{-4} & \sci{1}{-3} & 0.1       & \sci{1}{-4} & \sci{1}{-3} & 0.1       \\\hdashline
  TREC Spam     & \sci{1}{-3} & \sci{1}{-2} & 0.1       & \sci{1}{-3} & \sci{1}{-2} & 0.1       \\\hdashline
  Synthetic     & \sci{1}{-2} &           0 & \oneP{}   & \sci{1}{-2} &           0 & \oneP{}   \\\hdashline
  banana        & \sci{1}{-3} & \sci{1}{-3} & \oneP{}   & \sci{1}{-4} & \sci{1}{-3} & 0.5       \\
  cod\_rna      & \sci{1}{-3} & \sci{1}{-3} & 0.5       & \sci{1}{-3} & \sci{1}{-3} & 0.5       \\
  susy          & \sci{1}{-5} & \sci{1}{-2} & 0.1       & \sci{1}{-3} & \sci{1}{-3} & 0.5       \\
  ijcnn1        & \sci{1}{-3} & \sci{1}{-2} & 0.5       & \sci{1}{-3} & \sci{1}{-3} & 0.5       \\
  covtype.b     & \sci{1}{-3} & \sci{1}{-2} & 0.5       & \sci{1}{-3} & \sci{1}{-2} & 0.5       \\
  phishing      & \sci{1}{-3} & \sci{1}{-2} & 0.5       & \sci{1}{-3} & \sci{1}{-2} & 0.5       \\
  a9a           & \sci{1}{-3} & \sci{1}{-2} & \oneP{}   & \sci{1}{-3} & \sci{1}{-3} & 0.5       \\
  connect4      & \sci{1}{-3} & \sci{1}{-3} & 0.1       & \sci{1}{-3} & \sci{1}{-4} & \oneP{}   \\
  w8a           & \sci{1}{-3} & \sci{1}{-3} & 0.5       & \sci{1}{-3} & \sci{1}{-3} & 0.5       \\
  epsilon       & \sci{1}{-3} & \sci{1}{-3} & 0.5       & \sci{1}{-3} & \sci{1}{-3} & 0.5       \\\hdashline
  FashionMNIST  & \sci{1}{-3} & \sci{1}{-3} & 0.5       & \sci{1}{-3} & \sci{1}{-3} & 0.5       \\
  KMNIST        & \sci{1}{-3} & \sci{1}{-3} & 0.5       & \sci{1}{-3} & \sci{1}{-3} & 0.5       \\
  \bottomrule
\end{tabular}
   }
\end{table}

\begin{table}[ht]
  \centering
  \caption{Dataset-specific hyperparameter settings for the positive-negative~(PN) learners}\label{tab:App:Hyperparams:PN}
  {\small
\setlength{\tabcolsep}{8.5pt}

\begin{tabular}{@{}lrrrr@{}}
\toprule
  \multirow{2}{*}{Dataset}  &  \multicolumn{2}{c}{\PNte}  &  \multicolumn{2}{c}{\PNtr}\\\cmidrule(lr){2-3}\cmidrule(l){4-5}
                & $\learningRate$ & $\weightDecay$ & $\learningRate$ & $\weightDecay$   \\ \midrule
  MNIST         & \sci{1}{-3} & \sci{1}{-3} & \sci{1}{-3} & \sci{1}{-3}  \\
  20~Newsgroups & \sci{1}{-3} & \sci{1}{-3} & \sci{1}{-3} & \sci{1}{-2}  \\
  CIFAR10       & \sci{1}{-4} & \sci{1}{-3} & \sci{1}{-3} & \sci{1}{-2}  \\\hdashline
  TREC Spam     & \sci{1}{-3} & \sci{1}{-2} & \sci{1}{-3} & \sci{1}{-2}  \\\hdashline
  Synthetic     & \sci{1}{-2} &          0  & \sci{1}{-2} &           0  \\\hdashline
  banana        & \sci{1}{-4} & \sci{1}{-2} & \sci{1}{-4} & \sci{1}{-3}  \\
  cod\_rna      & \sci{1}{-3} & \sci{1}{-4} & \sci{1}{-3} & \sci{1}{-4}  \\
  susy          & \sci{1}{-4} & \sci{1}{-2} & \sci{1}{-5} & \sci{1}{-2}  \\
  ijcnn1        & \sci{1}{-3} & \sci{1}{-3} & \sci{1}{-3} & \sci{1}{-2}  \\
  covtype.b     & \sci{1}{-3} & \sci{1}{-2} & \sci{1}{-3} & \sci{1}{-2}  \\
  phishing      & \sci{1}{-3} & \sci{1}{-3} & \sci{1}{-3} & \sci{1}{-2}  \\
  a9a           & \sci{1}{-5} & \sci{1}{-2} & \sci{1}{-3} & \sci{1}{-3}  \\
  connect4      & \sci{1}{-3} & \sci{1}{-2} & \sci{1}{-3} & \sci{1}{-3}  \\
  w8a           & \sci{1}{-4} & \sci{1}{-4} & \sci{1}{-4} & \sci{1}{-3}  \\
  epsilon       & \sci{1}{-4} & \sci{1}{-3} & \sci{1}{-3} & \sci{1}{-3}  \\\hdashline
  FashionMNIST  & \sci{1}{-3} & \sci{1}{-3} & \sci{1}{-3} & \sci{1}{-3}  \\
  KMNIST        & \sci{1}{-3} & \sci{1}{-3} & \sci{1}{-3} & \sci{1}{-3}  \\
  \bottomrule
\end{tabular}
   }
\end{table}

\suppressfloats
\clearpage
 
\clearpage
\section{Additional Experimental Results}\label{sec:App:AddRes}

This section includes experiments we consider insightful but for which there was insufficient space to include in the paper's main body.  With the exception of the synthetic data experiments (see Section~\ref{sec:App:AddRes:Synthetic}) which focus on visually illustrative examples to build intuitions, performance evaluation is based on the inductive misclassification rate since it approximates the expected zero\=/one loss for an unseen example.

\subsection{Illustration using Synthetic Data}\label{sec:App:AddRes:Synthetic}

This section uses synthetic data to visualize scenarios where our algorithms succeed in spite of challenging conditions.

For simplicity, $\cdc{}$~and $\dec$~are linear\=/in\=/parameter models optimized by L\=/BFGS\@.  \PUc\ also trains a linear\=/in\=/parameter models without Gaussian kernels.  Since all methods use the same classifier architecture, our methods' performance advantage comes solely from algorithmic design.

Synthetic data were generated from multivariate Gaussians~$\normal{\boldsymbol{\mu}}{\eye{2}}$ with different means~$\boldsymbol{\mu}$ and identity covariance~$\eye{2}$. In all experiments, the positive-test and negative class\=/conditional distributions were
{\small%
\begin{align*}
  \plike{\te} &= \SynFrac{1}{2}\normal{\SynVect{{-}2}{-1}}{\SynVariance}
                    +\SynFrac{1}{2}\normal{\SynVect{{-}2}{\hphantom{-}1}}{\SynVariance}\\
  \nlike{} &=   \SynFrac{1}{2}\normal{\SynVect{\hphantom{-}2}{-1}}{\SynVariance}
                +\SynFrac{1}{2}\normal{\SynVect{\hphantom{-}2}{\hphantom{-}1}}{\SynVariance} \text{.}
\end{align*}%
}%
\noindent
${\prd{\te} = \prd{\tr} = 0.5}$ makes the ideal \textit{test} decision boundary ${x_1 = 0}$. The datasets in Figure~\ref{fig:App:AddRes:SyntheticData} vary only in the positive-train class\=/conditional distribution, denoted~$\SynPBase{\cdot}$ where ``$\cdot$'' is subfigure \subref{fig:App:AddRes:Separable}~to~\subref{fig:App:AddRes:SameLike}.

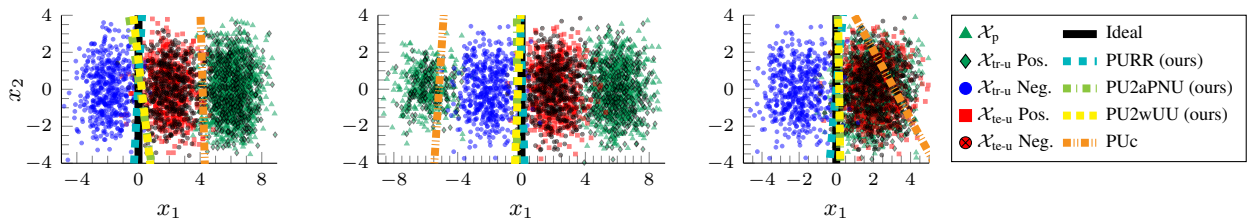
\begin{figure}[b]
  \centering
  \begin{minipage}[t]{0.26\textwidth}%
\begin{tikzpicture}
    \pgfplotstableread[col sep=comma] {plots/data/synthetic_separable.csv}\thedata
    \begin{axis}
        [
         ymin=\SynYMin,
         ymax=\SynYMax,
         xmin=\SynXMinSep,
         xmax=\SynXMax,
         xlabel={$x_1$},
         ylabel={$x_2$},
         ylabel shift =-6pt,  %
         point meta=explicit,
         width=\SynFigWidthSep,
         height=\SynFigHeight,
         axis x line*=bottom,  %
         axis y line*=left,    %
         xtick pos=left, %
         ytick pos=left, %
         xtick distance = {4},
         minor tick num = {7},
         ytick distance = {2},
         minor tick num = {3},
         label style={font=\footnotesize},
         tick label style={font=\scriptsize},
         clip mode=individual,                %
        ]
        \addplot[
                   scatter,
                   only marks,
                   opacity=0.5,
                   line width=0.2,
                   mark size=0.7pt,
                   legend image post style={scale=2.5},
                   legend style={font=\small},
                   scatter/classes={
                      0={mark=triangle*,color=ForestGreen,mark size=1.0pt},
                      2={mark=diamond*,color=ForestGreen,draw=black,mark size=1.0pt},
                      4={mark=*,blue},
                      1={mark=square*,red},
                      3={mark=otimes*,red,draw=black, line width=0.1}
                   },
                 ]
                 table[x index=0,y index=1, meta index=2] {\thedata};
        \draw [color=black, line width=\SynLineWidth] (-0.0,\SynYMin) -- (-0.0,\SynYMax);
        \addplot[dashed, domain=\SynXMinSep:\SynXMax, smooth, color=Aquamarine, line width=\SynLineWidth, opacity=1] {15.2682 * x + 0.445719} {};    %
        \addplot[dashdotted, domain=\SynXMinSep:\SynXMax, smooth, color=LimeGreen, line width=\SynLineWidth] {-5.47452 * x + 0.568992} { };  %
        \addplot[dotted, domain=\SynXMinSep:\SynXMax, smooth, color=Yellow, line width=\SynLineWidth] {-7.29372 * x + 1.54321} {};  %
        \addplot[densely dashdotdotted, domain=\SynXMinSep:\SynXMax, smooth, color=BurntOrange, line width=\SynLineWidth] {-28.4003 * x + 118.229} { };  %
    \end{axis}%
\end{tikzpicture}%
    \subcaption{Approx.\ linearly separable~$\uset{\tr}$}\label{fig:App:AddRes:Separable}%
  \end{minipage}%
  \begin{minipage}[t]{0.285\textwidth}%
    \centering%
\begin{tikzpicture}
    \pgfplotstableread[col sep=comma] {plots/data/synthetic_nonseparable.csv}\thedata
    \begin{axis}
        [
         ymin=\SynYMin,
         ymax=\SynYMax,
         xmin=\SynXMinNonsep,
         xmax=\SynXMax,
         xlabel={$x_1$},
         point meta=explicit,
         width=\SynFigWidthNonsep,
         height=\SynFigHeight,
         axis x line*=bottom,  %
         axis y line*=left,    %
         xtick pos=left, %
         ytick pos=left, %
         xtick distance = {4},
         minor x tick num = {7},
         ytick distance = {2},
         minor y tick num = {3},
         label style={font=\footnotesize},
         tick label style={font=\scriptsize},
         clip mode=individual,                %
        ]
        \addplot[
                   scatter,
                   only marks,
                   opacity=0.5,
                   line width=0.2,
                   mark size=0.7pt,
                   legend image post style={scale=1.5},
                   legend style={font=\small},
                   scatter/classes={
                      0={mark=triangle*,color=ForestGreen,mark size=1.0pt},
                      4={mark=*,blue},
                      2={mark=diamond*,color=ForestGreen,draw=black,mark size=1.0pt},
                      3={mark=otimes*,red,draw=black},
                      1={mark=square*,red}
                   },
                 ]
                 table[x index=0,y index=1, meta index=2] {\thedata};
        \draw [color=black, line width=\SynLineWidth] (-0.0,\SynYMin) -- (-0.0,\SynYMax);
        \addplot[dashed, domain=\SynXMinNonsep:\SynXMax, smooth, color=Aquamarine, line width=\SynLineWidth, opacity=1] {-28.8515 * x + 1.41399} {};    %
        \draw [dashdotted, domain=\SynXMinNonsep:\SynXMax, smooth, color=LimeGreen, line width=\SynLineWidth] (-0.2888,\SynYMin) -- (-0.2885,\SynYMax);  %
        \addplot[dotted, domain=\SynXMinNonsep:\SynXMax, smooth, color=Yellow, line width=\SynLineWidth] {22.1003 * x + 5.01917} {};  %
        \addplot[densely dashdotdotted, domain=\SynXMinNonsep:\SynXMax, smooth, color=BurntOrange, line width=\SynLineWidth] {13.6044 * x + 70.3757} { };  %
    \end{axis}%
\end{tikzpicture}%
    \subcaption{Non-linearly separable~$\uset{\tr}$}\label{fig:App:AddRes:Nonseparable}%
  \end{minipage}%
  \begin{minipage}[t]{0.22\textwidth}%
    \centering%
\begin{tikzpicture}
    \pgfplotstableread[col sep=comma] {plots/data/synthetic_same_like.csv}\thedata
    \begin{axis}
        [
         ymin=\SynYMin,
         ymax=\SynYMax,
         xmin=\SynXMinSep,
         xmax=\SynXMaxSameWidth,
         xlabel={$x_1$},
         point meta=explicit,
         width=\SynFigWidthSameLike,
         height=\SynFigHeight,
         axis x line*=bottom,  %
         axis y line*=left,    %
         xtick pos=left, %
         ytick pos=left, %
         xtick distance = {2},
         ytick distance = {2},
         minor tick num = {3},
         label style={font=\footnotesize},
         tick label style={font=\scriptsize},
         clip mode=individual,                %
        ]
        \addplot[
                   scatter,
                   only marks,
                   opacity=0.5,
                   line width=0.2,
                   mark size=0.7pt,
                   scatter/classes={
                      0={mark=triangle*,color=ForestGreen,mark size=1.0pt},
                      2={mark=diamond*,color=ForestGreen,draw=black,mark size=1.0pt},
                      4={mark=*,blue},
                      1={mark=square*,red},
                      3={mark=otimes*,red,draw=black, line width=0.1}
                   },
                 ]
                 table[x index=0,y index=1, meta index=2] {\thedata};
        \draw [color=black, line width=\SynLineWidth] (-0.0,\SynYMin) -- (-0.0,\SynYMax);  %
        \addplot[dashed, domain=\SynXMinSep:\SynXMax, smooth, color=Aquamarine, line width=\SynLineWidth, opacity=1] {10.9937 * x + 0.0416264} {};    %
        \addplot[dashdotted, domain=\SynXMinSep:\SynXMax, smooth, color=LimeGreen, line width=\SynLineWidth] {-54.1709 * x + 7.87775} { };  %
        \addplot[dotted, domain=\SynXMinSep:\SynXMax, smooth, color=Yellow, line width=\SynLineWidth] {-70.3884 * x + 11.0559} {};  %
        \addplot[densely dashdotdotted, domain=\SynXMinSep:\SynXMax, smooth, color=BurntOrange, line width=\SynLineWidth] {-1.85503 * x + 5.80523} { };  %
    \end{axis}%
\end{tikzpicture}%
    \subcaption{$\SynPLike{SameLike} = \nlike{}$}\label{fig:App:AddRes:SameLike}%
  \end{minipage}%
  \begin{minipage}[]{0.235\textwidth}
    \vspace{-3.5cm}
\begin{tikzpicture}
    \pgfplotstableread[col sep=comma] {plots/data/synthetic_separable.csv}\thedata
    \begin{axis}
        [
         hide axis,
         xmin=0,
         xmax=1,
         ymin=0,
         ymax=1,
         scale only axis,width=1mm, %
         legend cell align={left},              %
         legend style={font=\scriptsize},
         legend columns=5,
         legend image post style={scale=0.75}, %
         transpose legend,
        ]
        \addlegendimage{only marks,mark=triangle*,color=ForestGreen,mark size=3.5pt};
        \addlegendentry{$\pset{}$}

        \addlegendimage{only marks,mark=diamond*,color=ForestGreen,draw=black,mark size=1.0pt,mark size=3.5pt};
        \addlegendentry{$\uset{\tr}$ Pos.}

        \addlegendimage{only marks,mark=*,blue,mark size=3.0pt};
        \addlegendentry{$\uset{\tr}$ Neg.}

        \addlegendimage{only marks,mark=square*,red,mark size=3.0pt};
        \addlegendentry{$\uset{\te}$ Pos.}

        \addlegendimage{only marks,mark=otimes*,red,draw=black,line width=0.1,mark size=3.0pt};
        \addlegendentry{$\uset{\te}$ Neg.}

        \addlegendimage{color=black, line width=\SynLineWidth}
        \addlegendentry{Ideal}

        \addlegendimage{dashed, domain=\SynXMinSep:\SynXMax, smooth, color=Aquamarine, line width=\SynLineWidth, opacity=1};    %
        \addlegendentry{\nnpurpl\ourMethodText}

        \addlegendimage{dashdotted, domain=\SynXMinSep:\SynXMax, smooth, color=LimeGreen, line width=\SynLineWidth};  %
        \addlegendentry{\pupnu\ourMethodText}

        \addlegendimage{dotted, domain=\SynXMinSep:\SynXMax, smooth, color=Yellow, line width=\SynLineWidth};  %
        \addlegendentry{\punu\ourMethodText}

        \addlegendimage{densely dashdotdotted, domain=\SynXMinSep:\SynXMax, smooth, color=BurntOrange, line width=\SynLineWidth};  %
        \addlegendentry{\PUc}
    \end{axis}%
\end{tikzpicture}%
   \end{minipage}
  \caption{Predicted linear decision boundaries for three synthetic datasets (\eqsmall{${\np = \nTrU = \nTeU = 1,000}$}). Our three methods --~\nnpurpl{}, \pupnu{}, and~\punu{}~-- are robust to non\=/linear \& non\=/existent training class boundaries, but \PUc~fails in all three cases. Ideal boundary: ${x_1 = 0}$.}\label{fig:App:AddRes:SyntheticData}
\end{figure}

Figure~\ref{fig:App:AddRes:Separable}'s positive-train class\=/conditional distribution is
{\small
  \begin{equation}
      \SynPLike{Separable} = \SynFrac{1}{2}\normal{\SynVect{\hphantom{-}6}{-1}}{\SynVariance} +\SynFrac{1}{2}\normal{\SynVect{\hphantom{-}6}{\hphantom{-}1}}{\SynVariance}\text{,}
  \end{equation}%
}%
\noindent
making the training distribution's optimal separator linear.  \PUc{}~performed poorly on this setup for two reasons: covariate shift's assumption ${\ppostGeneral{\tr} = \ppostGeneral{\te}}$ does not hold, and the positive-train supports are functionally disjoint so importance function~$\Wx$ is practically unbounded.  Our methods all performed well, even \pupnu\ where inclusion of $\pset{}$'s~risk had minimal impact since for most good boundaries, $\pset{}$'s~risk was an inconsequential penalty.

Figure~\ref{fig:App:AddRes:Nonseparable} adds to~$\SynPLike{Separable}$ a third Gaussian where
{\small
  \begin{equation}\label{eq:AddRes:Syn:Inseparable}
    \SynPLike{Nonseparable} = \SynFrac{2}{3}\SynPLike{Separable} + \SynFrac{1}{3}\normal{\SynVect{{-}6}{\hphantom{-}0}}{\SynVariance}\text{,}%
  \end{equation}%
}%
\noindent
so the training distribution's optimal separator is non\=/linear. \PUc~performs poorly for the same reasons described above. The new centroid does not meaningfully affect~\nnpurpl. The most important takeaway is that linear~$\cdc{}$'s inability to partition~$\uset{\tr}$ has limited impact on \punu\ and \pupnu; $\uset{\tr}$'s misclassified examples act as a fixed penalty that only slightly offsets the two\=/step decision boundaries.

Figure~\ref{fig:App:AddRes:SameLike} uses the worst-case positive-train class\=/conditional, i.e.,~${\SynPLike{SameLike} = \nlike{}}$, making positive (labeled) data statistically identical to the (train and test) negative class\=/conditional distribution. Its training marginal~$\marg{\tr}$ is not separable -- linearly or otherwise.  Unlike~\PUc, our methods learned correct boundaries, which shows their robustness.

\clearpage
\subsection{Expanded MNIST, 20~Newsgroups, and~CIFAR10 Experiment Set}\label{sec:App:AddRes:Nonoverlap}

Table~\ref{tab:App:AddRes:FullClassPartition} is an expanded version of Section~\ref{sec:ExpRes:NonidenticalSupports}'s Table~\ref{tab:ExpRes:Nonoverlap}.  We provide these additional results to give the reader further evidence of our methods' superior performance. %

In this section, each of the three datasets (i.e.,~MNIST, 20~Newsgroups, and CIFAR10) now has two positive-training~(\pDsTrain) class configurations that are partially disjoint from the positive-test~(\pDsTest) class.  For each such configuration, Table~\ref{tab:App:AddRes:FullClassPartition} contains three experiments (in order):
\begin{enumerate}
  \setlength{\itemsep}{0pt}
  \item $\prd{\tr} < \prd{\te}$
  \item $\prd{\tr} = \prd{\te}$
  \item $\prd{\tr} > \prd{\te}$
\end{enumerate}
\noindent
It is easier to directly compare the effects of increasing/decreasing~$\prd{\tr}$ when the magnitude of the training prior increase and decrease are equivalent (e.g.,~for MNIST ${\prd{\te} = 0.5}$ so we tested performance at ${\prd{\tr} = \prd{\te} \pm0.12}$ and ${\prd{\tr} = \prd{\te} \pm0.21}$ depending on the class partition).  We maintained that rule of thumb when possible, but cases did arise where there were insufficient positive example with the labels in~\pDsTrain{} to support such a high positive prior.  In those cases, we clamp that \pDsTrain~class definition's maximum~$\prd{\tr}$.

The key takeaway from Table~\ref{tab:App:AddRes:FullClassPartition} is that across these additional, orthogonal definitions of~\pDsTrain, our methods still outperform \PUc{} and \nnpuOpt{} --- usually by a wide margin (statistical significance according to 1\%~paired t\=/test).

In all experiments, our methods' performance degraded as $\prd{\tr}$~increased since a larger prior makes it harder to identify the negative examples in~$\uset{\tr}$.  To gain an intuition about why this is true, consider the extreme case where ${\prd{\tr} = 1}$; learning is impossible since the positive-train class\=/conditional distribution may be arbitrarily different, and there are no negative samples that can be used to relate the two distributions. In contrast when ${\prd{\tr} = 0}$, identifying the negative set is trivial (i.e.,~all of $\uset{\tr}$ is negative), and NU~learning can be applied directly to learn~$\dec$.

\PUc~performs best when ${\prd{\tr} = \prd{\te}}$.  When $\prd{\tr}$~diverges from that middle point, \PUc's~performance declines.  To gain an intuition why that is, consider density-ratio estimation in terms of the component class conditionals.  When ${\prd{\tr} = \prd{\te}}$, $\Wx=1$ for all negative examples; from Table~\ref{tab:App:AddRes:FullClassPartition}'s results, we know that \PUc~performs best when there is no bias, i.e.,~\mbox{\pDsTrain{} = \pDsTest{}}.  A static positive prior eliminates one possible source of bias making density-ratio estimation easier and more accurate.

\begin{table}[!h]
  \centering
  \caption{Full MNIST, 20~Newsgroups, and CIFAR10 experimental class partition results. Each result is the inductive misclassification rate~(\%) mean and standard deviation over 100~trials for MNIST, 20~Newsgroups, and CIFAR10 with different positive \& negative class definitions. For \textit{all} experiments with positive bias (i.e.,~rows~\mbox{2--8} for each dataset), all three of our methods had statistically significant better performance than \PUc{} and \nnpuOpt{} according to a 1\%~paired t\=/test. Boldface indicates a shifted task's best performing method. Negative~(N) \& positive\=/test~(\pDsTest) class definitions are identical for each dataset's first three experiments. Positive train~(\pDsTrain) specified as \pDsTest\ denotes no bias. Our three methods -- \nnpurpl, \pupnu, and \punu{} -- are denoted with~\ourMethodKey{}.}\label{tab:App:AddRes:FullClassPartition}
\newcommand{\DsName}[1]{\multirow{8}{*}{\rotatebox[origin=c]{90}{#1}}}

\newcommand{\BigArrowBase}[1]{\multicolumn{1}{c@{}}{\multirow{3}{*}{$\Bigg#1$}}}
\newcommand{\BigUpArrow}{\BigArrowBase{\uparrow}}
\newcommand{\BigDownArrow}{\BigArrowBase{\downarrow}}

\renewcommand{\arraystretch}{1.2}
\setlength{\dashlinedash}{0.4pt}
\setlength{\dashlinegap}{1.5pt}
\setlength{\arrayrulewidth}{0.3pt}

\newcommand{\NDefFull}[1]{\multirow{7}{*}{\shortstack[l]{#1}}}
\newcommand{\PTestDefFull}[1]{\NDefFull{#1}}

\setlength{\tabcolsep}{5.4pt}

{\centering
  \footnotesize
  \begin{tabular}{@{}llllllrrrrrr@{}}
  \toprule
  \multirow{2}{*}{} & \multirow{2}{*}{N} & \multirow{2}{*}{\pDsTest} & \multirow{2}{*}{\pDsTrain} &  \multirow{2}{*}{$\prd{\tr}$} & \multirow{2}{*}{$\prd{\te}$} &    & \multicolumn{2}{c}{Two\=/Step (\putwo)}  & \multicolumn{2}{c}{Baselines}  & \multicolumn{1}{c@{}}{Ref.} \\\cmidrule(lr){8-9}\cmidrule(lr){10-11}\cmidrule(l){12-12}
                    &   & & &  & & \hspace{-6pt}\nnpurpl\ourMethodKey{} & \bpnu\ourMethodKey{} & \wuu\ourMethodKey{} & \PUc{} & \nnpuOpt{} & \PNte{} \\\midrule
  \DsName{MNIST} & \NDefFull{0, 2, 4,\\ 6, 8} & \PTestDefFull{1, 3, 5,\\ 7, 9} & \SamePTest & 0.5 & 0.5 & \NRes{10.0}{1.3} & \NRes{10.0}{1.2} & \NRes{11.6}{1.6} & \NRes{8.6}{0.8} & \NRes{5.5}{0.5} & \BigUpArrow{} \\\cdashline{4-11}
    &       &         &  \multirow{3}{*}{7, 9}   & 0.29 & 0.5 & \NRes{6.8}{0.8} & \NResT{5.3}{0.6} & \NRes{6.0}{0.7} & \NRes{29.2}{2.1} & \NRes{36.7}{2.7} &  \\\cdashline{5-11}
    &       &         &     & 0.5  & 0.5 & \NRes{9.4}{1.5} & \NResT{7.1}{0.9} & \NRes{8.3}{1.5} & \NRes{26.8}{2.4} & \NRes{35.1}{2.5} &  \\\cdashline{5-11}
    &       &         &     & 0.71 & 0.5 & \NRes{14.0}{3.0} & \NResT{11.1}{1.4} & \NRes{14.8}{3.1} & \NRes{26.9}{3.0} & \NRes{34.5}{2.9} & \NRes{2.8}{0.2}  \\\cdashline{4-11}
    &       &         & \multirow{3}{*}{1, 3, 5}   & 0.38 & 0.5 & \NRes{8.1}{1.0} & \NResT{6.5}{0.8} & \NRes{7.6}{0.9} & \NRes{20.2}{2.5} & \NRes{25.9}{1.1} & \BigDownArrow{} \\\cdashline{5-11}
    &       &         &   & 0.5  & 0.5 & \NRes{10.0}{1.6} & \NResT{8.4}{1.1} & \NRes{10.2}{1.4} & \NRes{18.5}{2.9} & \NRes{26.9}{1.2} &   \\\cdashline{5-11}
    &       &         &  & 0.63 & 0.5 & \NRes{12.5}{2.3} & \NResT{11.4}{1.3} & \NRes{14.3}{2.3} & \NRes{18.6}{3.3} & \NRes{28.5}{1.2} &  \\\cdashline{2-12}
    & {0, 2}  & {5, 7}  & {1, 3}  & 0.5  & 0.5 & \NRes{4.0}{0.8} & \NRes{3.6}{0.9} & \NResT{3.1}{0.7} & \NRes{17.1}{4.6} & \NRes{30.9}{5.3} & \NRes{1.1}{0.2} \\\midrule
  \DsName{20~Newsgroups} & \NDefFull{sci, soc,\\talk} & \PTestDefFull{alt, comp,\\ misc, rec} & \SamePTest & 0.56 & 0.56 & \NRes{15.4}{1.3} & \NRes{14.9}{1.0} & \NRes{16.7}{2.3} & \NRes{14.9}{1.0} & \NRes{14.1}{0.8} & \BigUpArrow{} \\\cdashline{4-11}
    & &  &  \multirow{3}{*}{misc, rec}  & 0.37 & 0.56 & \NRes{13.9}{0.7} & \NResT{12.8}{0.6} & \NRes{14.3}{0.9} & \NRes{28.9}{1.8} & \NRes{28.8}{1.3} & \\\cdashline{5-11}
    & &  &    & 0.56 & 0.56 & \NRes{17.5}{2.1} & \NResT{13.5}{0.8} & \NRes{15.1}{1.3} & \NRes{23.9}{3.0} & \NRes{28.8}{1.7} & \\\cdashline{5-11}
    & &  &    & 0.65 & 0.56 & \NRes{20.2}{2.8} & \NResT{14.0}{0.9} & \NRes{15.9}{1.5} & \NRes{21.8}{3.3} & \NRes{29.0}{1.8} & \NRes{10.5}{0.5} \\\cdashline{4-11}
    & &  & \multirow{3}{*}{comp}    & 0.37 & 0.56 & \NResT{13.3}{0.6} & \NRes{13.7}{0.6} & \NRes{14.4}{0.7} & \NRes{30.3}{2.0} & \NRes{31.4}{0.7} & \BigDownArrow{} \\\cdashline{5-11}
    & &  &        & 0.56 & 0.56 & \NRes{16.0}{1.5} & \NResT{14.9}{0.7} & \NRes{15.7}{0.9} & \NRes{28.6}{2.6} & \NRes{31.2}{0.8} & \\\cdashline{5-11}
    & &  &        & 0.65 & 0.56 & \NRes{19.2}{2.4} & \NResT{15.6}{0.9} & \NRes{16.5}{1.2} & \NRes{27.8}{2.7} & \NRes{31.3}{0.7} & \\\cdashline{2-12}
    & {misc, rec} & {soc, talk} & {alt, comp} & 0.55 & 0.46 & \NRes{5.9}{1.0} & \NRes{7.1}{1.1} & \NResT{5.6}{1.7} & \NRes{18.5}{4.3} & \NRes{35.3}{5.2} & \NRes{2.1}{0.3}  \\\midrule
  \DsName{CIFAR10}  & \NDefFull{Bird, Cat,\\Deer, Dog,\\Frog, Horse} & \PTestDefFull{Plane, \\ Auto, Ship, \\ Truck} & \SamePTest & 0.4 & 0.4 & \NRes{14.1}{0.8} & \NRes{14.2}{1.3} & \NRes{15.4}{1.7} & \NRes{13.8}{0.7} & \NRes{12.3}{0.6} & \BigUpArrow{} \\\cdashline{4-11}
    & & & \multirow{3}{*}{Plane}  & 0.14 & 0.4 & \NRes{12.1}{0.7} & \NResT{11.9}{0.7} & \NRes{12.4}{0.9} & \NRes{26.7}{1.4} & \NRes{26.7}{1.0} & \\\cdashline{5-11}
    & & & & 0.4 & 0.4 & \NResT{13.8}{0.9} & \NRes{14.5}{1.4} & \NRes{15.1}{1.6} & \NRes{20.6}{1.5} & \NRes{27.4}{1.0} &  \\\cdashline{5-11}
    & & & & 0.6 & 0.4 & \NResT{16.1}{1.1} & \NRes{16.7}{1.5} & \NRes{20.0}{2.7} & \NRes{21.5}{1.6} & \NRes{28.4}{1.0} & \NRes{9.7}{0.5}\\\cdashline{4-11}
    & & & \multirow{3}{*}{\shortstack[l]{Auto,\\Truck}}  & 0.25 & 0.4 & \NRes{12.7}{0.7} & \NResT{12.4}{0.7} & \NRes{12.8}{0.8} & \NRes{19.2}{1.1} & \NRes{20.3}{0.8} & \BigDownArrow{} \\\cdashline{5-11}
    & & & & 0.4 & 0.4 & \NRes{14.1}{0.9} & \NResT{13.9}{1.1} & \NRes{14.4}{1.2} & \NRes{17.7}{1.0} & \NRes{20.3}{0.8} &  \\\cdashline{5-11}
    & & & & 0.55 & 0.4 & \NResT{16.0}{1.1} & \NRes{16.2}{1.6} & \NRes{17.1}{2.2} & \NRes{18.3}{1.1} & \NRes{20.5}{0.9} &  \\\cdashline{2-12}
    & {Deer, Horse} & {Plane, Auto} & {Cat, Dog} & 0.5 & 0.5 & \NRes{14.1}{0.9} & \NRes{14.9}{1.5}  & \NResT{11.2}{0.8} & \NRes{33.1}{2.7} & \NRes{47.5}{2.0} & \NRes{7.7}{0.4} \\
  \bottomrule
\end{tabular}
}
 \end{table}

\clearpage
\subsection{Case Study: Arbitrary Adversarial Concept Drift}\label{sec:App:AddRes:Spam}

This section's experiments model adversarial settings where the positive class\=/conditional distribution shifts significantly faster than the negative class distribution.  As explained in Section~\ref{sec:App:SpamDatasets}, the training set was composed of spam and ham emails from the TREC0\underline{5} dataset; the test set was composed of spam and ham emails from the TREC0\underline{7} dataset.  The two dataset's ham emails are quite different -- TREC05 relies heavily on Enron emails while TREC07 contains many emails received on a university email server.  We are therefore confident our fixed-negative-distribution assumption in Eq.~\eqref{eq:PreviousWork:NegAssume} does not hold.

Table~\ref{tab:App:AddRes:Spam} and Figure~\ref{fig:App:AddRes:SpamPlot} compare our methods to \PUc{} and \nnPU{} across three different training priors~($\prd{\tr}$).  Under all three experimental conditions, our three methods outperformed both \PUc{} and \nnpuOpt{} according to a 1\%~paired t-test.  \punu{}~was the top performer for all experiments.  As evidenced by the PN~misclassification rate, a highly accurate classifier can be constructed for this dataset.  Similarly, $\cdc$~accurately labels~$\uset{\tr}$.  The resulting surrogate negative set is more useful than~$\pset{}$ to classify the spam emails from the test distribution.  \pupnu{}~performed slightly worse than~\punu{} because the spam emails in~$\pset{}$ are of very limited value due to the significant adversarial concept drift.\footnote{\punu{} and \pupnu{} used top\K{} weighting (see Section~\ref{sec:App:AddRes:AltStep1:Methods}) for step~\#1.}

\begin{table}[h]
  \centering
  \caption{Inductive misclassification rate~(\%) mean and standard deviation over 100~trials for arbitrary adversarial concept drift on the TREC spam email datasets.  In all experiments, our three methods --~\nnpurpl{}, \pupnu{}, \&~\punu{}~-- (which are denoted by~\ourMethodKey) statistically outperformed \PUc{} and~\nnpuOpt{} according to a paired t\=/test (${p < 0.01}$) with \punu{}~the top performer across all training priors~($\prd{\tr}$).}\label{tab:App:AddRes:Spam}
\renewcommand{\arraystretch}{1.2}
\setlength{\dashlinedash}{0.4pt}
\setlength{\dashlinegap}{1.5pt}
\setlength{\arrayrulewidth}{0.3pt}

\setlength{\tabcolsep}{7.0pt}

\newcommand{\NResSpam}[2]{\NRes{#1}{#2}}
\newcommand{\NResSpamTop}[2]{\textbf{\NRes{#1}{#2}}}

{\centering
  \footnotesize
  \begin{tabular}{@{}llllllrrrrrr@{}}
    \toprule
    \multicolumn{2}{c}{Train} & \multicolumn{2}{c}{Test} &  \multirow{2}{*}{$\prd{\tr}$} & \multirow{2}{*}{$\prd{\te}$} &    & \multicolumn{2}{c}{Two\=/Step (\putwo)}  & \multicolumn{2}{c}{Baselines}  & \multicolumn{1}{c@{}}{Ref.} \\\cmidrule(r){1-2}\cmidrule(lr){3-4}\cmidrule(lr){8-9}\cmidrule(lr){10-11}\cmidrule(l){12-12}
    Pos. & Neg.  & Pos.    &  Neg.     &      &     & \hspace{-6pt}\nnpurpl\ourMethodKey{} & \bpnu\ourMethodKey{} & \wuu\ourMethodKey{} & \PUc{}  & \hspace{-3pt}\nnpuOpt{} & \PNte{} \\\midrule
    \NDef{2005 \\ Spam} & \NDef{2005 \\ Ham} & \NDef{2007 \\ Spam}   & \NDef{2007 \\ Ham} & 0.4  & 0.5 & \NResSpam{26.5}{2.6}  & \NResSpam{26.9}{3.1}  & \NResSpamTop{25.1}{3.1}  & \NRes{35.2}{11.3} & \NRes{40.9}{3.1} & \multicolumn{1}{c@{}}{$\uparrow$} \\\cdashline{5-11}
                        &                    &                       &                    & 0.5  & 0.5 & \NResSpam{27.5}{3.4}  & \NResSpam{28.6}{4.5}  & \NResSpamTop{25.1}{3.3}  & \NRes{34.6}{10.2} & \NRes{40.5}{2.7} & \NRes{0.6}{0.3} \\\cdashline{5-11}
                        &                    &                       &                    & 0.6  & 0.5 & \NResSpam{30.8}{4.2}  & \NResSpam{33.0}{5.7}  & \NResSpamTop{29.3}{6.5}  & \NRes{38.5}{10.8} & \NRes{41.1}{2.9} & \multicolumn{1}{c@{}}{$\downarrow$} \\
  \bottomrule
\end{tabular}
}
 \end{table}

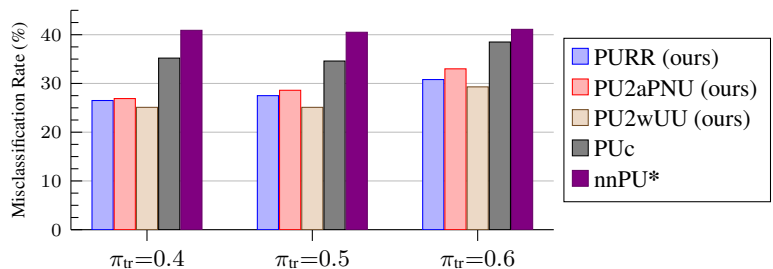
\begin{figure}[h]
  \centering
  \pgfplotstableread[col sep=comma]{plots/data/spam-bar.csv}\datatable%
  \begin{tikzpicture}%
    \begin{axis}[%
        axis lines*=left,%
        ymajorgrids,  %
        bar width=8pt,%
        height=4.5cm,%
        width=8.0cm,%
        ymin=0,%
        ymax=\SpamYMax,%
        ybar={\BarLineWidth},%
        ytick distance={10},%
        minor y tick num={3},%
        y tick label style={font=\scriptsize},%
        ylabel={\scriptsize Misclassification Rate~(\%)},%
        ylabel shift =-2pt,  %
        title style={yshift=-4pt,},
        xtick=data,%
        x tick label style={font=\small,
                            align=center},%
        every tick/.style={color=black,
                           line width=\BarLineWidth},%
        xticklabels from table={\datatable}{Dataset},%
        enlarge x limits=0.23,%
        legend style={at={(1.02,0.5)},
                      anchor=west,
                      font=\small,},
        legend cell align={left},              %
        title style = {text depth=0.5ex},      %
    ]%
    \addplot table [x expr=\coordindex, y index=1,] {\datatable};%
    \addplot table [x expr=\coordindex, y index=2] {\datatable};%
    \addplot table [x expr=\coordindex, y index=3] {\datatable};%
    \addplot table [x expr=\coordindex, y index=4] {\datatable};%
    \addplot table [x expr=\coordindex, y index=5] {\datatable};%
    \legend{\nnpurpl\ourMethodText{},\pupnu\ourMethodText{},\punu\ourMethodText{},\PUc{},\nnpuOpt{}}%
    \end{axis}%
  \end{tikzpicture}%
   \caption{Mean inductive misclassification rate~(\%) over 100~trials for the TREC spam datasets across three training priors~($\prd{\tr}$).  Our \punu{}~method was the top performer across all experiments.}\label{fig:App:AddRes:SpamPlot}
\end{figure}

\clearpage
\subsection{Identical Positive Supports with Bias}\label{sec:App:AddRes:MarginalBias}

The positive bias applied in this section's experiments is totally different from that in Sections~\ref{sec:ExpRes:NonidenticalSupports} and~\ref{sec:ExpRes:Spam}. Here we mimic situations where the labeled data are complete but non-representative resulting in identical marginal distribution supports but shifts in the marginal distribution's magnitude. We follow the experimental setup described in \citet{Sakai:2019}'s \PUc{}~paper.  LIBSVM~\citep{LIBSVM} benchmarks are used exclusively to ensure suitability with the SVM\=/like~\PUc; benchmarks ``banana,'' ``susy,'' ``ijcnn1,'' and ``a9a'' appear in~\citet{Sakai:2019}'s \PUc{}~paper.

\citeauthor{Sakai:2019}'s bias operation is based on the median feature vector.  Formally, given dataset~${\train \subset \domainX}$, define $\cmed$ as the median of set $\setbuild{\norm{\X - \xbar}_{2}}{\X \in \train}$ where $\norm{\cdot}_{2}$ is the $L_{2}$~(Euclidean) norm and $\xbar$~is $\train$'s mean vector, i.e.,
\begin{equation*}
 \xbar = \frac{1}{\abs{\train}} \sum_{\X \in \train} \X \text{.}
\end{equation*}
\noindent
Partition~$\train$ into subsets~${\trainLo \defeq \setbuild{\X \in \train}{\norm{\X - \xbar}_{2} < \cmed}}$ and ${\trainHi \defeq \train \setminus \trainLo}$.  Examples in $\pset{}$ and~$\uset{\tr}$ are selected from~$\trainLo$ with probability ${p=0.9}$ and from~$\trainHi$ with probability~${1-p}$.  ${p=0.1}$ is used when constructing $\uset{\te}$~and the test set. This bias operation simplifies density-ratio estimation since~${\forall_{\X \in \train}~\Wx \in \set{\frac{1}{9}, 9}}$. Their setting ${\prd{\tr} = \prd{\te} = 0.5}$ also simplifies density estimation as detailed in Section~\ref{sec:App:AddRes:Nonoverlap}.

We modified \citeauthor{Sakai:2019}'s setup such that $\train$~was exclusively the original dataset's positive-valued examples.   Negative examples were sampled uniformly at random.

\paragraph{Analysis} The experiments enumerated in Table~\ref{tab:App:AddRes:PUc} and shown visually in Figure~\ref{fig:App:AddRes:PUc:Complete} used the bias procedure described above on 10~LIBSVM datasets. According to a 1\%~paired t-test, \nnpurpl{} and \pupnu{} outperformed the baselines, \PUc{} and \nnpuOpt{}, on all ten benchmarks; \punu{} outperformed the baselines on nine of ten benchmarks.

\nnpurpl{} was the top performer on three benchmarks; \pupnu{} was the top performer on five benchmarks while \punu{} was the top performer on two benchmarks.  Each estimator is best suited to a different feature dimension range.  \nnpurpl{} performed best when the dataset had fewer features (e.g.,~${{<}50}$) while \pupnu{} performed well when the dimension was moderate.  \punu{} was the top performer when the dimension was large (e.g.,~${{\geq}300}$).

Accurate risk estimation is more challenging when the training sets are comparatively small but the feature count is high.  We expect that is causing \nnpurpl{} to struggle to reconcile/relate the different labeled losses (e.g.,~positive\=/labeled, unlabeled train, unlabeled test) in these higher dimension datasets.

\clearpage
\begin{table}[t]
  \centering
  \caption{Inductive misclassification rate~(\%) mean and standard deviation over 100~trials with \citet{Sakai:2019}'s median feature vector\=/based bias for 10~LIBSVM datasets. \underline{Underlining} denotes a statistically significant performance improvement versus~\PUc{} and \nnpuOpt{} according to a 1\%~paired t\=/test. Boldface indicates each dataset's best performing method. ${\np = 300}$ and ${\nTrU = \nTeU = 700}$.  Datasets are ordered by increasing dimension. Our three methods -- \nnpurpl, \pupnu, and \punu{} -- are denoted with~\ourMethodKey{}.}\label{tab:App:AddRes:PUc}
  {\small%
\setlength{\tabcolsep}{9.0pt}

{
  \small
  \begin{tabular}{@{}lrrrrrrr@{}}
    \toprule
    \multirow{2}{*}{Dataset} & \multirow{2}{*}{$d$} &        & \multicolumn{2}{c}{Two\=/Step (\putwo)}  & \multicolumn{2}{c}{Baselines} & \multicolumn{1}{c@{}}{Ref.} \\\cmidrule(lr){4-5}\cmidrule(lr){6-7}\cmidrule(l){8-8}
                      &        & \nnpurpl\ourMethodKey{} & \bpnu\ourMethodKey{} & \wuu\ourMethodKey{} & \PUc{}           & \nnpuOpt{}       & \PNte{}  \\\midrule
    banana     & 2      &  \NResSS{12.9}{2.1}   & \NResTop{11.8}{1.6}     &  \NResSS{13.3}{2.3}    & \NRes{17.4}{3.4} & \NRes{28.8}{3.8} & \NRes{ 8.6}{0.6}  \\%
    cod-rna    & 8      & \NResTop{14.7}{2.6}   &  \NResSS{15.1}{3.2}     &  \NResSS{15.5}{2.9}    & \NRes{25.2}{5.0} & \NRes{24.9}{2.3} & \NRes{ 6.5}{0.9}  \\%
    susy       & 18     & \NResTop{24.2}{2.1}   &  \NResSS{25.6}{2.2}     &  \NResSS{25.8}{2.2}    & \NRes{27.3}{4.3} & \NRes{45.9}{3.9} & \NRes{20.5}{1.3}  \\%
    ijcnn1     & 22     &  \NResSS{22.7}{2.8}   & \NResTop{17.7}{2.8}     &    \NRes{24.6}{3.1}    & \NRes{23.9}{3.6} & \NRes{34.7}{3.6} & \NRes{ 6.8}{0.8}  \\%
    covtype.b  & 54     & \NResTop{29.5}{2.9}   &  \NResSS{32.5}{3.2}     &  \NResSS{29.9}{2.4}    & \NRes{39.4}{4.2} & \NRes{55.5}{2.8} & \NRes{22.3}{1.4}  \\%
    phishing   & 68     &  \NResSS{11.3}{1.4}   & \NResTop{9.6}{1.0}      &  \NResSS{11.1}{1.8}    & \NRes{13.8}{4.1} & \NRes{22.5}{4.1} & \NRes{ 6.2}{0.6}  \\%
    a9a        & 123    &  \NResSS{27.1}{2.1}   & \NResTop{26.6}{1.8}     &  \NResSS{27.1}{2.1}    & \NRes{32.8}{2.6} & \NRes{32.5}{2.3} & \NRes{20.6}{1.0}  \\%
    connect4   & 126    &  \NResSS{34.9}{3.1}   & \NResTop{32.9}{2.7}     &  \NResSS{35.0}{2.9}    & \NRes{37.0}{2.8} & \NRes{45.1}{2.6} & \NRes{21.6}{1.3}  \\%
    w8a        & 300    &  \NResSS{17.2}{2.6}   &  \NResSS{21.0}{2.9}     & \NResTop{16.8}{2.9}    & \NRes{29.3}{6.2} & \NRes{41.1}{4.3} & \NRes{ 6.6}{0.7}  \\%
    epsilon    & 2,000  &  \NResSS{33.5}{4.8}   &  \NResSS{36.5}{5.0}     & \NResTop{31.5}{1.7}    & \NRes{62.8}{6.7} & \NRes{64.6}{1.5} & \NRes{23.7}{1.1}  \\
    \bottomrule
  \end{tabular}
}
  }%
\end{table}

\begin{figure}[t]
  \centering
\newcommand{\pucLegendEntrySpacer}{~~~~}
\pgfplotstableread[col sep=comma]{plots/data/bar-puc.csv}\datatable%
\begin{tikzpicture}%
  \begin{axis}[%
        axis lines*=left,%
        ymajorgrids,  %
        bar width=\PUcBarWidth,%
        height=6.0cm,%
        width=15.60cm,%
        ymin=0,%
        ymax=\PUcYMax,%
        ybar={\BarLineWidth},%
        ytick distance={10},%
        minor y tick num={3},%
        y tick label style={font=\footnotesize},%
        ylabel={\scriptsize Misclassification Rate~(\%)},%
        xtick=data,%
        x tick label style={font=\scriptsize,align=center},%
        typeset ticklabels with strut,
        every tick/.style={color=black, line width=\BarLineWidth},%
        xticklabels from table={\datatable}{Dataset},%
        legend columns=-1,
        legend style={%
                      at={(0.5,-0.15)},
                      anchor=north,
                      /tikz/every even column/.append style={column sep=0.5cm}
                     },  %
        legend cell align={left},              %
        legend style={font=\footnotesize},
        title style = {text depth=0.5ex}       %
    ]%
    \addplot table [x expr=\coordindex, y index=1] {\datatable};%
    \addplot table [x expr=\coordindex, y index=2] {\datatable};%
    \addplot table [x expr=\coordindex, y index=3] {\datatable};%
    \addplot table [x expr=\coordindex, y index=4] {\datatable};%
    \addplot table [x expr=\coordindex, y index=5] {\datatable};%
    \legend{\nnpurpl\ourMethodText\pucLegendEntrySpacer{},\pupnu\ourMethodText\pucLegendEntrySpacer{},\punu\ourMethodText\pucLegendEntrySpacer{},\PUc{},\nnpuOpt{}}%
  \end{axis}%
\end{tikzpicture}%
   \caption{Mean inductive misclassification rate~(\%) over 100~trials with \citet{Sakai:2019}'s median feature vector\=/based bias for the 10~LIBSVM datasets in Section~\ref{sec:App:AddRes:MarginalBias}.}\label{fig:App:AddRes:PUc:Complete}
\end{figure}
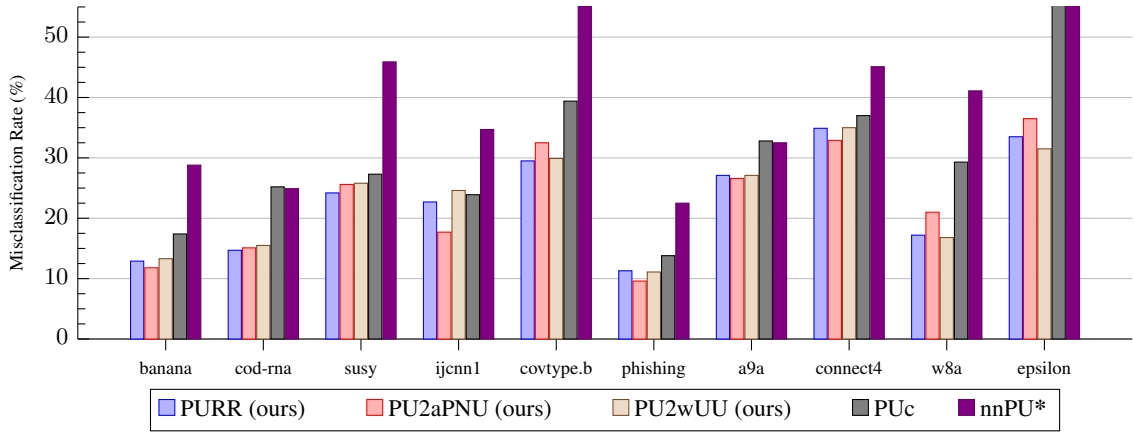

\clearpage
\subsection{Comparison to \bpu~Selection Bias Method \pusb}\label{sec:App:AddRes:PUSB}

Recall that baseline \PUc{} is a covariate-shift \bpu{}~method.  A NeurIPS reviewer requested an experiment comparing our proposed approaches to a selection bias \bpu~learning baseline.  This section compares our algorithms to \citet{Kato:2019}'s Positive-Unlabeled Selection Bias~(\pusb) method.

Let random variable ${\srv \in \domainS}$ denote whether some training example~${(\xrv,\yrv) \sim \pjointsimp}$ is labeled.  For all types of PU~learning (e.g.,~unbiased, \bpu, \acronym), it is straightforward that ${\srv = \yp}$ implies ${\yrv = \yp}$, i.e.,
\begin{equation}\label{eq:App:AddRes:PUSB:ProbLabeledPos}
  \pbase{}{\yrv = \yp \vert \srv = \yp} = 1
\end{equation}
\noindent
and
\begin{equation}\label{eq:App:AddRes:PUSB:ProbNegLabeled}
  \pbase{}{\srv = \yp \vert \yrv = \yn} = 0 \text{.}
\end{equation}

\pusb~makes what \citeauthor{Kato:2019} term the \textit{invariance-of-order assumption}. Formally, for any pair of training examples ${\X_{i},\X_{j} \in \domainX}$, it holds that
\begin{equation}\label{eq:App:AddRes:PUSB:InvarianceOrder}
  \pbase{}{\yrv = \yp \vert \X_{i}} \geq \pbase{}{\yrv = \yp \vert \X_{j}} \iff \pbase{}{\srv = \yp \vert \X_{i}} \geq \pbase{}{\srv = \yp \vert \X_{j}} \text{.}
\end{equation}
\noindent
In words, a training example is at least as likely to be positive-valued as another example if and only if it is at least as likely to be labeled as that other example.  As mentioned in Section~\ref{sec:APULearning}, it is not possible to directly compare our approaches to existing selection bias \bpu{}~methods like~\pusb{}.  Such \bpu~learning methods assume access to only a single unlabeled set~($\uset{\te}$) drawn from the test distribution while \acronym~learning provides two unlabeled sets~($\uset{\tr}$ and~$\uset{\te}$).

To ensure a fair comparison, we sought to replicate \citet{Kato:2019}'s experimental setup as closely as possible provided the constraints of our method -- even using their source code\footnote{\citeauthor{Kato:2019}'s source code is publicly available at \url{https://github.com/MasaKat0/PUlearning}.} verbatim where possible (e.g.,~\pusb{} used the \texttt{Chainer}~\citep{Chainer} neural network framework as specified by \citeauthor{Kato:2019}). Like in the \pusb~paper, we analyzed the performance of all methods on the MNIST~\citep{LeCun:1998} dataset.  To enrich the comparison, we also consider the drop\=/in MNIST variants FashionMNIST~\citep{FashionMNIST} and KMNIST~\citep{KMNIST}.

\paragraph{Dataset Construction} Our experiments exactly duplicate \citeauthor{Kato:2019}'s procedure for constructing biased-positive set~$\pset{}$.  Specifically, a multilayer perceptron~(MLP) with four hidden layers of 300~neurons each and ReLU activation is trained using the PN~logistic loss on the dataset's complete training and test sets.  $\pset{}$ is then selected u.a.r.\ without replacement from those positive-valued training examples the aforementioned MLP identifies as having the highest positive posterior.  Unlabeled test set~$\uset{\te}$ and the inductive test set are drawn u.a.r.\ without replacement from the complete training and test sets respectively.

\citeauthor{Kato:2019} uses the complete MNIST training set as the unlabeled set.  Since both \PUc{} and our methods require two unlabeled sets, we cannot follow the same methodology here.  Instead, we limit the size of the test unlabeled set and create unlabeled training set~$\uset{\tr}$ by selecting its positive examples according to the procedure described above for~$\pset{}$ and selecting its negative-valued examples u.a.r.\ without replacement from the training set's negative elements.  Table~\ref{tab:App:AddRes:PUSB:DatasetSizes} details our experiments' positive priors as well as the dataset and mini\=/batch sizes.

\paragraph{Hyperparameters} Identical hyperparameters were used for the MNIST, FashionMNIST, and KMNIST datasets.  Our methods, \nnpuOpt, and \PNte{} used identical hyperparameter settings as those tuned for MNIST in Section~\ref{sec:ExpRes}'s experiments.

\pusb{}'s hyperparameters match those specified by \citeauthor{Kato:2019} for MNIST, e.g.,~learning rate ${\learningRate = 10^{-5}}$ and weight decay ${\weightDecay = 5 \cdot 10^{-3}}$.  \pusb{}~learners were trained for 250~epochs using the Adam~\citep{Kingma:2015} optimizer.  As in the original paper, \pusb{}'s neural network had four hidden layers of 300~neurons each and batch normalization before each ReLU activation.

\paragraph{Results Analysis} Table~\ref{tab:App:AddExp:PUSB:Results} compares the performance of our methods -- \nnpurpl, \pupnu, and~\punu{} -- to the extended baseline set -- \PUc, \pusb, and~\nnpuOpt{} -- for the experimental setup described above.  To mitigate the effects of different unlabeled set configurations, our experiments tested two unlabeled set sizes, with one size half the other (Table~\ref{tab:App:AddRes:PUSB:DatasetSizes}). \PUc{}'s and our methods' results in Table~\ref{tab:App:AddExp:PUSB:Results:3K} used 6,000~total unlabeled samples, i.e.,~the same quantity used by \pusb{} in Table~\ref{tab:App:AddExp:PUSB:Results:6K}.  Figure~\ref{fig:App:AddRes:PUSB} visualizes these cross\=/table, matching-unlabeled-set-size results graphically.  For \nnpuOpt{} in Figure~\ref{fig:App:AddRes:PUSB}, three unlabeled set configurations are considered namely, \nnpuTE{} and \nnpuAll{} with ${\abs{\uset{\tr}} = \abs{\uset{\te}} = 3{,}000}$ as well as \nnpuTE{} with ${\abs{\uset{\te}} = 6{,}000}$. Observe that these are the only \nnpuOpt{} configurations using at most 6,000~unlabeled examples.

As mentioned in Section~\ref{sec:ExpRes:NonidenticalSupports}, when there is little to no dataset shift, shift-unaware methods (e.g.,~\nnPU) are expected to be the top performer.  As an intuition why -- when a method searches for a non-existent phenomenon, any patterns found will not generalize.  Since \nnpuOpt{} is the top performer for MNIST and~KMNIST despite not accounting for shift at all, it then stands to reason that \citet{Kato:2019}'s invariance-of-order bias induces only a small shift here.

We saw in Section~\ref{sec:ExpRes:NonidenticalSupports} that for such mild shifts (e.g.,~no bias), \PUc{}~often outperforms our methods.  We generally see the same trend in Table~\ref{tab:App:AddExp:PUSB:Results} for MNIST and KMNIST (primary exception being \pupnu{} for MNIST). This is again expected. Under mild shifts, covariate shift's consistent input-output relation assumption generally holds.  In addition, importance function~${\Wx \approx 1}$ for all~$\X$ under limited bias, in which case \PUc{} simplifies to essentially standard \nnPU{}.

All of our methods outperformed all baselines for FashionMNIST.  What is more, our methods outperformed \pusb{} in all but one case (\punu{} for KMNIST) even after accounting for unlabeled set size (Figure~\ref{fig:App:AddRes:PUSB}). In fact, \pusb{} always lagged \nnpuOpt{}.  This hints at a level of brittleness for \citeauthor{Kato:2019}'s method since \pusb{} struggled on a bias condition it specifically targets.

\clearpage

\begin{table}[t]
  \centering
  \caption{Positive priors, dataset sizes (including the validation set), and mini\=/batch sizes for Section~\ref{sec:App:AddRes:PUSB}'s invariance of order selection bias experiments.  The first column lists the table where each setup's corresponding results are enumerated.}\label{tab:App:AddRes:PUSB:DatasetSizes}
\begin{tabular}{@{}lrrrrrrrrrr@{}}
  \toprule
                  & \multicolumn{2}{c}{{\footnotesize Prior}} & \multicolumn{4}{c}{{\footnotesize Dataset Size}}     & \multicolumn{4}{c}{{\footnotesize Batch Size}} \\\cmidrule(lr){2-3}\cmidrule(lr){4-7}\cmidrule(l){8-11}
                                             & $\prd{\tr}$  & $\prd{\te}$  & $\np$      & $\nTrU$    & $\nTeU$    & $\nTest$  & $\decx$ & $\cdcx$ & \pusb{} & \PNte{} \\\midrule
  Table~\ref{tab:App:AddExp:PUSB:Results:3K} & 0.5          & 0.5          & 1,000      & 3,000      & 3,000      & 5,000     & 2,500   & 2,500   & 1,000   & 2,000   \\
  Table~\ref{tab:App:AddExp:PUSB:Results:6K} & 0.5          & 0.5          & 1,000      & 6,000      & 6,000      & 5,000     & 5,000   & 5,000   & 1,000   & 4,000   \\
  \bottomrule
\end{tabular}
 \end{table}

\begin{table}[t]
  \caption{Inductive misclassification rate~(\%) mean and standard deviation over 100~trials for the experiments using \citet{Kato:2019}'s invariance-of-order setup on the MNIST, FashionMNIST, and KMNIST datasets.  Bold face denotes each dataset's best performing method according to mean misclassification rate. Our methods -- \nnpurpl{}, \pupnu{}, and \punu{} -- are denoted with~\ourMethodKey.}\label{tab:App:AddExp:PUSB:Results}
    \begin{subtable}[h]{\textwidth}
      \centering
      \caption{${\abs{\uset{\tr}} = \abs{\uset{\te}} = 3{,}000}$}\label{tab:App:AddExp:PUSB:Results:3K}
      {\small
\newcommand{\PUResN}[2]{#1~(#2)}
\newcommand{\PUResB}[2]{\textbf{\PUResN{#1}{#2}}}

\renewcommand{\arraystretch}{1.2}
\setlength{\dashlinedash}{0.4pt}
\setlength{\dashlinegap}{1.5pt}
\setlength{\arrayrulewidth}{0.3pt}

\setlength{\tabcolsep}{7.7pt}

\begin{tabular}{@{}lrrrrrrr@{}}
  \toprule
  \multirow{2}{*}{Dataset} &            & \multicolumn{2}{c}{\footnotesize Two\=/Step (\putwo)} & \multicolumn{3}{c}{\footnotesize Baselines} & \multicolumn{1}{c@{}}{\footnotesize Ref.}\\\cmidrule(lr){3-4}\cmidrule(lr){5-7}\cmidrule(l){8-8}
               & \nnpurpl{}\ourMethodKey{} & \bpnu{}\ourMethodKey{} & \wuu{}\ourMethodKey{} & \PUc{}               & \pusb{}              & \nnpuOpt{}           & \PNte{}                \\\midrule
  MNIST        & \PUResN{13.0}{2.3}        & \PUResN{ 9.7}{1.3}     & \PUResN{11.6}{1.4}    & \PUResN{10.6}{1.1}   & \PUResN{15.9}{1.0}   & \PUResB{ 8.8}{0.9}   & \PUResN{ 3.6}{0.3}     \\
  FashionMNIST & \PUResN{ 6.4}{1.4}        & \PUResB{ 5.3}{0.7}     & \PUResN{ 5.9}{1.0}    & \PUResN{ 9.0}{1.1}   & \PUResN{10.5}{1.2}   & \PUResN{ 8.5}{1.3}   & \PUResN{ 3.5}{0.3}     \\
  KMNIST       & \PUResN{31.6}{2.4}        & \PUResN{29.7}{2.2}     & \PUResN{33.7}{2.3}    & \PUResN{27.3}{1.4}   & \PUResN{33.4}{1.2}   & \PUResB{24.6}{1.4}   & \PUResN{16.4}{0.8}     \\
  \bottomrule
\end{tabular}
       }
    \end{subtable}
    \hfill
    \begin{subtable}[h]{\textwidth}
      \centering
      \caption{${\abs{\uset{\tr}} = \abs{\uset{\te}} = 6{,}000}$}\label{tab:App:AddExp:PUSB:Results:6K}
      {\small
\newcommand{\PUResN}[2]{#1~(#2)}
\newcommand{\PUResB}[2]{\textbf{\PUResN{#1}{#2}}}

\renewcommand{\arraystretch}{1.2}
\setlength{\dashlinedash}{0.4pt}
\setlength{\dashlinegap}{1.5pt}
\setlength{\arrayrulewidth}{0.3pt}

\setlength{\tabcolsep}{7.7pt}

\begin{tabular}{@{}lrrrrrrr@{}}
  \toprule
  \multirow{2}{*}{Dataset} &            & \multicolumn{2}{c}{\footnotesize Two\=/Step (\putwo)} & \multicolumn{3}{c}{\footnotesize Baselines} & \multicolumn{1}{c@{}}{\footnotesize Ref.}\\\cmidrule(lr){3-4}\cmidrule(lr){5-7}\cmidrule(l){8-8}
               & \nnpurpl{}\ourMethodKey{} & \bpnu{}\ourMethodKey{} & \wuu{}\ourMethodKey{} & \PUc{}               & \pusb{}              & \nnpuOpt{}           & \PNte{}                \\\midrule
  MNIST        & \PUResN{10.5}{1.8}        & \PUResN{ 8.5}{1.2}     & \PUResN{ 9.3}{1.0}    & \PUResN{10.2}{1.1}   & \PUResN{14.2}{1.0}   & \PUResB{ 8.0}{0.9}   & \PUResN{ 2.8}{0.2}     \\
  FashionMNIST & \PUResN{ 5.6}{1.3}        & \PUResB{ 4.8}{0.6}     & \PUResN{ 5.0}{0.8}    & \PUResN{ 9.1}{1.2}   & \PUResN{10.0}{1.2}   & \PUResN{ 8.2}{1.2}   & \PUResN{ 3.1}{0.2}     \\
  KMNIST       & \PUResN{29.6}{2.2}        & \PUResN{29.3}{2.1}     & \PUResN{32.0}{2.2}    & \PUResN{27.0}{1.4}   & \PUResN{32.1}{1.2}   & \PUResB{24.1}{1.4}   & \PUResN{13.7}{0.7}     \\
  \bottomrule
\end{tabular}
       }
     \end{subtable}
\end{table}

\begin{figure}[t]
  \centering
  \pgfplotstableread[col sep=comma]{plots/data/pusb-bar.csv}\datatable%
  \begin{tikzpicture}%
    \begin{axis}[%
        axis lines*=left,%
        ymajorgrids,  %
        bar width=9pt,%
        height=6cm,%
        width=9.6cm,%
        ymin=0,%
        ymax=40,%
        ybar={\BarLineWidth},%
        ytick distance={10},%
        minor y tick num={3},%
        y tick label style={font=\footnotesize},%
        ylabel={\footnotesize Misclassification Rate~(\%)},%
        ylabel shift =-2pt,  %
        title style={yshift=-4pt,},
        xtick=data,%
        x tick label style={font=\small,
                            align=center},%
        every tick/.style={color=black,
                           line width=\BarLineWidth},%
        xticklabels from table={\datatable}{Dataset},%
        enlarge x limits=0.275,%
        legend style={at={(1.02,0.5)},
                      anchor=west,
                      font=\small,},
        legend cell align={left},              %
        title style = {text depth=0.5ex},      %
    ]%
    \addplot table [x expr=\coordindex, y index=1,] {\datatable};%
    \addplot table [x expr=\coordindex, y index=2] {\datatable};%
    \addplot table [x expr=\coordindex, y index=3] {\datatable};%
    \addplot table [x expr=\coordindex, y index=4] {\datatable};%
    \addplot table [x expr=\coordindex, y index=5] {\datatable};%
    \addplot table [x expr=\coordindex, y index=6] {\datatable};%
    \legend{\nnpurpl\ourMethodText{},\pupnu\ourMethodText{},\punu\ourMethodText{},\PUc{},\nnpuOpt{},\pusb{}}%
    \end{axis}%
  \end{tikzpicture}%
   \caption{Mean inductive misclassification rate~(\%) over 100~trials for the experiments using \citet{Kato:2019}'s invariance-of-order setup on the MNIST, FashionMNIST, and KMNIST datasets.  All learners saw up to 6,000~total unlabeled examples with results cross\=/compiled between Table~\ref{tab:App:AddExp:PUSB:Results:3K} (for our methods and \PUc) and Table~\ref{tab:App:AddExp:PUSB:Results:6K} (for \pusb).  Here \nnpuOpt{} considers three different unlabeled set configurations as described in Section~\ref{sec:App:AddRes:PUSB}.}\label{fig:App:AddRes:PUSB}
\end{figure}
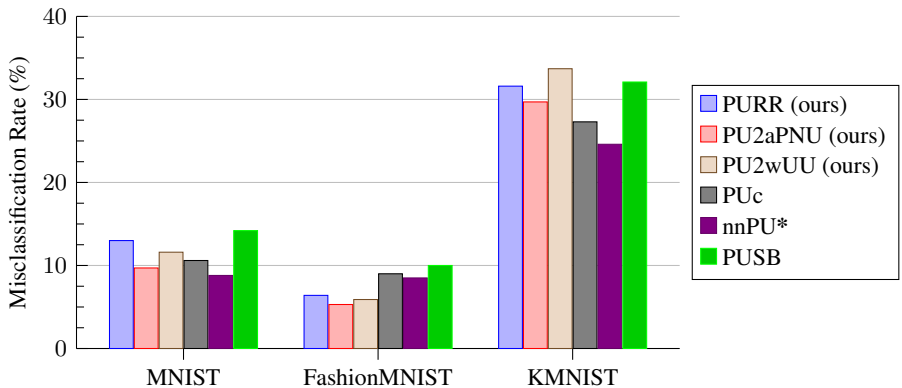
 
\clearpage
\subsection{Empirical Comparison of Absolute-Value and Non-Negativity Corrections}\label{sec:App:AddRes:AbspuVsNnpu}

Section~\ref{sec:AbsValCorrection} describes our streamlined absolute\=/value correction to address PU~learning overfitting.  This section compares our simpler absolute\=/value correction to \citet{Kiryo:2017}'s non\=/negativity correction using~$\max$ and ``defitting.''

\subsubsection{Ordinary Positive-Unlabeled Learning Performance Without Distributional Shift}

We first consider a direct comparison of \nnPU{} and \absPU{} on \textit{unshifted} data.  $\pset{}$~and $\uset{}$ are constructed identically to the procedure used to construct the positive-labeled and unlabeled-train datasets in our \acronym{}~learning experiments.  Unlike before, the inductive test set is now drawn from the \textit{training} distribution. We then trained classifiers using \nnPU{} and \absPU{} with the sigmoid loss.  In all experiments, the classifiers had identical initial weights and were trained on identical dataset splits.

Hyperparameters (including~$\gradAtten$) were tuned using \nnPU{}; these identical hyperparameters were then used for \absPU{} (i.e.,~not in any way tuned for absolute-value correction).  Therefore, the results represent the \textit{performance floor} when transitioning from \nnPU{} to \absPU{}.  This was done due to time constraints.

Table~\ref{tab:App:AbsPuVsNnpu:PUc} compares \absPU{} and \nnPU{} for the datasets in Sections~\ref{sec:ExpRes:NonidenticalSupports}\footnote{The test conditions for MNIST, 20~Newsgroups, and~CIFAR10 correspond to the unbiased test conditions (i.e.,~row~1 for each dataset where \mbox{\pDsTrain = \pDsTest}) in Table~\ref{tab:ExpRes:Nonoverlap}/Table~\ref{tab:App:AddRes:FullClassPartition}.},~\ref{sec:ExpRes:Spam}, and~\ref{sec:App:AddRes:MarginalBias}. We also report the difference between \absPU{} and \nnPU{} with a positive number indicating that \absPU{} performed better that~\nnPU{}.

\absPU{} was the top performer on eight of fourteen benchmarks and tied with \nnPU{} on two others; the results are generally too close to be statistically significant. Both methods had comparable variances. In summary, \absPU{} is both simpler and saw similar or slightly better performance than \nnPU{} on unbiased data, even under conditions (i.e.,~hyperparameters) that favor~\nnPU{}\@.

\begin{table}[h]
  \centering
  \caption{Comparison of inductive misclassification rate~(\%) mean and standard deviation over 100~trials for \absPU{} and \nnPU{} on unshifted data.  Boldface denotes the best performing algorithm according to mean misclassification rate. For the difference~(Diff.) column, a positive value denotes that~\absPU{} outperformed~\nnPU{}.}\label{tab:App:AbsPuVsNnpu:PUc}
\newcommand{\NResB}[2]{\textbf{\NRes{#1}{#2}}}
\newcommand{\MRTwo}[1]{\multirow{2}{*}{#1}}
\newcommand{\MRTwoS}[1]{\multirow{2}{*}{\shortstack[r]{#1}}}

\setlength{\tabcolsep}{9.0pt}

\newcommand{\StdDevRes}[4]{\ifempty{#1}{\phantom{--}}{--}#2 (\ifempty{#3}{\phantom{--}}{--}\ifempty{#4}{0\phantom{.0}}{#4})}
{
  \small
  \begin{tabular}{@{}lrrr@{}}
    \toprule
    \MRTwo{Dataset}& \MRTwo{\absPU{}}        & \MRTwo{\nnPU{}}          & \MRTwoS{\nnPU{} -- \absPU{}\\(Diff.)} \\
                   &                         &                          &                             \\\midrule
    MNIST          &   \NRes{ 6.6}{0.7}      & \NResB{ 6.5}{0.7}        & \StdDevRes{-}{0.1}{}{}      \\
    20~Newsgroups  &  \NResB{13.3}{1.3}      &  \NRes{13.5}{1.2}        & \StdDevRes{}{0.2}{-}{0.1}   \\
    CIFAR10        &  \NResB{12.4}{0.7}      & \NResB{12.4}{0.7}        & \StdDevRes{}{0}{}{}         \\\hdashline
    TREC Spam      &  \NResB{ 2.0}{1.0}      &  \NRes{ 2.1}{0.9}        & \StdDevRes{}{0.1}{-}{0.1}   \\\hdashline
    banana         &  \NResB{10.5}{1.0}      & \NResB{10.5}{1.1}        & \StdDevRes{}{0}{}{0.1}      \\
    cod-rna        &  \NResB{10.3}{1.8}      &  \NRes{10.4}{2.0}        & \StdDevRes{}{0.1}{}{0.2}    \\
    susy           &   \NRes{28.8}{1.7}      & \NResB{28.7}{1.8}        & \StdDevRes{-}{0.1}{}{0.1}   \\
    ijcnn1         &  \NResB{10.1}{1.4}      &  \NRes{10.2}{1.5}        & \StdDevRes{}{0.1}{}{0.1}    \\
    covtype.b      &  \NResB{32.8}{2.2}      &  \NRes{33.3}{2.1}        & \StdDevRes{}{0.5}{-}{0.1}   \\
    phishing       &   \NRes{ 8.6}{1.3}      & \NResB{ 8.5}{1.2}        & \StdDevRes{-}{0.1}{-}{0.1}  \\
    a9a            &  \NResB{15.9}{1.1}      &  \NRes{16.0}{1.2}        & \StdDevRes{}{0.1}{}{0.1}    \\
    connect4       &   \NRes{24.6}{2.2}      & \NResB{24.4}{2.0}        & \StdDevRes{-}{0.2}{-}{0.2}  \\
    w8a            &  \NResB{17.8}{1.6}      &  \NRes{17.9}{1.6}        & \StdDevRes{}{0.1}{}{}       \\
    epsilon        &  \NResB{31.1}{1.4}      &  \NRes{31.2}{1.7}        & \StdDevRes{}{0.1}{}{0.3}    \\
    \bottomrule
  \end{tabular}
}
 \end{table}

\subsubsection{Ordinary Positive-Unlabeled Learning Performance Under Distribution Shift}\label{sec:App:AbsPuVsNnpu:WithShift}

The previous section compared the performance of \nnPU{} and \absPU{} under ideal conditions, i.e.,~no positive shift.  This section compares \nnPU{} and \absPU{} \textit{with positive shift}, specifically under the \acronym{}~learning conditions we use in our experimental evaluation.

Like in the previous section, all classifiers in each experimental trial had identical initial weights and saw identical dataset splits.  Hyperparameters (including~$\gradAtten$) were tuned using \nnPU{}; these identical hyperparameters were then used for \absPU{} (i.e.,~not in any way tuned for absolute-value correction).  Therefore, the results again represent the \textit{performance floor} if transitioning from \nnPU{} to \absPU{}.  This choice was made due to limited time.

Recall from Section~\ref{sec:ExpRes} that evaluation baseline \nnpuOpt{} considers two \nnPU{}-based classifiers -- one trained with unlabeled set $\uset{\te}$ and the other trained with unlabeled set ${\uset{\tr} \cup \uset{\te}}$ (using the true composite prior), and we report whichever of those two classifiers performed best on average.  In this section, we introduce \abspuOpt{}, which like \nnpuOpt{}, considers two classifiers separately trained with the different unlabeled set configurations: $\uset{\te}$ and~${\uset{\tr} \cup \uset{\te}}$. The only difference is that \abspuOpt{}, as its name would suggest, uses our \absPU{}~risk estimator. We specifically separated this section to delineate the baseline performance of our contribution~(\absPU) versus existing methods~(\nnPU).

Table~\ref{tab:App:AbsPuVsNnpu:Bias:Nonidentical} compares \abspuOpt{} and \nnpuOpt{} for the extended set of experiments in Table~\ref{tab:App:AddRes:FullClassPartition} (see Section~\ref{sec:App:AddRes:Nonoverlap}).  Recall that those experiments tested cases where some positive subclasses exist only in the test distribution.  Similar to Table~\ref{tab:App:AbsPuVsNnpu:PUc}, a positive value in the column labeled ``\text{Diff.}'' denotes that \abspuOpt{} performed better than \nnpuOpt{}.

For multiple positive\=/train~(\pDsTrain) class configurations (e.g.,~MNIST\@ ~\mbox{\pDsTrain = $\set{1,3,5}$}), \abspuOpt{} and \nnpuOpt{} exhibited similar performance.  When there was a large difference between the two methods (e.g.,~20~Newsgroups \mbox{\pDsTrain = $\set{\text{misc}, \text{rec}}$}), \abspuOpt{} had significantly better mean accuracy -- reducing the misclassification rate by multiple percentage points. The difference between the methods was most pronounced when~\pDsTrain{} and~\pDsTest{} are disjoint.

These results indicate that in some cases, \abspuOpt{} is learning decision boundaries that better generalize to \textit{unseen types of data}.  To be clear, this does not apply to all datasets (CIFAR10~exhibited little difference between the methods except when the positive supports were disjoint) nor even to all class partitions within a dataset (see MNIST positive-train classes~$\set{7,9}$ versus~$\set{1,3,5}$). It should also be noted that missing positive subclasses is a more extreme form of positive shift.  The next set of results considers the more mild case of marginal-distribution magnitude shifts.

\begin{table}[h]  %
  \centering
  \caption{Comparison of inductive misclassification rate~(\%) mean and standard deviation over 100~trials for \abspuOpt{} and \nnpuOpt{} for the experimental shift tasks (eight per dataset) in Table~\ref{tab:App:AddRes:FullClassPartition} with partially/fully disjoint positive class supports. Boldface denotes the best performing task according to mean misclassification rate. For the difference column, a positive value indicates~\abspuOpt{} outperformed~\nnpuOpt{}.}\label{tab:App:AbsPuVsNnpu:Bias:Nonidentical}
\newcommand{\DsNamePU}[1]{\multirow{7}{*}{\rotatebox[origin=c]{90}{#1}}}

\newcommand{\BigArrowBase}[1]{\multicolumn{1}{c@{}}{\multirow{3}{*}{$\Bigg#1$}}}
\newcommand{\BigUpArrow}{\BigArrowBase{\uparrow}}
\newcommand{\BigDownArrow}{\BigArrowBase{\downarrow}}

\renewcommand{\arraystretch}{1.2}
\setlength{\dashlinedash}{0.4pt}
\setlength{\dashlinegap}{1.5pt}
\setlength{\arrayrulewidth}{0.3pt}

\newcommand{\NDefFullPU}[1]{\multirow{6}{*}{\shortstack[l]{#1}}}
\newcommand{\PTestDefFull}[1]{\NDefFullPU{#1}}

\setlength{\tabcolsep}{5.4pt}

{\centering
  \small
  \begin{tabular}{@{}llllllrrr@{}}
  \toprule
  \multirow{1}{*}{} & \multirow{1}{*}{N} & \multirow{1}{*}{\pDsTest} & \multirow{1}{*}{\pDsTrain} &  \multirow{1}{*}{$\prd{\tr}$} & \multirow{1}{*}{$\prd{\te}$} & \abspuOpt{}      & \nnpuOpt{}       & Diff. \\\midrule
  \DsNamePU{MNIST} & \NDefFullPU{0, 2, 4,\\ 6, 8} & \PTestDefFull{1, 3, 5,\\ 7, 9}
                              & \multirow{3}{*}{7, 9}    & 0.29 & 0.5  & \NResPUT{34.4}{2.6} &  \NResPU{36.7}{2.7} & \NResPUDiff{}{2.3}{}{0.1}    \\\cdashline{5-9}
    &       &         &                                  & 0.5  & 0.5  & \NResPUT{33.1}{2.3} &  \NResPU{35.1}{2.5} & \NResPUDiff{}{2.0}{}{0.2}    \\\cdashline{5-9}
    &       &         &                                  & 0.71 & 0.5  & \NResPUT{32.7}{2.2} &  \NResPU{34.5}{2.9} & \NResPUDiff{}{1.8}{}{0.7}    \\\cdashline{4-9}
    &       &                 & \multirow{3}{*}{1, 3, 5} & 0.38 & 0.5  & \NResPUT{25.9}{1.2} & \NResPUT{25.9}{1.1} & \NResPUDiff{}{}{-}{0.1}    \\\cdashline{5-9}
                                   &       &         &   & 0.5  & 0.5  &  \NResPU{27.1}{1.3} & \NResPUT{26.9}{1.2} & \NResPUDiff{-}{0.2}{-}{0.1}    \\\cdashline{5-9}
                                   &       &         &   & 0.63 & 0.5  &  \NResPU{28.7}{1.1} & \NResPUT{28.5}{1.2} & \NResPUDiff{-}{0.2}{}{0.1}    \\\cdashline{2-9}
      & {0, 2}  & {5, 7}      & {1, 3}                   & 0.5  & 0.5  & \NResPUT{25.7}{6.9} &  \NResPU{30.9}{5.3} & \NResPUDiff{}{5.2}{-}{1.6}    \\\midrule
  \DsNamePU{20~Newsgroups} & \NDefFullPU{sci, soc,\\talk} & \PTestDefFull{alt, comp,\\ misc, rec}
                            & \multirow{3}{*}{misc, rec} & 0.37 & 0.56 & \NResPUT{27.0}{1.9} &  \NResPU{28.8}{1.3} & \NResPUDiff{}{1.8}{-}{0.6}    \\\cdashline{5-9}
               & &            &                          & 0.56 & 0.56 & \NResPUT{26.0}{1.7} &  \NResPU{28.8}{1.7} & \NResPUDiff{}{2.8}{}{}    \\\cdashline{5-9}
               & &            &                          & 0.65 & 0.56 & \NResPUT{25.9}{1.7} &  \NResPU{29.0}{1.8} & \NResPUDiff{}{3.1}{}{0.1}    \\\cdashline{4-9}
               & &            & \multirow{3}{*}{comp}    & 0.37 & 0.56 & \NResPUT{31.2}{0.7} &  \NResPU{31.4}{0.7} & \NResPUDiff{}{0.2}{}{}    \\\cdashline{5-9}
               & &            &                          & 0.56 & 0.56 & \NResPUT{31.0}{0.9} &  \NResPU{31.2}{0.8} & \NResPUDiff{}{0.2}{-}{0.1}    \\\cdashline{5-9}
               & &            &                          & 0.65 & 0.56 & \NResPUT{31.0}{0.8} &  \NResPU{31.3}{0.7} & \NResPUDiff{}{0.3}{-}{0.1}    \\\cdashline{2-9}
               & {misc, rec} & {soc, talk} & {alt, comp} & 0.55 & 0.46 & \NResPUT{34.6}{5.0} &  \NResPU{35.3}{5.2} & \NResPUDiff{}{0.7}{}{0.2}    \\\midrule
  \DsNamePU{CIFAR10}  & \NDefFullPU{Bird, Cat,\\Deer, Dog,\\Frog, Horse} & \PTestDefFull{Plane, \\ Auto, Ship, \\ Truck}
                               & \multirow{3}{*}{Plane}  & 0.14 & 0.4  & \NResPUT{26.5}{1.0} &  \NResPU{26.7}{1.0} & \NResPUDiff{}{0.2}{}{}    \\\cdashline{5-9}
                           & & &                         & 0.4  & 0.4  & \NResPUT{27.4}{1.0} & \NResPUT{27.4}{1.0} & \NResPUDiff{}{}{}{}    \\\cdashline{5-9}
                           & & &                         & 0.6  & 0.4  & \NResPUT{28.3}{1.1} &  \NResPU{28.4}{1.0} & \NResPUDiff{}{0.1}{-}{0.1}    \\\cdashline{4-9}
    & & & \multirow{3}{*}{\shortstack[l]{Auto,\\Truck}}  & 0.25 & 0.4  & \NResPUT{20.3}{0.8} & \NResPUT{20.3}{0.8} & \NResPUDiff{}{}{}{}    \\\cdashline{5-9}
                                                   & & & & 0.4  & 0.4  &  \NResPU{20.4}{0.9} & \NResPUT{20.3}{0.8} & \NResPUDiff{-}{0.1}{-}{0.1}    \\\cdashline{5-9}
                                                   & & & & 0.55 & 0.4  &  \NResPU{20.9}{0.9} & \NResPUT{20.5}{0.9} & \NResPUDiff{-}{0.4}{}{}    \\\cdashline{2-9}
            & {Deer, Horse} & {Plane, Auto} & {Cat, Dog} & 0.5  & 0.5  & \NResPUT{44.6}{1.8} &  \NResPU{47.5}{2.0} & \NResPUDiff{}{2.9}{}{0.2}    \\
  \bottomrule
\end{tabular}
}
 \end{table}

Table~\ref{tab:App:AbsPuVsNnpu:Bias:PUc} compares \abspuOpt{} and \nnpuOpt{} for the 10~LIBSVM datasets in Table~\ref{tab:App:AddRes:PUc} (see Section~\ref{sec:App:AddRes:MarginalBias}).  Recall that in these experiments, the positive-train and positive-test class\=/conditionals have identical supports.  For seven of ten benchmarks, \abspuOpt{} had better mean performance than \nnpuOpt{} and had equivalent performance on one other benchmark.  \abspuOpt{} did have generally higher result variance.  For some benchmarks (e.g.,~\texttt{ijcnn1}, \texttt{covtype.b}, \texttt{epsilon}, etc.), the change in variance was more than offset by the improvement in mean accuracy.  Had the \abspuOpt{} learning rates been tuned directly instead of using \nnpuOpt{}'s~hyperparameter settings, we expect this variance difference would have been mitigated.  Again however, limited time prevented that experiment.

\begin{table}[h]  %
  \centering
  \caption{Comparison of inductive misclassification rate~(\%) mean and standard deviation over 100~trials for \abspuOpt{} and \nnpuOpt{} for the 10~LIBSVM datasets in Table~\ref{tab:App:AddRes:PUc} under \citet{Sakai:2019}'s mean feature vector bias. Boldface denotes the best performing task according to mean misclassification rate. For the difference column, a positive value indicates~\abspuOpt{} outperformed~\nnpuOpt{}.}\label{tab:App:AbsPuVsNnpu:Bias:PUc}
\setlength{\tabcolsep}{9.5pt}

{
  \small
  \begin{tabular}{@{}lrrrr@{}}
    \toprule
    \multirow{1}{*}{Dataset} & \multirow{1}{*}{$d$} & \abspuOpt{}         & \nnpuOpt{}       & Diff.  \\\midrule
    banana      & 2      & \NResPUT{28.5}{4.1} &  \NResPU{28.8}{3.8} & \NResPUDiff{}{0.3}{-}{0.3}    \\%
    cod-rna     & 8      &  \NResPU{25.1}{2.5} & \NResPUT{24.9}{2.3} & \NResPUDiff{-}{0.2}{-}{0.2}   \\%
    susy        & 18     & \NResPUT{45.9}{3.9} & \NResPUT{45.9}{3.9} & \NResPUDiff{}{}{}{}           \\%
    ijcnn1      & 22     & \NResPUT{33.3}{3.9} &  \NResPU{34.7}{3.6} & \NResPUDiff{}{1.4}{-}{0.3}    \\%
    covtype.b   & 54     & \NResPUT{54.6}{3.1} &  \NResPU{55.5}{2.8} & \NResPUDiff{}{0.9}{-}{0.3}    \\%
    phishing    & 68     &  \NResPU{22.9}{4.2} & \NResPUT{22.5}{4.1} & \NResPUDiff{-}{0.4}{-}{0.1}   \\%
    a9a         & 123    & \NResPUT{32.0}{2.5} &  \NResPU{32.5}{2.3} & \NResPUDiff{}{0.5}{-}{0.2}    \\%
    connect4    & 126    & \NResPUT{44.9}{3.1} &  \NResPU{45.1}{2.6} & \NResPUDiff{}{0.2}{-}{0.5}    \\%
    w8a         & 300    & \NResPUT{40.0}{4.0} &  \NResPU{41.1}{4.3} & \NResPUDiff{}{1.1}{}{0.3}     \\%
    epsilon     & 2,000  & \NResPUT{64.1}{1.4} &  \NResPU{64.6}{1.5} & \NResPUDiff{}{0.5}{}{0.1}     \\
    \bottomrule
  \end{tabular}
}
 \end{table}

In summary, \abspuOpt{}'s performance is comparable or slightly/significantly better than that of \nnpuOpt{} under \acronym{}~learning conditions that are deleterious to ordinary PU~risk estimators but that may be more realistic to real\=/world data.

\subsubsection{Effect of Absolute-Value Correction on Our \acronym{}~Learning Methods}

This section examines the effect of using absolute\=/value correction over non\=/negativity correction for our three \acronym{}~learning methods -- \nnpurpl{}, \pupnu{}, and \punu{}.  Recall that non\=/negativity correction requires custom ERM~algorithms to support ``defitting.'' Section~\ref{sec:App:ErmAlgs} describes our methods' custom ERM~frameworks when using non\=/negativity.

Due to time constraints, hyperparameter tuning was performed using non\=/negativity correction with the same hyperparameters used for the absolute\=/value based methods. Therefore, these results maximally favor the baseline of non\=/negativity correction.

Table~\ref{tab:App:AbsNn:Nonoverlap}'s experiments are identical to Table~\ref{tab:App:AddRes:FullClassPartition} in Section~\ref{sec:App:AddRes:Nonoverlap}.  ``abs''~denotes our standard \acronym{}~learning methods (see Sections~\ref{sec:WUU} and~\ref{sec:PURPL}) while ``nn''~denotes our methods modified to use \citet{Kiryo:2017}'s~non\=/negativity correction.  For~MNIST, neither absolute\=/value correction nor non\=/negativity clearly outperformed the other.  For the more challenging 20~Newsgroups and CIFAR10 datasets, absolute\=/value correction had consistently better performance than non\=/negativity.  The only exception were the disjoint support experiments and one experimental setup for \punu{} on 20~Newsgroups.  Although not shown in Table~\ref{tab:App:AddRes:FullClassPartition} due to limited space, both correction strategies had comparable variance.

\begin{table}[h]
  \centering
  \caption{Comparison of mean inductive misclassification rate~(\%) over 100~trials for the non-overlapping support experiments in Table~\ref{tab:App:AddRes:FullClassPartition} when using absolute\=/value~(abs) and non\=/negativity~(nn) corrections for our \acronym{}~learning methods.  The best performing method (according to mean misclassification rate) is shown in bold.  A positive difference~(Diff.) denotes that our absolute\=/value correction had better performance. Result standard deviations are comparable for both correction methods but are not shown here to improve table clarity.}\label{tab:App:AbsNn:Nonoverlap}
\newcommand{\DatasetHeader}[1]{\multicolumn{3}{c}{#1}}
\newcommand{\ResultHeader}{abs & nn & Diff.}

\renewcommand{\arraystretch}{1.2}
\setlength{\dashlinedash}{0.4pt}
\setlength{\dashlinegap}{1.5pt}
\setlength{\arrayrulewidth}{0.3pt}

\setlength{\tabcolsep}{8.2pt}
\newcommand{\PDefAbs}[1]{\multirow{3}{*}{\shortstack[l]{#1}}}
\newcommand{\DsName}[1]{\multirow{8}{*}{\rotatebox[origin=c]{90}{#1}}}

\newcommand{\ZPH}{0\phantom{.0}}

\newcommand{\NDefFull}[1]{\multirow{7}{*}{\shortstack[l]{#1}}}
\newcommand{\PTestDefFull}[1]{\NDefFull{#1}}

{\small
  \begin{tabular}{@{}lllllrrrrrrrrr@{}}
    \toprule
     & \multirow{2}{*}{\pDsTest} & \multirow{2}{*}{\pDsTrain} &  \multirow{2}{*}{$\prd{\tr}$} &  \multirow{2}{*}{$\prd{\te}$}& \DatasetHeader{\nnpurpl{}} & \DatasetHeader{\pupnu{}} & \DatasetHeader{\punu{}} \\\cmidrule(lr){6-8}\cmidrule(lr){9-11}\cmidrule(l){12-14}
     &                           &                            &                               &                              & \ResultHeader{}            & \ResultHeader{}          & \ResultHeader{}         \\\midrule
    \DsName{MNIST} & \PTestDefFull{1, 3, 5,\\ 7, 9} & \pDsTest{}
                                 & 0.5    & 0.5   & \textbf{10.0}  & 10.2           & 0.2     & 10.0           & \textbf{ 9.8} & --0.2      & \textbf{11.6}    & 11.7             & 0.1     \\\cdashline{3-14}
    &  & \PDefAbs{7, 9}          & 0.29   & 0.5   & 6.8            & \textbf{6.6}   & --0.2   & \textbf{ 5.3}  & \textbf{ 5.3} & \ZPH{}     & \textbf{ 6.0}    & \textbf{ 6.0}    & \ZPH{}  \\\cdashline{4-14}
    &  &                         & 0.5    & 0.5   & \textbf{ 9.4}  & \textbf{9.4}   & \ZPH{}  & \textbf{ 7.1}  & \textbf{ 7.1} & \ZPH{}     & \textbf{ 8.3}    & \textbf{ 8.3}    & \ZPH{}  \\\cdashline{4-14}
    &  &                         & 0.71   & 0.5   & \textbf{14.0}  & 14.6           & 0.6     & \textbf{11.1}  &         11.3  & 0.2        & \textbf{14.8}    & 15.2             & 0.4     \\\cdashline{3-14}
    &  & \PDefAbs{1, 3, 5}       & 0.38   & 0.5   & 8.1            & \textbf{8.0}   & --0.1   & \textbf{ 6.5}  & \textbf{ 6.5} & \ZPH{}     & \textbf{ 7.6}    & 7.7              & 0.1     \\\cdashline{4-14}
    &  &                         & 0.5    & 0.5   & 10.0           & \textbf{9.9}   & --0.1   & \textbf{ 8.4}  & \textbf{ 8.4} & \ZPH{}     & \textbf{10.2}    & \textbf{10.2}    & \ZPH{}  \\\cdashline{4-14}
    &  &                         & 0.63   & 0.5   & \textbf{12.5}  & 12.9           & 0.4     & \textbf{11.4}  & \textbf{11.4} & \ZPH{}     & \textbf{14.3}    & 14.5             & 0.2     \\\cdashline{3-14}
    & 5, 7  & 1, 3               & 0.5    & 0.5   &          4.0   & \textbf{ 3.9}  & --0.1   & \textbf{ 3.6}  & \textbf{ 3.6} & \ZPH{}     & \textbf{ 3.1}    &          3.2     & 0.1     \\\midrule
  \DsName{20~Newsgroups} & \PTestDefFull{alt, comp,\\ misc, rec} & \pDsTest{}
                                 & 0.56   & 0.56  & \textbf{15.4}  & 15.5           & 0.1     & \textbf{14.9}  & 15.0          & 0.1        & \textbf{16.7}    & \textbf{16.7}    & \ZPH{}  \\\cdashline{3-14}
    &  & \PDefAbs{misc,\\rec}    & 0.37   & 0.56  & \textbf{13.9}  & \textbf{13.9}  & \ZPH{}  & \textbf{12.8}  & \textbf{12.8} & \ZPH{}     & \textbf{14.3}    & \textbf{14.3}    & \ZPH{}  \\\cdashline{4-14}
    &  &                         & 0.56   & 0.56  & \textbf{17.5}  & 17.7           & 0.2     & \textbf{13.5}  & \textbf{13.5} & \ZPH{}     & \textbf{15.1}    & \textbf{15.1}    & \ZPH{}  \\\cdashline{4-14}
    &  &                         & 0.65   & 0.56  & \textbf{20.2}  & 20.8           & 0.6     & \textbf{14.0}  & \textbf{14.0} & \ZPH{}     & \textbf{15.9}    & \textbf{15.9}    & \ZPH{}  \\\cdashline{3-14}
    &  & \PDefAbs{comp}          & 0.37   & 0.56  & \textbf{13.3}  & \textbf{13.3}  & \ZPH{}  & \textbf{13.7}  & \textbf{13.7} & \ZPH{}     & 14.5             & \textbf{14.4}    & --0.1    \\\cdashline{4-14}
    &  &                         & 0.56   & 0.56  & \textbf{16.0}  & 16.5           & 0.5     & \textbf{14.9}  & \textbf{14.9} & \ZPH{}     & \textbf{15.7}    & \textbf{15.7}    & \ZPH{}  \\\cdashline{4-14}
    &  &                         & 0.65   & 0.56  & \textbf{19.2}  & 19.6           & 0.4     & \textbf{15.6}  & \textbf{15.6} & \ZPH{}     & \textbf{16.5}    & \textbf{16.5}    & \ZPH{}  \\\cdashline{3-14}
    & soc, talk & alt, comp      & 0.55   & 0.46  & 5.9            & \textbf{5.8}   & --0.1   & \textbf{ 7.1}  & \textbf{ 7.1} & \ZPH{}     & \textbf{ 5.6}    & 5.7              & 0.1     \\\midrule
  \DsName{CIFAR10} & \PTestDefFull{Plane, \\ Auto, Ship, \\ Truck} & \pDsTest{}
                                 & 0.4    & 0.4   & \textbf{14.1}  & 14.3           &   0.2   & \textbf{14.2}  &         14.4  & 0.2        & \textbf{15.4}    & 15.8             & 0.4     \\\cdashline{3-14}
    &  & \PDefAbs{Plane}         & 0.14   & 0.4   & \textbf{11.9}  & 12.0           &   0.1   & \textbf{11.9}  &         12.0  & 0.1        & \textbf{12.4}    & \textbf{12.4}    & \ZPH{}  \\\cdashline{4-14}
    &  &                         & 0.4    & 0.4   & \textbf{13.8}  & 14.0           &   0.2   & \textbf{14.5}  &         14.6  & 0.1        & \textbf{15.1}    & 15.5             & 0.4     \\\cdashline{4-14}
    &  &                         & 0.6    & 0.4   & \textbf{16.1}  & 16.6           &   0.5   & \textbf{16.7}  &         17.1  & 0.4        & \textbf{20.0}    & 20.2             & 0.2     \\\cdashline{3-14}
    &  & \PDefAbs{Auto,\\Truck}  & 0.25   & 0.4   & \textbf{12.7}  & 12.8           &   0.1   & \textbf{12.4}  &         12.5  & 0.1        & \textbf{12.8}    & 13.0             & 0.2     \\\cdashline{4-14}
    &  &                         & 0.4    & 0.4   & \textbf{14.1}  & 14.3           &   0.2   & \textbf{13.9}  &         14.0  & 0.1        & \textbf{14.4}    & 14.6             & 0.2     \\\cdashline{4-14}
    &  &                         & 0.55   & 0.4   & \textbf{16.0}  & 16.4           &   0.4   & \textbf{16.2}  &         16.3  & 0.1        & \textbf{17.1}    & 17.4             & 0.3     \\\cdashline{3-14}
    & Plane, Auto & Cat, Dog     & 0.5    & 0.5   & 14.1           & \textbf{14.0}  & --0.1   &         14.9   & \textbf{14.8} & --0.1      & \textbf{11.2}    & 11.3             & 0.1     \\
    \bottomrule
  \end{tabular}
}
 \end{table}

Table~\ref{tab:App:AbsNn:PUc}'s experiments match the experimental conditions for the 10~LIBSVM datasets in Table~\ref{tab:App:AddRes:PUc} from Section~\ref{sec:App:AddRes:MarginalBias}.  Biasing follows \citet{Sakai:2019}'s median feature vector-based approach.  Neither absolute\=/value nor non\=/negativity correction consistently outperformed the other in these LIBSVM experiments. Note though that since absolute\=/value correction is a simpler method with one less hyperparameter,~$\gradAtten$, to tune, comparable performance implicitly favors absolute\=/value correction over non\=/negativity.

\begin{table}[h]
  \centering
  \caption{Comparison of inductive misclassification rate~(\%) mean and standard deviation over 100~trials for Table~\ref{tab:App:AddRes:PUc}'s LIBSVM dataset experiments using \citeauthor{Sakai:2019}'s mean feature vector biasing with absolute\=/value~(abs) and non\=/negativity~(nn) corrections for our \acronym{}~learning methods.  The best performing method (according to mean misclassification rate) is shown in bold.  A positive difference~(Diff.) denotes that our absolute\=/value correction had better performance than non\=/negativity correction.}\label{tab:App:AbsNn:PUc}
\newcommand{\DatasetHeader}[1]{\multicolumn{3}{c}{#1}}
\newcommand{\ResultHeader}{abs & nn & Diff.}

\renewcommand{\arraystretch}{1.2}
\setlength{\dashlinedash}{0.4pt}
\setlength{\dashlinegap}{1.5pt}
\setlength{\arrayrulewidth}{0.3pt}

\setlength{\tabcolsep}{5.8pt}

\newcommand{\ResAbs}[2]{#1 (#2)}  %
\newcommand{\ResAbsT}[2]{\textbf{\ResAbs{#1}{#2}}}  %
\newcommand{\ResAbsDiff}[4]{\ifempty{#1}{\phantom{--}}{--}\ifempty{#2}{0\phantom{.0}}{#2} (\ifempty{#3}{\phantom{--}}{--}\ifempty{#4}{0\phantom{.0}}{#4})}  %

{\footnotesize
  \begin{tabular}{@{}lrrrrrrrrr@{}}
    \toprule
    \multirow{2}{*}{Dataset} & \DatasetHeader{\nnpurpl{}} & \DatasetHeader{\pupnu{}} & \DatasetHeader{\punu{}} \\\cmidrule(lr){2-4}\cmidrule(lr){5-7}\cmidrule(l){8-10}
                             & \ResultHeader{}            & \ResultHeader{}          & \ResultHeader{}         \\\midrule
    banana     & \ResAbsT{12.9}{2.1}   & \ResAbsT{12.9}{2.2}   & \ResAbsDiff{}{}{}{0.1}      &  \ResAbs{11.8}{1.6}  & \ResAbsT{11.7}{1.6}   & \ResAbsDiff{-}{0.1}{}{}     & \ResAbsT{13.3}{2.3}   &  \ResAbs{14.0}{2.3}   & \ResAbsDiff{}{0.7}{}{}       \\
    cod-rna    &  \ResAbs{14.7}{2.6}   & \ResAbsT{14.6}{2.9}   & \ResAbsDiff{-}{0.1}{}{0.3}  & \ResAbsT{15.1}{3.2}  & \ResAbsT{15.1}{3.2}   & \ResAbsDiff{}{}{}{}         & \ResAbsT{15.5}{2.9}   & \ResAbsT{15.5}{3.3}   & \ResAbsDiff{}{}{}{0.4}       \\
    susy       & \ResAbsT{24.2}{2.1}   &  \ResAbs{24.6}{2.1}   & \ResAbsDiff{}{0.4}{}{}      & \ResAbsT{25.6}{2.2}  & \ResAbsT{25.6}{2.2}   & \ResAbsDiff{}{}{}{}         & \ResAbsT{25.8}{2.2}   &  \ResAbs{26.0}{2.1}   & \ResAbsDiff{}{0.2}{-}{0.1}   \\
    ijcnn1     & \ResAbsT{22.7}{2.8}   &  \ResAbs{23.0}{2.8}   & \ResAbsDiff{}{0.3}{}{}      & \ResAbsT{17.7}{2.8}  &  \ResAbs{19.0}{2.9}   & \ResAbsDiff{}{1.3}{}{0.1}   & \ResAbsT{24.6}{3.1}   &  \ResAbs{24.9}{2.9}   & \ResAbsDiff{}{0.3}{-}{0.2}   \\
    covtype.b  & \ResAbsT{29.5}{2.9}   &  \ResAbs{29.6}{2.9}   & \ResAbsDiff{}{0.1}{}{}      & \ResAbsT{32.5}{3.2}  &  \ResAbs{32.6}{3.1}   & \ResAbsDiff{}{0.1}{-}{0.1}  & \ResAbsT{29.9}{2.4}   &  \ResAbs{30.1}{2.7}   & \ResAbsDiff{}{0.2}{}{0.3}    \\
    phishing   & \ResAbsT{11.3}{1.4}   &  \ResAbs{11.9}{1.4}   & \ResAbsDiff{}{0.6}{}{}      & \ResAbsT{ 9.6}{1.0}  & \ResAbsT{ 9.6}{1.0}   & \ResAbsDiff{}{}{}{}         & \ResAbsT{11.1}{1.8}   &  \ResAbs{11.7}{1.9}   & \ResAbsDiff{}{0.6}{}{0.1}    \\
    a9a        &  \ResAbs{27.1}{2.1}   & \ResAbsT{27.0}{2.1}   & \ResAbsDiff{-}{0.1}{}{}     &  \ResAbs{26.6}{1.8}  & \ResAbsT{26.5}{1.8}   & \ResAbsDiff{-}{0.1}{}{}     &  \ResAbs{27.1}{2.1}   & \ResAbsT{27.0}{2.0}   & \ResAbsDiff{-}{0.1}{-}{0.1}  \\
    connect4   &  \ResAbs{34.9}{3.1}   & \ResAbsT{34.2}{2.6}   & \ResAbsDiff{-}{0.7}{-}{0.5} & \ResAbsT{32.9}{2.7}  &  \ResAbs{33.0}{2.7}   & \ResAbsDiff{}{0.1}{}{}      &  \ResAbs{35.0}{2.9}   & \ResAbsT{34.9}{2.6}   & \ResAbsDiff{-}{0.1}{-}{0.3}  \\
    w8a        &  \ResAbs{17.2}{2.6}   & \ResAbsT{17.1}{2.4}   & \ResAbsDiff{-}{0.1}{-}{0.2} &  \ResAbs{21.0}{2.9}  & \ResAbsT{20.3}{2.9}   & \ResAbsDiff{-}{0.7}{}{}     & \ResAbsT{16.8}{2.9}   &  \ResAbs{18.4}{2.7}   & \ResAbsDiff{}{1.6}{-}{0.2}   \\
    epsilon    &  \ResAbs{33.5}{4.8}   & \ResAbsT{32.7}{3.1}   & \ResAbsDiff{-}{0.8}{-}{1.7} & \ResAbsT{36.5}{5.0}  &  \ResAbs{37.8}{6.9}   & \ResAbsDiff{}{1.3}{}{1.9}   &  \ResAbs{31.5}{1.7}   & \ResAbsT{31.3}{1.7}   & \ResAbsDiff{-}{0.2}{}{}      \\
    \bottomrule
  \end{tabular}
}
 \end{table}
 
\suppressfloats
\clearpage
\newpage
\subsection{Alternate Methods for Step~\#1 of Our Two-Step Methods}\label{sec:App:AddRes:AltStep1}

Recall from Section~\ref{sec:WUU:Step1} that our two-step methods' first step transform unlabeled training set~$\uset{\tr}$ into surrogate negative set~$\nsetPS{}$ by \emph{soft weighting} each ${\X \in \uset{\tr}}$ using classifier
\begin{equation}\label{eq:App:AddRes:SoftWeighting}
  \cdcxSoft \defeq \cdcx \approx \ppost{\tr}{\yn} \text{.}
\end{equation}
\noindent
In this section, we propose and empirically evaluate two alternative step~\#1 methods -- hard and top\K{} weighting. Regardless of which step~\#1 method is used to create~$\nsetPS{}$, no changes are required to our step~\#2 risk estimators -- \wuu{} and~\bpnu\@.

\subsubsection{Overview of the Alternate Step~\#1 Methods}\label{sec:App:AddRes:AltStep1:Methods}

\paragraph{Hard Weighting} \citet{Guo:2017}~show that modern neural networks are generally poorly calibrated and tend to report ``peaky'' confidence estimates.  $\cdc$~is vulnerable to similar ``peaky'' behavior. \textit{Hard weighting} assigns each unlabeled training example,~${\X \in \uset{\tr}}$, weight
\begin{equation}\label{eq:App:AddRes:HardWeighting}
  \cdcxHard \defeq \round{\cdcx}
\end{equation}
\noindent
where for ${a \in \real}$, ${\round{a}}$~rounds $a$~to the nearest integer (i.e.,~0 or~1 for probabilistic classifier~$\cdc$).

Hard weighting simulates worst\=/case ``peaked'' behavior.  Although not statistically consistent for non\=/separable data, hard weighting may sometimes outperform soft-weighting due to its thresholding effect.

\paragraph{Top\K{} Weighting} To broadly summarize \citeauthor{Guo:2017}'s primary contribution, neural network probability estimates may be inaccurate.  Our \emph{top\K{} weighting} method attempts to overcome that inaccuracy by focusing, not on the specific probability values predicted by~$\cdc$, but instead on the \textit{ordering} of those posterior estimates.

By definition, the expected number of positive-labeled examples in~$\uset{\tr}$ is~${\prd{\tr} \cdot \nTrU}$, where ${\nTrU \defeq \abs{\uset{\tr}}}$ is the unlabeled training set size and $\prd{\tr}$~is the positive training prior.  Define ${\KSym \defeq \round{\prd{\tr} \cdot \nTrU} \in \intsNN}$.  After training~$\cdc$ (same as before), let set~$\usetTrK$ be the \KSym~examples in~$\uset{\tr}$ with the highest predicted posteriors according to~$\cdc$.  Top\K{} weighting assigns weight~1 to any ${\X \in \usetTrK}$ and weight~0 to any ${\X \in \left( \uset{\tr} \setminus \usetTrK \right)}$.  Formally, for any ${\X \in \uset{\tr}}$,
  \begin{equation}\label{eq:App:AddRes:TopKWeighting}
    \cdcxTopK \defeq \begin{cases}
        1 & \X \in \usetTrK  \\
        0 & \text{Otherwise}
      \end{cases} \text{.}
  \end{equation}

Observe that top\K{} weighting uses strictly more information than both soft and hard weighting.  However, by relying on~$\prd{\tr}$ to estimate~\KSym, top\K{} weighting is generally more deleteriously affected by misestimation of~$\prd{\tr}$.

\subsubsection{Step~\#1 Labeling Accuracy}

These experiments examine how accurately our three proposed step~\#1 methods label~$\uset{\tr}$.  The labeling error rate is defined as
\begin{equation}\label{eq:App:AddRes:Step1ErrorRate}
  \text{Error Rate}_{\mathcal{M}} \defeq \frac{100\%}{\nTrU} \sum_{\X \in \uset{\tr}} \frac{\abs{2\cdcxAlt{\mathcal{M}} - 1 - \y_{\X}}}{2} \text{,}
\end{equation}
\noindent
where ${\y_{\X} \in \domainY}$ is unlabeled training example~$\X$'s true (unknown) label and ${\mathcal{M} \in \set{\text{soft}, \text{hard}, \text{top\K}}}$ denotes the step~\#1 method.  For hard and top\K{} weighting, Eq.~\eqref{eq:App:AddRes:Step1ErrorRate} corresponds to their (scaled) transductive misclassification rate on~$\uset{\tr}$.  Note that the difference between the soft and hard weightings' labeling error rates is indicative of the ``peakiness'' of~$\cdc$, with a smaller gap indicating that $\cdc$'s~estimates are more peaked.

For all experiments in this section, the three step~\#1 methods saw identical dataset splits and used the same initial model parameters.

\paragraph{Analysis} Table~\ref{tab:AltS1:S1Acc:PUc} compares the three weighting methods' step~\#1 labeling error rate for the 10~LIBSVM datasets in Table~\ref{tab:App:AddRes:PUc} (see Section~\ref{sec:App:AddRes:MarginalBias}).  Recall that in these experiments, the positive-train and positive-test class conditionals have identical supports.  The step~\#1 methods' labeling error rates varied widely from around~10\% on the \texttt{phishing} dataset to 30\==40\% for the \texttt{covtype.b} and \texttt{epsilon} datasets.

Recall from Table~\ref{tab:App:AddRes:PUc} that \nnpurpl{} was the top performer for the \texttt{cod\=/rna}, \texttt{susy}, and \texttt{covtype.b} datasets. The step~\#1 labeling error on those three datasets ranged from moderate to poor.  However, \texttt{epsilon} had soft weighting's worst step~\#1 labeling error rate yet \punu{} still outperformed \nnpurpl{} (see Table~\ref{tab:App:AddRes:PUc}).  This demonstrates that step~\#1 labeling accuracy alone does not determine which algorithm class, i.e.,~two\=/step or joint, is best.

\begin{table}[t]
  \centering
  \caption{Comparison of the soft, hard, and top\K{} weighting schemes' step~\#1 labeling error rate mean and standard deviation across 100~trials for the 10~LIBSVM datasets in Table~\ref{tab:App:AddRes:PUc}.}\label{tab:AltS1:S1Acc:PUc}
\renewcommand{\arraystretch}{1.2}
\setlength{\dashlinedash}{0.4pt}
\setlength{\dashlinegap}{1.5pt}
\setlength{\arrayrulewidth}{0.3pt}

\setlength{\tabcolsep}{8.5pt}

{\footnotesize
  \begin{tabular}{@{}lrrrr@{}}
    \toprule
    Dataset    & $d$    & Soft                  & Hard                  & Top\K    \\\midrule
    banana     & 2      & \ResTSAcc{20.1}{3.8}  & \ResTSAcc{13.2}{1.8}  & \ResTSAcc{12.4}{1.8}   \\
    cod-rna    & 8      & \ResTSAcc{20.8}{4.0}  & \ResTSAcc{13.2}{1.9}  & \ResTSAcc{12.7}{1.7}   \\
    susy       & 18     & \ResTSAcc{39.6}{2.7}  & \ResTSAcc{30.7}{2.4}  & \ResTSAcc{30.6}{2.4}   \\
    ijcnn1     & 22     & \ResTSAcc{27.0}{4.5}  & \ResTSAcc{19.3}{2.5}  & \ResTSAcc{15.8}{2.7}   \\
    covtype.b  & 54     & \ResTSAcc{44.1}{4.2}  & \ResTSAcc{37.1}{3.3}  & \ResTSAcc{34.9}{2.7}   \\
    phishing   & 68     & \ResTSAcc{13.5}{3.6}  & \ResTSAcc{10.4}{1.5}  & \ResTSAcc{ 9.6}{1.1}   \\
    a9a        & 123    & \ResTSAcc{24.4}{3.8}  & \ResTSAcc{17.4}{1.4}  & \ResTSAcc{18.1}{1.8}   \\
    connect4   & 128    & \ResTSAcc{34.2}{5.5}  & \ResTSAcc{27.8}{4.4}  & \ResTSAcc{24.7}{2.4}   \\
    w8a        & 300    & \ResTSAcc{27.9}{5.1}  & \ResTSAcc{19.6}{1.6}  & \ResTSAcc{18.5}{1.7}   \\
    epsilon    & 2,000  & \ResTSAcc{44.2}{1.5}  & \ResTSAcc{32.5}{2.6}  & \ResTSAcc{33.7}{2.0}   \\
    \bottomrule
  \end{tabular}
}
 \end{table}

Table~\ref{tab:AltS1:S1Acc:Nonoverlap} compares the three weighting methods' step~\#1 labeling error rate for the extended set of experiments in Table~\ref{tab:App:AddRes:FullClassPartition} (see Section~\ref{sec:App:AddRes:Nonoverlap}).  Recall that those experiments, on datasets MNIST, 20~Newsgroups, and~CIFAR10, replicated scenarios where some positive subclasses exist only in the test distribution.  As expected, the easier MNIST dataset had better step~\#1 labeling accuracy than the more challenging 20~Newsgroups and CIFAR10 datasets.  On the whole, these three datasets had better average step~\#1 labeling error rate than the 10~LIBSVM datasets discussed above.

\begin{table}[t]
  \centering
  \caption{Comparison of the soft, hard, and top\K{} weighting schemes' step~\#1 labeling error rate mean and standard deviation across 100~trials for Table~\ref{tab:App:AddRes:FullClassPartition}'s experiments on partially/fully disjoint positive-class support for MNIST, 20~Newsgroups, and CIFAR10.}\label{tab:AltS1:S1Acc:Nonoverlap}
\newcommand{\DsName}[1]{\multirow{8}{*}{\rotatebox[origin=c]{90}{#1}}}

\newcommand{\BigArrowBase}[1]{\multicolumn{1}{c@{}}{\multirow{3}{*}{$\Bigg#1$}}}
\newcommand{\BigUpArrow}{\BigArrowBase{\uparrow}}
\newcommand{\BigDownArrow}{\BigArrowBase{\downarrow}}

\renewcommand{\arraystretch}{1.2}
\setlength{\dashlinedash}{0.4pt}
\setlength{\dashlinegap}{1.5pt}
\setlength{\arrayrulewidth}{0.3pt}

\newcommand{\NDefFull}[1]{\multirow{7}{*}{\shortstack[l]{#1}}}
\newcommand{\PTestDefFull}[1]{\NDefFull{#1}}

\setlength{\tabcolsep}{8.4pt}

{\centering
  \footnotesize
  \begin{tabular}{@{}llllllrrr@{}}
  \toprule
   & N & \pDsTest{} & \pDsTrain{} &  $\prd{\tr}$ & $\prd{\te}$ & Soft & Hard & Top\K \\\midrule
  \DsName{MNIST} & \NDefFull{0, 2, 4,\\ 6, 8} & \PTestDefFull{1, 3, 5,\\ 7, 9} & \SamePTest{}
                  & 0.5  & 0.5  & \ResTSAcc{16.2}{2.3}  & \ResTSAcc{12.6}{1.0}  & \ResTSAcc{10.6}{1.0}   \\\cdashline{4-9}
    &     &    &  \multirow{3}{*}{7, 9}
                  & 0.29 & 0.5  & \ResTSAcc{11.0}{2.0}  & \ResTSAcc{ 6.8}{0.8}  & \ResTSAcc{ 5.4}{0.5}   \\\cdashline{5-9}
    &     &    &  & 0.5  & 0.5  & \ResTSAcc{10.8}{1.9}  & \ResTSAcc{ 8.3}{0.9}  & \ResTSAcc{ 6.5}{0.6}   \\\cdashline{5-9}
    &     &    &  & 0.71 & 0.5  & \ResTSAcc{11.1}{2.4}  & \ResTSAcc{ 8.3}{0.5}  & \ResTSAcc{ 7.7}{0.6}   \\\cdashline{4-9}
    &     &    & \multirow{3}{*}{1, 3, 5}
                  & 0.38 & 0.5  & \ResTSAcc{12.6}{1.6}  & \ResTSAcc{ 9.0}{0.9}  & \ResTSAcc{ 7.3}{0.7}   \\\cdashline{5-9}
    &     &    &  & 0.5  & 0.5  & \ResTSAcc{14.0}{2.4}  & \ResTSAcc{10.7}{0.9}  & \ResTSAcc{ 9.0}{1.0}   \\\cdashline{5-9}
    &     &    &  & 0.63 & 0.5  & \ResTSAcc{15.0}{3.1}  & \ResTSAcc{11.4}{0.7}  & \ResTSAcc{10.3}{1.0}   \\\cdashline{2-9}
    & {0, 2}  & {5, 7}  & {1, 3}
                  & 0.5  & 0.5  & \ResTSAcc{ 8.9}{1.7}  & \ResTSAcc{ 6.5}{0.7}  & \ResTSAcc{ 5.4}{0.4}   \\\midrule
  \DsName{20~Newsgroups} & \NDefFull{sci, soc,\\talk} & \PTestDefFull{alt, comp,\\ misc, rec} & \SamePTest {}
                  & 0.56 & 0.56 & \ResTSAcc{23.1}{4.3}  & \ResTSAcc{16.9}{1.3}  & \ResTSAcc{16.5}{1.3}   \\\cdashline{4-9}
    & &  &  \multirow{3}{*}{misc, rec}
                  & 0.37 & 0.56 & \ResTSAcc{15.5}{1.3}  & \ResTSAcc{12.0}{1.4}  & \ResTSAcc{ 9.3}{1.3}   \\\cdashline{5-9}
    & &  &        & 0.56 & 0.56 & \ResTSAcc{14.1}{1.5}  & \ResTSAcc{11.6}{1.0}  & \ResTSAcc{10.2}{1.2}   \\\cdashline{5-9}
    & &  &        & 0.65 & 0.56 & \ResTSAcc{12.8}{1.5}  & \ResTSAcc{10.3}{0.9}  & \ResTSAcc{ 9.8}{1.1}   \\\cdashline{4-9}
    & &  & \multirow{3}{*}{comp}
                  & 0.37 & 0.56 & \ResTSAcc{15.7}{0.9}  & \ResTSAcc{12.1}{0.8}  & \ResTSAcc{11.1}{1.0}   \\\cdashline{5-9}
    & &  &        & 0.56 & 0.56 & \ResTSAcc{14.3}{1.2}  & \ResTSAcc{12.0}{1.1}  & \ResTSAcc{11.6}{1.2}   \\\cdashline{5-9}
    & &  &        & 0.65 & 0.56 & \ResTSAcc{12.8}{1.3}  & \ResTSAcc{10.7}{1.1}  & \ResTSAcc{11.1}{1.3}   \\\cdashline{2-9}
    & {misc, rec} & {soc, talk} & {alt, comp}
                  & 0.55 & 0.46 & \ResTSAcc{12.4}{1.0}  & \ResTSAcc{10.6}{1.1}  & \ResTSAcc{10.1}{1.2}   \\\midrule
  \DsName{CIFAR10}  & \NDefFull{Bird, Cat,\\Deer, Dog,\\Frog, Horse} & \PTestDefFull{Plane, \\ Auto, Ship, \\ Truck} & \SamePTest{}
                  & 0.4  & 0.4  & \ResTSAcc{21.3}{3.1}  & \ResTSAcc{16.6}{1.3}  & \ResTSAcc{14.6}{0.7}   \\\cdashline{4-9}
    & & & \multirow{3}{*}{Plane}
                  & 0.14 & 0.4  & \ResTSAcc{16.6}{3.1}  & \ResTSAcc{ 9.2}{0.6}  & \ResTSAcc{ 8.4}{0.5}   \\\cdashline{5-9}
    & & &         & 0.4  & 0.4  & \ResTSAcc{21.3}{4.0}  & \ResTSAcc{14.7}{1.1}  & \ResTSAcc{13.1}{0.7}   \\\cdashline{5-9}
    & & &         & 0.6  & 0.4  & \ResTSAcc{21.8}{3.0}  & \ResTSAcc{15.0}{1.0}  & \ResTSAcc{13.8}{0.6}   \\\cdashline{4-9}
    & & & \multirow{3}{*}{\shortstack[l]{Auto,\\Truck}}
                  & 0.25 & 0.4  & \ResTSAcc{14.7}{1.4}  & \ResTSAcc{10.4}{0.8}  & \ResTSAcc{ 9.1}{0.5}   \\\cdashline{5-9}
    & & &         & 0.4  & 0.4  & \ResTSAcc{15.7}{2.6}  & \ResTSAcc{12.5}{1.1}  & \ResTSAcc{10.9}{0.7}   \\\cdashline{5-9}
    & & &         & 0.55 & 0.4  & \ResTSAcc{17.0}{3.3}  & \ResTSAcc{13.3}{1.2}  & \ResTSAcc{11.8}{0.6}   \\\cdashline{2-9}
    & {Deer, Horse} & {Plane, Auto} & {Cat, Dog}
                  & 0.5  & 0.5  & \ResTSAcc{30.9}{2.6}  & \ResTSAcc{21.4}{1.1}  & \ResTSAcc{21.0}{1.0}   \\
  \bottomrule
\end{tabular}
}
 \end{table}

\subsubsection{Step~\#1 Method's Effect on Overall Two-Step Performance}

These experiments study how each step~\#1 method affects our two step methods' -- \pupnu{} and \punu{} -- inductive, test (i.e.,~end-to-end) misclassification rate.  As in the previous section, all methods saw identical dataset splits and initial model parameters in each experimental trial.

\paragraph{Analysis} Table~\ref{tab:AltS1:PUc} compares the two-step inductive misclassification rate when using the three step~\#1 methods for the 10~LIBSVM datasets in Table~\ref{tab:App:AddRes:PUc} (see Section~\ref{sec:App:AddRes:MarginalBias}).  Soft weighting was the best performing method for all ten datasets for \pupnu{} and for seven of ten datasets for \punu{}.  It is also noteworthy that only soft weighting learned a meaningful classifier for the \texttt{epsilon} dataset. In fact, hard and top\K{} weighting performed worse than random chance for \texttt{epsilon}.

\begin{table}[t]
  \centering
  \caption{Effect of step~\#1 method on our two-step methods' overall inductive misclassification rate~(\%) for the 10~LIBSVM datasets in Table~\ref{tab:App:AddRes:PUc}.  The table's upper half reports each method's misclassification rate mean and standard deviation over 100~trials. Boldface denotes each experimental setup's best performing method according to mean misclassification rate.  The table's lower half is an alternate visualization showing the difference (Diff.) in misclassification rate mean and standard deviation w.r.t.\ to our soft method.  {\color{BrickRed}Red} denotes that the associated alternate step~\#1 method had worse (i.e.,~higher) mean misclassification rate than soft weighting while {\color{ForestGreen} green} denotes that the alternate method had a better (i.e.,~lower) mean misclassification rate.}\label{tab:AltS1:PUc}
\renewcommand{\arraystretch}{1.2}
\setlength{\dashlinedash}{0.4pt}
\setlength{\dashlinegap}{1.5pt}
\setlength{\arrayrulewidth}{0.3pt}

\setlength{\tabcolsep}{8.1pt}

{\footnotesize
  \begin{tabular}{@{}lrrrrrrr@{}}
    \toprule
    \multirow{2}{*}{Dataset} & \multirow{2}{*}{$d$} & \TSAlgHeader{\pupnu{}}  & \TSAlgHeader{\punu{}} \\\cmidrule(lr){3-5}\cmidrule(l){6-8}
                             &                      & \TSResHeaderAlt{}          & \TSResHeaderAlt{}         \\\midrule
    banana     & 2     & \ResTSB{11.7}{1.6}  & \ResTSN{13.4}{2.0}    & \ResTSN{13.1}{1.9}     & \ResTSN{13.4}{2.4}   & \ResTSN{14.0}{2.6}   & \ResTSB{13.3}{2.5}  \\
    cod-rna    & 8     & \ResTSB{14.6}{3.8}  & \ResTSN{18.6}{3.7}    & \ResTSN{18.3}{3.9}     & \ResTSB{15.5}{3.2}   & \ResTSN{17.4}{3.2}   & \ResTSN{16.7}{3.2}  \\
    susy       & 18    & \ResTSB{25.8}{2.6}  & \ResTSN{27.8}{3.2}    & \ResTSN{27.6}{2.5}     & \ResTSB{25.8}{2.4}   & \ResTSN{26.3}{3.5}   & \ResTSN{26.1}{2.6}  \\
    ijcnn1     & 22    & \ResTSB{18.0}{2.7}  & \ResTSN{22.1}{3.6}    & \ResTSN{18.4}{2.9}     & \ResTSN{25.1}{3.4}   & \ResTSN{24.3}{3.3}   & \ResTSB{21.4}{2.9}  \\
    covtype.b  & 54    & \ResTSB{32.1}{3.2}  & \ResTSN{40.6}{3.5}    & \ResTSN{37.0}{3.5}     & \ResTSB{29.7}{2.5}   & \ResTSN{40.1}{3.7}   & \ResTSN{34.8}{4.2}  \\
    phishing   & 68    & \ResTSB{ 9.6}{0.9}  & \ResTSN{ 9.8}{1.0}    & \ResTSN{10.0}{1.1}     & \ResTSN{11.6}{2.1}   & \ResTSN{10.9}{1.4}   & \ResTSB{10.1}{1.2}  \\
    a9a        & 123   & \ResTSB{26.8}{1.6}  & \ResTSN{28.6}{1.7}    & \ResTSN{27.9}{1.6}     & \ResTSB{27.4}{2.1}   & \ResTSN{28.5}{1.8}   & \ResTSN{28.3}{1.9}  \\
    connect4   & 126   & \ResTSB{32.9}{2.1}  & \ResTSN{37.2}{2.8}    & \ResTSN{35.6}{2.3}     & \ResTSB{34.8}{2.7}   & \ResTSN{38.0}{3.1}   & \ResTSN{35.3}{2.8}  \\
    w8a        & 300   & \ResTSB{21.6}{2.4}  & \ResTSN{23.7}{2.0}    & \ResTSN{24.6}{2.3}     & \ResTSB{16.9}{2.7}   & \ResTSN{22.4}{2.4}   & \ResTSN{22.0}{2.7}  \\
    epsilon    & 2,000 & \ResTSB{35.0}{4.4}  & \ResTSN{58.6}{3.2}    & \ResTSN{54.6}{3.4}     & \ResTSB{31.2}{1.1}   & \ResTSN{52.6}{3.9}   & \ResTSN{52.9}{5.8}  \\
    \bottomrule
    \toprule
    \multirow{2}{*}{Dataset} & \multirow{2}{*}{$d$} & \TSAlgHeader{\pupnu{}}  & \TSAlgHeader{\punu{}} \\\cmidrule(lr){3-5}\cmidrule(l){6-8}
                             &                      & \TSResHeader{}          & \TSResHeader{}         \\\midrule
    banana     & 2     & \ResTSN{11.7}{1.6}  & \ResTSDWW{1.7}{0.4}   & \ResTSDWW{1.4}{0.4}    & \ResTSN{13.4}{2.4}   & \ResTSDWW{0.5}{0.2}  & \ResTSDBW{0.1}{0.1}  \\
    cod-rna    & 8     & \ResTSN{14.6}{3.8}  & \ResTSDWB{4.1}{0.1}   & \ResTSDWW{3.7}{0.1}    & \ResTSN{15.5}{3.2}   & \ResTSDWW{1.9}{\ZR}  & \ResTSDWW{1.2}{\ZR}  \\
    susy       & 18    & \ResTSN{25.8}{2.6}  & \ResTSDWW{2.0}{0.6}   & \ResTSDWB{1.9}{0.1}    & \ResTSN{25.8}{2.4}   & \ResTSDWW{0.6}{1.1}  & \ResTSDWW{0.4}{0.2}  \\
    ijcnn1     & 22    & \ResTSN{18.0}{2.7}  & \ResTSDWW{4.1}{0.8}   & \ResTSDWW{0.3}{0.2}    & \ResTSN{25.1}{3.4}   & \ResTSDBW{0.8}{0.2}  & \ResTSDBB{3.7}{0.5}  \\
    covtype.b  & 54    & \ResTSN{32.1}{3.2}  & \ResTSDWW{8.5}{0.2}   & \ResTSDWW{4.9}{0.3}    & \ResTSN{29.7}{2.5}   & \ResTSDWW{10.3}{1.2} & \ResTSDWW{5.0}{1.7}  \\
    phishing   & 68    & \ResTSN{ 9.6}{0.9}  & \ResTSDWW{0.2}{\ZR}   & \ResTSDWW{0.4}{0.2}    & \ResTSN{11.6}{2.1}   & \ResTSDBB{0.6}{0.7}  & \ResTSDBB{1.5}{0.9}  \\
    a9a        & 123   & \ResTSN{26.8}{1.6}  & \ResTSDWW{1.9}{0.2}   & \ResTSDWW{1.2}{\ZR}    & \ResTSN{27.4}{2.1}   & \ResTSDWB{1.2}{0.4}  & \ResTSDWB{0.9}{0.2}  \\
    connect4   & 126   & \ResTSN{32.9}{2.1}  & \ResTSDWW{4.3}{0.7}   & \ResTSDWW{2.7}{0.2}    & \ResTSN{34.8}{2.7}   & \ResTSDWW{3.3}{0.4}  & \ResTSDWW{0.6}{0.1}  \\
    w8a        & 300   & \ResTSN{21.6}{2.4}  & \ResTSDWB{2.1}{0.4}   & \ResTSDWB{3.0}{0.1}    & \ResTSN{16.9}{2.7}   & \ResTSDWB{5.6}{0.3}  & \ResTSDWW{5.1}{\ZR}  \\
    epsilon    & 2,000 & \ResTSN{35.0}{4.4}  & \ResTSDWB{23.7}{1.2}  & \ResTSDWB{19.6}{0.9}   & \ResTSN{31.2}{1.1}   & \ResTSDWW{21.4}{2.8} & \ResTSDWW{21.7}{4.6} \\
    \bottomrule
  \end{tabular}
}
 \end{table}

Table~\ref{tab:AltS1:Nonoverlap} compares the two-step, inductive misclassification rate when using the three step~\#1 methods for the experiments in Table~\ref{tab:App:AddRes:FullClassPartition} on MNIST, 20~Newsgroups, and CIFAR10 (see Section~\ref{sec:App:AddRes:Nonoverlap}).  For the vast majority of setups, top\K{} weighting was the best performing method for both \pupnu{} and \punu\@.  Top\K{} often improved performance over soft weighting by 10\==20\% or more -- in particular for \punu\@.  The one experimental setup where soft weighting consistently performed as well or better than top\K{} was when the positive train and test supports were disjoint. Observe that in those experiments, the positive and negative classes are composed of fewer constituent labels.  As such, we believe that top\K{} weighting is exacerbating overfitting in those models resulting in the worse performance.

\begin{table}[t]
  \centering
  \caption{Effect of step~\#1 method on our two-step methods' overall inductive misclassification rate~(\%) for the MNIST, 20~Newsgroups, and~CIFAR10 datasets.  The table's upper half reports each method's misclassification rate mean and standard deviation over 100~trials. Boldface denotes each experimental setup's best performing method according to mean misclassification rate.  The table's lower half is an alternate visualization showing the difference (Diff.) in misclassification rate mean and standard deviation w.r.t.\ to our soft method. {\color{BrickRed}Red} denotes that the associated alternate step~\#1 method had worse (i.e.,~higher) mean misclassification rate than soft weighting while {\color{ForestGreen} green} denotes that the alternate method had a better (i.e.,~lower) mean misclassification rate.}\label{tab:AltS1:Nonoverlap}
\newcommand{\DsName}[1]{\multirow{8}{*}{\rotatebox[origin=c]{90}{#1}}}

\newcommand{\BigArrowBase}[1]{\multicolumn{1}{c@{}}{\multirow{3}{*}{$\Bigg#1$}}}
\newcommand{\BigUpArrow}{\BigArrowBase{\uparrow}}
\newcommand{\BigDownArrow}{\BigArrowBase{\downarrow}}

\renewcommand{\arraystretch}{1.2}
\setlength{\dashlinedash}{0.4pt}
\setlength{\dashlinegap}{1.5pt}
\setlength{\arrayrulewidth}{0.3pt}

\newcommand{\NDefFull}[1]{\multirow{7}{*}{\shortstack[l]{#1}}}
\newcommand{\PTestDefFull}[1]{\NDefFull{#1}}

\setlength{\tabcolsep}{7.0pt}

{\centering
  \scriptsize
  \begin{tabular}{@{}llllllrrrrrr@{}}
  \toprule
  \multirow{2}{*}{} & \multirow{2}{*}{N} & \multirow{2}{*}{\pDsTest} & \multirow{2}{*}{\pDsTrain} &  \multirow{2}{*}{$\prd{\tr}$} & \multirow{2}{*}{$\prd{\te}$} &  \TSAlgHeader{\pupnu{}}  & \TSAlgHeader{\punu{}} \\\cmidrule(lr){7-9}\cmidrule(l){10-12}
                  & &   & & &  & \TSResHeaderAlt{}          & \TSResHeaderAlt{}         \\\midrule
  \DsName{MNIST} & \NDefFull{0, 2, 4,\\ 6, 8} & \PTestDefFull{1, 3, 5,\\ 7, 9} & \SamePTest{}
                  & 0.5  & 0.5  & \ResTSN{10.2}{1.5}  & \ResTSN{ 9.8}{1.3}     & \ResTSB{ 7.8}{1.0}      & \ResTSN{11.8}{1.5}   & \ResTSN{10.6}{1.1}     & \ResTSB{ 9.3}{1.0}  \\\cdashline{4-12}
    &     &    &  \multirow{3}{*}{7, 9}
                  & 0.29 & 0.5  & \ResTSN{ 5.4}{0.5}  & \ResTSN{ 5.3}{0.5}     & \ResTSB{ 4.9}{0.4}      & \ResTSN{ 6.1}{0.7}   & \ResTSB{ 5.5}{0.4}     & \ResTSN{ 5.6}{0.3}  \\\cdashline{5-12}
    &     &    &  & 0.5  & 0.5  & \ResTSN{ 6.9}{0.9}  & \ResTSN{ 7.7}{1.2}     & \ResTSB{ 5.9}{0.6}      & \ResTSN{ 8.0}{1.3}   & \ResTSN{ 7.5}{0.9}     & \ResTSB{ 6.4}{0.5}  \\\cdashline{5-12}
    &     &    &  & 0.71 & 0.5  & \ResTSN{11.0}{1.4}  & \ResTSN{12.9}{1.2}     & \ResTSB{ 9.9}{1.3}      & \ResTSN{14.9}{3.7}   & \ResTSN{12.9}{1.1}     & \ResTSB{ 9.9}{1.0}  \\\cdashline{4-12}
    &     &    & \multirow{3}{*}{1, 3, 5}
                  & 0.38 & 0.5  & \ResTSN{ 6.4}{0.8}  & \ResTSN{ 6.6}{0.8}     & \ResTSB{ 5.7}{0.6}      & \ResTSN{ 7.6}{0.9}   & \ResTSN{ 7.0}{0.7}     & \ResTSB{ 6.6}{0.6}  \\\cdashline{5-12}
    &     &    &  & 0.5  & 0.5  & \ResTSN{ 8.4}{1.1}  & \ResTSN{ 9.0}{1.0}     & \ResTSB{ 7.2}{0.9}      & \ResTSN{10.0}{1.4}   & \ResTSN{ 9.3}{0.9}     & \ResTSB{ 8.0}{0.8}  \\\cdashline{5-12}
    &     &    &  & 0.63 & 0.5  & \ResTSN{11.3}{1.4}  & \ResTSN{12.8}{1.4}     & \ResTSB{10.2}{1.4}      & \ResTSN{14.1}{2.5}   & \ResTSN{12.9}{1.2}     & \ResTSB{10.5}{1.1}  \\\cdashline{5-12}
    & {0, 2}  & {5, 7}  & {1, 3}
                  & 0.5  & 0.5  & \ResTSN{ 3.5}{1.0}  & \ResTSN{ 4.1}{1.1}     & \ResTSB{ 2.8}{0.6}      & \ResTSN{ 3.1}{0.7}   & \ResTSN{ 3.7}{0.9}     & \ResTSB{ 2.8}{0.4}  \\\midrule
  \DsName{20~Newsgroups} & \NDefFull{sci, soc,\\talk} & \PTestDefFull{alt, comp,\\ misc, rec} & \SamePTest {}
                  & 0.56 & 0.56 & \ResTSN{14.9}{1.3}  & \ResTSN{14.9}{1.4}     & \ResTSB{14.4}{1.5}      & \ResTSN{16.6}{2.5}   & \ResTSN{15.9}{1.8}     & \ResTSB{15.5}{1.9}  \\\cdashline{4-12}
    & &  &  \multirow{3}{*}{misc, rec}
                  & 0.37 & 0.56 & \ResTSN{12.8}{0.6}  & \ResTSN{12.9}{0.8}     & \ResTSB{12.4}{0.6}      & \ResTSN{14.2}{0.9}   & \ResTSN{13.7}{0.9}     & \ResTSB{13.1}{0.7}  \\\cdashline{5-12}
    & &  &        & 0.56 & 0.56 & \ResTSN{13.6}{0.9}  & \ResTSN{14.0}{0.9}     & \ResTSB{13.4}{0.9}      & \ResTSN{15.1}{1.3}   & \ResTSN{14.8}{1.1}     & \ResTSB{14.1}{1.1}  \\\cdashline{5-12}
    & &  &        & 0.65 & 0.56 & \ResTSN{14.0}{0.9}  & \ResTSN{14.4}{0.9}     & \ResTSB{13.8}{0.9}      & \ResTSN{15.8}{1.3}   & \ResTSN{15.3}{1.2}     & \ResTSB{14.6}{0.9}  \\\cdashline{4-12}
    & &  & \multirow{3}{*}{comp}
                  & 0.37 & 0.56 & \ResTSN{13.7}{0.6}  & \ResTSN{13.8}{0.6}     & \ResTSB{13.0}{0.7}      & \ResTSN{14.5}{0.8}   & \ResTSN{14.1}{0.7}     & \ResTSB{13.3}{0.7}  \\\cdashline{5-12}
    & &  &        & 0.56 & 0.56 & \ResTSN{14.9}{0.7}  & \ResTSN{15.7}{0.7}     & \ResTSB{14.3}{0.8}      & \ResTSN{15.7}{0.9}   & \ResTSN{15.9}{0.8}     & \ResTSB{14.6}{0.9}  \\\cdashline{5-12}
    & &  &        & 0.65 & 0.56 & \ResTSN{15.5}{1.1}  & \ResTSN{16.5}{1.0}     & \ResTSB{15.2}{1.2}      & \ResTSN{16.3}{1.4}   & \ResTSN{16.7}{1.3}     & \ResTSB{15.4}{1.3}  \\\cdashline{2-12}
    & {misc, rec} & {soc, talk} & {alt, comp}
                  & 0.55 & 0.46 & \ResTSB{ 7.2}{1.2}  & \ResTSN{ 8.1}{1.2}     & \ResTSN{ 7.5}{1.2}      & \ResTSB{ 5.8}{1.6}   & \ResTSN{ 7.1}{1.5}     & \ResTSB{ 5.8}{1.4}  \\\midrule
  \DsName{CIFAR10}  & \NDefFull{Bird, Cat,\\Deer, Dog,\\Frog, Horse} & \PTestDefFull{Plane, \\ Auto, Ship, \\ Truck} & \SamePTest{}
                  & 0.4  & 0.4  & \ResTSN{13.9}{1.2}  & \ResTSN{13.6}{0.9}     & \ResTSB{12.0}{0.7}      & \ResTSN{15.0}{1.2}   & \ResTSN{14.7}{0.9}     & \ResTSB{13.2}{0.8}  \\\cdashline{4-12}
    & & & \multirow{3}{*}{Plane}
                  & 0.14 & 0.4  & \ResTSB{12.0}{0.8}  & \ResTSN{12.1}{0.6}     & \ResTSN{12.2}{0.7}      & \ResTSN{12.5}{0.9}   & \ResTSN{11.8}{0.6}     & \ResTSB{11.7}{0.7}  \\\cdashline{5-12}
    & & &         & 0.4  & 0.4  & \ResTSN{14.4}{1.3}  & \ResTSN{15.4}{1.1}     & \ResTSB{14.1}{0.8}      & \ResTSN{14.9}{1.4}   & \ResTSN{15.0}{1.2}     & \ResTSB{13.2}{0.8}  \\\cdashline{5-12}
    & & &         & 0.6  & 0.4  & \ResTSB{16.7}{1.5}  & \ResTSN{20.0}{1.5}     & \ResTSN{16.9}{1.1}      & \ResTSN{20.1}{2.3}   & \ResTSN{20.0}{1.8}     & \ResTSB{15.5}{1.1}  \\\cdashline{4-12}
    & & & \multirow{3}{*}{\shortstack[l]{Auto,\\Truck}}
                  & 0.25 & 0.4  & \ResTSB{12.4}{0.7}  & \ResTSN{12.6}{0.7}     & \ResTSB{12.4}{0.7}      & \ResTSN{12.8}{0.7}   & \ResTSN{12.4}{0.7}     & \ResTSB{12.2}{0.7}  \\\cdashline{5-12}
    & & &         & 0.4  & 0.4  & \ResTSN{14.0}{1.2}  & \ResTSN{14.7}{1.1}     & \ResTSB{13.4}{0.8}      & \ResTSN{14.4}{1.2}   & \ResTSN{14.6}{1.3}     & \ResTSB{13.1}{0.8}  \\\cdashline{5-12}
    & & &         & 0.55 & 0.4  & \ResTSN{16.2}{1.6}  & \ResTSN{17.7}{1.8}     & \ResTSB{15.3}{1.0}      & \ResTSN{17.0}{2.1}   & \ResTSN{17.7}{2.1}     & \ResTSB{14.8}{1.0}  \\\cdashline{2-12}
    & {Deer, Horse} & {Plane, Auto} & {Cat, Dog}
                  & 0.5  & 0.5  & \ResTSB{15.1}{1.7}  & \ResTSN{20.2}{1.2}     & \ResTSN{19.2}{1.1}      & \ResTSB{11.2}{0.8}   & \ResTSN{16.3}{1.3}     & \ResTSN{14.2}{1.0}  \\
  \bottomrule
  \toprule
  \multirow{2}{*}{} & \multirow{2}{*}{N} & \multirow{2}{*}{\pDsTest} & \multirow{2}{*}{\pDsTrain} &  \multirow{2}{*}{$\prd{\tr}$} & \multirow{2}{*}{$\prd{\te}$} &  \TSAlgHeader{\pupnu{}}  & \TSAlgHeader{\punu{}} \\\cmidrule(lr){7-9}\cmidrule(l){10-12}
                  & &   & & &  & \TSResHeader{}          & \TSResHeader{}         \\\midrule
  \DsName{MNIST} & \NDefFull{0, 2, 4,\\ 6, 8} & \PTestDefFull{1, 3, 5,\\ 7, 9} & \SamePTest{}
                  & 0.5  & 0.5  & \ResTSN{10.2}{1.5}  & \ResTSDBB{0.4}{0.2}    & \ResTSDBB{2.5}{0.5}     & \ResTSN{11.8}{1.5}   & \ResTSDBB{1.3}{0.4}    & \ResTSDBB{2.6}{0.5}   \\\cdashline{4-12}
    &     &    &  \multirow{3}{*}{7, 9}
                  & 0.29 & 0.5  & \ResTSN{ 5.4}{0.5}  & \ResTSDBW{0.1}{\ZR}    & \ResTSDBB{0.5}{0.1}     & \ResTSN{ 6.1}{0.7}   & \ResTSDBB{0.5}{0.3}    & \ResTSDBB{0.5}{0.3}   \\\cdashline{5-12}
    &     &    &  & 0.5  & 0.5  & \ResTSN{ 6.9}{0.9}  & \ResTSDWW{0.8}{0.3}    & \ResTSDBB{1.0}{0.3}     & \ResTSN{ 8.0}{1.3}   & \ResTSDBB{0.5}{0.4}    & \ResTSDBB{1.6}{0.8}   \\\cdashline{5-12}
    &     &    &  & 0.71 & 0.5  & \ResTSN{11.0}{1.4}  & \ResTSDWB{1.9}{0.2}    & \ResTSDBB{1.1}{0.1}     & \ResTSN{14.9}{3.7}   & \ResTSDBB{2.0}{2.7}    & \ResTSDBB{5.0}{2.7}   \\\cdashline{4-12}
    &     &    & \multirow{3}{*}{1, 3, 5}
                  & 0.38 & 0.5  & \ResTSN{ 6.4}{0.8}  & \ResTSDWW{0.2}{\ZR}    & \ResTSDBB{0.7}{0.2}     & \ResTSN{ 7.6}{0.9}   & \ResTSDBB{0.6}{0.2}    & \ResTSDBB{1.0}{0.3}   \\\cdashline{5-12}
    &     &    &  & 0.5  & 0.5  & \ResTSN{ 8.4}{1.1}  & \ResTSDWW{0.6}{\ZR}    & \ResTSDBB{1.2}{0.2}     & \ResTSN{10.0}{1.4}   & \ResTSDBB{0.8}{0.5}    & \ResTSDBB{2.1}{0.7}   \\\cdashline{5-12}
    &     &    &  & 0.63 & 0.5  & \ResTSN{11.3}{1.4}  & \ResTSDWW{1.5}{\ZR}    & \ResTSDBW{1.1}{\ZR}     & \ResTSN{14.1}{2.5}   & \ResTSDBB{1.2}{1.3}    & \ResTSDBB{2.1}{0.7}   \\\cdashline{2-12}
    & {0, 2}  & {5, 7}  & {1, 3}
                  & 0.5  & 0.5  & \ResTSN{ 3.5}{1.0}  & \ResTSDWW{0.6}{0.1}    & \ResTSDBB{0.7}{0.4}     & \ResTSN{ 3.1}{0.7}   & \ResTSDWW{0.6}{0.3}    & \ResTSDBB{0.3}{0.3}   \\\midrule
  \DsName{20~Newsgroups} & \NDefFull{sci, soc,\\talk} & \PTestDefFull{alt, comp,\\ misc, rec} & \SamePTest {}
                  & 0.56 & 0.56 & \ResTSN{14.9}{1.3}  & \ResTSDSW{0.2}         & \ResTSDBW{0.5}{0.3}     & \ResTSN{16.6}{2.5}   & \ResTSDBB{0.7}{0.8}    & \ResTSDBB{1.1}{0.6}   \\\cdashline{4-12}
    & &  &  \multirow{3}{*}{misc, rec}
                  & 0.37 & 0.56 & \ResTSN{12.8}{0.6}  & \ResTSDWW{0.1}{0.1}    & \ResTSDBW{0.4}{\ZR}     & \ResTSN{14.2}{0.9}   & \ResTSDBW{0.5}{\ZR}    & \ResTSDBB{1.1}{0.2}   \\\cdashline{5-12}
    & &  &        & 0.56 & 0.56 & \ResTSN{13.6}{0.9}  & \ResTSDWW{0.3}{\ZR}    & \ResTSDBW{0.2}{\ZR}     & \ResTSN{15.1}{1.3}   & \ResTSDBB{0.3}{0.2}    & \ResTSDBB{1.0}{0.2}   \\\cdashline{5-12}
    & &  &        & 0.65 & 0.56 & \ResTSN{14.0}{0.9}  & \ResTSDWW{0.4}{\ZR}    & \ResTSDBB{0.2}{0.1}     & \ResTSN{15.8}{1.3}   & \ResTSDBB{0.5}{0.1}    & \ResTSDBB{1.2}{0.4}   \\\cdashline{4-12}
    & &  & \multirow{3}{*}{comp}
                  & 0.37 & 0.56 & \ResTSN{13.7}{0.6}  & \ResTSDWW{0.1}{\ZR}    & \ResTSDBW{0.6}{\ZR}     & \ResTSN{14.5}{0.8}   & \ResTSDBW{0.4}{\ZR}    & \ResTSDBW{1.2}{\ZR}   \\\cdashline{5-12}
    & &  &        & 0.56 & 0.56 & \ResTSN{14.9}{0.7}  & \ResTSDWW{0.8}{\ZR}    & \ResTSDBW{0.6}{\ZR}     & \ResTSN{15.7}{0.9}   & \ResTSDWB{0.2}{0.1}    & \ResTSDBB{0.1}{\ZR}   \\\cdashline{5-12}
    & &  &        & 0.65 & 0.56 & \ResTSN{15.5}{1.1}  & \ResTSDWW{1.0}{\ZR}    & \ResTSDBW{0.4}{0.1}     & \ResTSN{16.3}{1.4}   & \ResTSDWB{0.4}{0.1}    & \ResTSDBB{0.9}{0.1}   \\\cdashline{2-12}
    & {misc, rec} & {soc, talk} & {alt, comp}
                  & 0.55 & 0.46 & \ResTSN{ 7.2}{1.2}  & \ResTSDWW{1.0}{0.1}    & \ResTSDWW{0.3}{\ZR}     & \ResTSN{ 5.8}{1.6}   & \ResTSDWB{1.4}{0.1}    & \ResTSDSB{0.3}        \\\midrule
  \DsName{CIFAR10}  & \NDefFull{Bird, Cat,\\Deer, Dog,\\Frog, Horse} & \PTestDefFull{Plane, \\ Auto, Ship, \\ Truck} & \SamePTest{}
                  & 0.4  & 0.4  & \ResTSN{13.9}{1.2}  & \ResTSDBB{0.3}{0.3}    & \ResTSDBB{1.9}{0.5}     & \ResTSN{15.0}{1.2}   & \ResTSDBB{0.4}{0.3}    & \ResTSDBB{1.9}{0.4}   \\\cdashline{4-12}
    & & & \multirow{3}{*}{Plane}
                  & 0.14 & 0.4  & \ResTSN{12.0}{0.8}  & \ResTSDWB{0.1}{0.1}    & \ResTSDWB{0.1}{0.1}     & \ResTSN{12.5}{0.9}   & \ResTSDBB{0.8}{0.3}    & \ResTSDBB{0.8}{0.2}   \\\cdashline{5-12}
    & & &         & 0.4  & 0.4  & \ResTSN{14.4}{1.3}  & \ResTSDWB{1.0}{0.2}    & \ResTSDBB{0.3}{0.5}     & \ResTSN{14.9}{1.4}   & \ResTSDWB{0.1}{0.2}    & \ResTSDBB{1.8}{0.6}   \\\cdashline{5-12}
    & & &         & 0.6  & 0.4  & \ResTSN{16.7}{1.5}  & \ResTSDWW{3.3}{\ZR}    & \ResTSDWB{0.2}{0.4}     & \ResTSN{20.1}{2.3}   & \ResTSDBB{0.1}{0.5}    & \ResTSDBB{4.6}{1.2}   \\\cdashline{4-12}
    & & & \multirow{3}{*}{\shortstack[l]{Auto,\\Truck}}
                  & 0.25 & 0.4  & \ResTSN{12.4}{0.7}  & \ResTSDWW{0.2}{\ZR}    & \ResTSDSW{\ZR}          & \ResTSN{12.8}{0.7}   & \ResTSDBB{0.4}{0.1}    & \ResTSDBW{0.6}{\ZR}   \\\cdashline{5-12}
    & & &         & 0.4  & 0.4  & \ResTSN{14.0}{1.2}  & \ResTSDWB{0.6}{0.1}    & \ResTSDBB{0.6}{0.4}     & \ResTSN{14.4}{1.2}   & \ResTSDWW{0.3}{0.1}    & \ResTSDBB{1.3}{0.4}   \\\cdashline{5-12}
    & & &         & 0.55 & 0.4  & \ResTSN{16.2}{1.6}  & \ResTSDWW{1.5}{0.2}    & \ResTSDBB{0.9}{0.7}     & \ResTSN{17.0}{2.1}   & \ResTSDWW{0.7}{\ZR}    & \ResTSDBB{2.2}{1.1}   \\\cdashline{2-12}
    & {Deer, Horse} & {Plane, Auto} & {Cat, Dog}
                  & 0.5  & 0.5  & \ResTSN{15.1}{1.7}  & \ResTSDWB{5.1}{0.5}    & \ResTSDWB{4.1}{0.6}     & \ResTSN{11.2}{0.8}   & \ResTSDWW{5.1}{0.6}    & \ResTSDWW{3.0}{0.3}   \\
  \bottomrule
\end{tabular}
}
 \end{table}

\subsubsection{Discussion}

The experiments in the previous subsection demonstrate that the best performing step~\#1 method is benchmark/setup dependent.  If a user is highly confident that their data is readily and easily separable (like MNIST), top\K{} weighting may perform particularly well.  Although not shown here, we empirically observed that misestimation of training prior~$\prd{\tr}$ negatively affects top\K{} weighting's accuracy -- many times severely.

If the training datasets (e.g.,~$\pset{}$, $\uset{\tr}$, and~$\uset{\te}$) are large enough that asymptotic consistency guarantees generally apply, soft weighting may perform best. We made soft-weighting the focus of Section~\ref{sec:WUU} due to its stronger statistical guarantees.  Had top\K{} weighting been used in Section~\ref{sec:ExpRes:NonidenticalSupports}'s experiments instead of soft weighting, our performance advantage over the baselines, \PUc{} and~\nnpuOpt{}, would have widened.
 
\suppressfloats
\clearpage
\subsection{Analyzing the Effect of Positive and Negative Class-Conditional Distribution Shift}\label{sec:App:AddRes:BiasSweep}

The goal of these experiments is to:
\begin{enumerate}
  \setlength{\itemsep}{0pt}
  \item Demonstrate the effectiveness of our approaches across the entire spectrum of positive-train class\=/conditional distribution shift.
  \item Study how our methods perform when the assumption of a fixed negative class\=/conditional distribution is violated.
\end{enumerate}

We look at these trends across three datasets (as in Section~\ref{sec:ExpRes:NonidenticalSupports}): MNIST, 20~Newsgroups, and CIFAR10.  The positive and negatives classes are formed by combining two labels from the original dataset (the use of two labels per class is necessary for this experimental setup). Table~\ref{tab:App:AddRes:ShiftClassDefinition} enumerates each dataset's positive and negative class definitions; these definitions apply for both train and test.  The dataset sizes are listed in Table~\ref{tab:App:AddRes:ShiftDatasetSizes}; note that $\nTest$~is the size of the inductive test set used to measure performance.  The validation set was one-fifth the training set size.  The priors were also fixed such that ${\prd{\tr} = \prd{\te} = 0.5}$.

\begin{table}[h]
  \centering
  \caption{Positive and negative class definitions for the class-conditional bias experiments}\label{tab:App:AddRes:ShiftClassDefinition}
  \begin{tabular}{@{}lllll@{}}
    \toprule
    \multirow{2}{*}{Dataset}       & \multicolumn{2}{c}{Positive} & \multicolumn{2}{c}{Negative} \\\cmidrule(lr){2-3}\cmidrule(l){4-5}
                  & $C_1$   & $C_2$    & $C_1$   & $C_2$ \\\midrule
    MNIST         & 8       & 9        & 3       & 4        \\
    20~Newsgroups & sci     & rec      & comp    & talk     \\
    CIFAR10       & Auto    & Plane    & Ship    & Truck    \\
    \bottomrule
  \end{tabular}
\end{table}

\begin{table}[h]
  \centering
  \caption{Dataset sizes for the class-conditional bias experiments}\label{tab:App:AddRes:ShiftDatasetSizes}
  \begin{tabular}{@{}lrrrr@{}}
    \toprule
    Dataset       & $\np$   & $\nTrU$  & $\nTeU$ & $\nTest$ \\\midrule
    MNIST         & 250     & 5,000    & 5,000   & 1,500   \\
    20~Newsgroups & 500     & 2,000    & 2,000   & 1,000   \\
    CIFAR10       & 500     & 5,000    & 5,000   & 1,500   \\
    \bottomrule
  \end{tabular}
\end{table}

The default rule in this section is that the positive/negative train/test classes are selected uniformly at random without replacement from their respective subclasses.  In each experiment, either the positive-train or negative-train class\=/conditional distribution is shifted (never both).  The test distribution is never biased and is identical for all experiments.

\paragraph{Positive-Train Shift} In these experiments, the positive-train class\=/conditional distribution~(i.e.,~$\plike{\tr}$) is shifted.  Recall that each positive class is composed of two labels; denote them~$C_1$ and~$C_2$ (e.g.,~${C_1 = \text{Auto}}$ and ${C_2 = \text{Plane}}$ for~CIFAR10).  ${\Pr\sbrack{\text{Label}_{\tr}{ = }C_1 \vert \yrv = \yp}}$ is the probability that any positive-valued \textit{training} example has original label~$C_1$.  Since there are two labels per class,
\begin{equation}\label{eq:}
  \Pr\sbrack{\text{Label}_{\tr}{ =}C_2 \vert \yrv = \yp} = 1- \Pr\sbrack{\text{Label}_{\tr}{ = }C_1 \vert \yrv = \yp}\text{.}
\end{equation}
\noindent
The positive-train class\=/conditional distribution shift entails sweeping ${\Pr\sbrack{\text{Label}_{\tr}{ = }C_1 \vert \yrv = \yp}}$ from~0.5 to~1 (i.e.,~from unbiased on the left to maximally biased on the right).  This setup is more challenging than shifting the positive-test distribution since it entails the learner seeing fewer \textit{labeled} examples from positive subclass~$C_2$.

Figures~\ref{fig:App:AddRes:MNIST:PosClassConditional},~\ref{fig:App:AddRes:20NG:PosClassConditional}, and~\ref{fig:App:AddRes:CIFAR:PosClassConditional} show the positive-train shift's effect on the MNIST, 20~Newsgroups, and~CIFAR10 misclassification rate respectively (where~$C_1$ corresponds to digit~8, document category~``rec'', and image type~``automobile'').  \nnpurpl's performance was consistent across the entire bias range while the two step methods'~(\punu\ and~\pupnu) performance improved as bias increased (due to easier identification of negative examples as explained in Section~\ref{sec:ExpRes:NonidenticalSupports}).  In contrast, \PUc's performance degrades as bias increases; this degradation is largely due to poor density estimation and demonstrates why covariate shift methods can be non\=/ideal.

\PNtr{}~and \PNte{}~are trained using (labeled)~$\uset{\tr}$ and~$\uset{\te}$. Since the test distributions are never biased, \PNte~is unaffected by shift. In contrast, as ${\Pr\sbrack{\text{Label}_{\tr}{ = }C_1 \vert \yrv = \yp}}$ increases, there are fewer examples in~$\uset{\tr}$ with label~$C_2$ causing a degradation in~\PNtr{}'s performance.

\PUc's and \nnpuOpt's performance begins to degrade at the same point where \PNtr's~and \PNte's~performance begins to diverge.  For \nnpuOpt{} in particular, this degradation is primarily attributable to fewer examples labeled~$C_2$ in~$\pset{}$. \PUc{}~is more robust to bias than \nnpuOpt{} (as shown by the slower rate of degradation) since it considers distributional shifts.

\paragraph{Negative-Train Shift} These experiments follow the same basic concept as the positive-train class\=/conditional distribution shift described above except that the bias is instead applied to the negative-train class\=/conditional distribution, i.e.,~$\nlike{\tr}$.  This bias means that ${\nlike{\tr} \ne \nlike{\te}}$.  To reiterate, \textit{these experiments deliberately violate Eq.~\eqref{eq:PreviousWork:NegAssume}'s~assumption} upon which our methods are predicated.  The goal here is to understand our methods' robustness under intentionally deleterious conditions.  It is more deleterious to bias the negative class in~$\uset{\tr}$ since both two-step methods and \nnpurpl{} use $\uset{\tr}$'s~negative risk in dependent calculations; any error propagates and compounds in these subsequent operations.

Let~$C_1$ and~$C_2$ now be the two labels that make up the negative class (e.g.,~${C_1 = \text{Ship}}$ and ${C_2 = \text{Truck}}$ for~CIFAR10).  Now, ${\Pr\sbrack{\text{Label}_{\tr}\text{ = }C_1 \vert \yrv = \yn}}$ is swept along the x\=/axis from~0.5 to~1 (unbiased to maximally biased).  The results for MNIST, 20~Newsgroups, and CIFAR10 are in Figures~\ref{fig:App:AddRes:MNIST:NegClassConditional},~\ref{fig:App:AddRes:20NG:NegClassConditional}, and~\ref{fig:App:AddRes:CIFAR:NegClassConditional} respectively.

With the exception of \punu~on MNIST, all of our methods showed moderate robustness to some negative class\=/conditional distribution bias.  In particular, \pupnu\ was almost as robust as \PUc\ in some cases.  \nnpuOpt's robustness here is expected since anything not in $\pset{}$~is assumed negative; even under bias, sufficient negative examples exist for each label in~$\uset{\te}$ to allow \nnpuOpt\ to learn how to classify those examples.

\begin{figure*}[h]
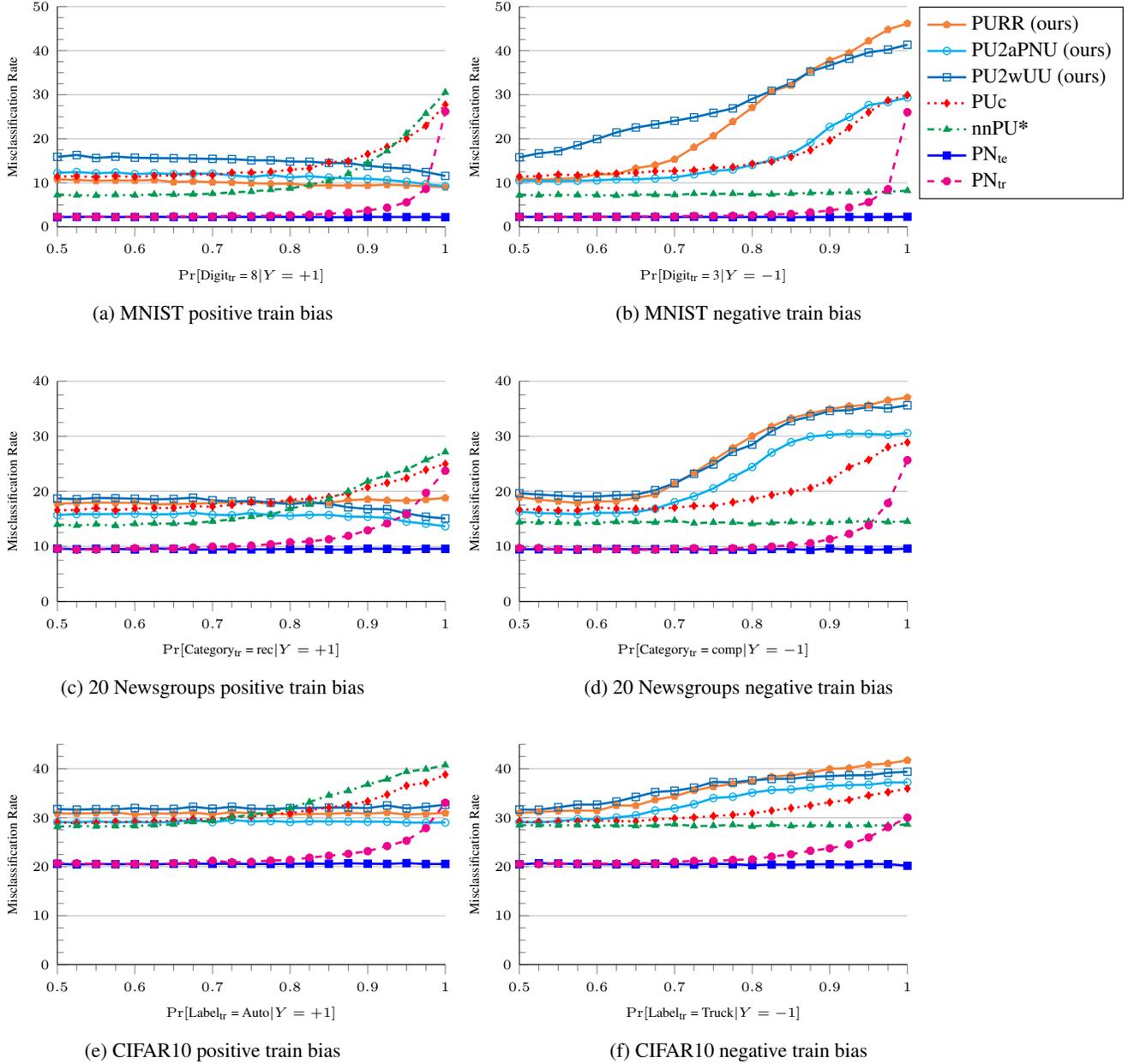

  \centering
  \hspace{-1.0cm}
  \begin{subfigure}[b]{0.40\textwidth}
    \shiftPlot{mnist_positive}{Digit\textsubscript{\tr} = 8}{+1}{50}{\ShiftLegendOff}
    \caption{MNIST positive train bias}\label{fig:App:AddRes:MNIST:PosClassConditional}
  \end{subfigure}
  {}~\hspace{0.3cm}
  \begin{subfigure}[b]{0.52\textwidth}
    \shiftPlot{mnist_negative}{Digit\textsubscript{\tr} = 3}{-1}{50}{}
    \caption{MNIST negative train bias}\label{fig:App:AddRes:MNIST:NegClassConditional}
  \end{subfigure}

  \vspace{18pt}
  \hspace{-1.0cm}
  \begin{subfigure}[t]{0.40\textwidth}
    \shiftPlot{newsgroups_positive}{Category\textsubscript{\tr} = rec}{+1}{40}{\ShiftLegendOff}%
    \caption{20~Newsgroups positive train bias}\label{fig:App:AddRes:20NG:PosClassConditional}%
  \end{subfigure}
  {}~\hspace{0.3cm}
  \begin{subfigure}[t]{0.52\textwidth}
    \shiftPlot{newsgroups_negative}{Category\textsubscript{\tr} = comp}{-1}{40}{\ShiftLegendOff}
    \caption{20~Newsgroups negative train bias}\label{fig:App:AddRes:20NG:NegClassConditional}%
  \end{subfigure}

  \vspace{18pt}
  \hspace{-1.0cm}
  \begin{subfigure}[b]{0.40\textwidth}
    \shiftPlot{cifar_positive}{Label\textsubscript{\tr} = Auto}{+1}{45}{\ShiftLegendOff}
    \caption{CIFAR10 positive train bias}\label{fig:App:AddRes:CIFAR:PosClassConditional}
  \end{subfigure}
  {}~\hspace{0.3cm}
  \begin{subfigure}[b]{0.52\textwidth}
    \shiftPlot{cifar_negative}{Label\textsubscript{\tr} = Truck}{-1}{45}{\ShiftLegendOff}
    \caption{CIFAR10 negative train bias}\label{fig:App:AddRes:CIFAR:NegClassConditional}
  \end{subfigure}
  \caption{Effect of positive~($\plikeTR$) or negative~($\nlikeTR$) training class\=/conditional distribution shift on inductive misclassification rate~(\%) for the MNIST, 20~Newsgroups, and CIFAR10 datasets.  The x\=/axis corresponds to~${\Pr\sbrack{\text{Label}_{\tr}\text{ = }C_1 \vert y = \yhat}}$ where~${\yhat \in \set{\pm1}}$. Each data point is the average of 100~trials.}\label{fig:App:AddRes:ClassConditionalShift}
\end{figure*}
\suppressfloats
\clearpage
\newpage

\subsection{Effect of Prior Probability Misestimation}\label{sec:App:AddRes:PriorEstimationError}

As explained in Section~\ref{sec:APULearning}, this work treats the positive-class priors, $\prd{\tr}$~and~$\prd{\te}$, as known.  This set of experiments examines our methods' performance when the priors are misspecified.

\paragraph{Experimental Setup}  These experiments reuse the partially disjoint positive-support experiment setups from Section~\ref{sec:ExpRes:NonidenticalSupports}'s Table~\ref{tab:ExpRes:Nonoverlap}.  Therefore, we are specifically considering the MNIST, 20~Newsgroups, and~CIFAR10 datasets with Table~\ref{tab:App:PriorShift:ExpSetup} summarizing the experimental setups.

$\prd{\tr}$ and~$\prd{\te}$ in Table~\ref{tab:App:PriorShift:ExpSetup} are the \textit{actual} prior probabilities used to construct each training and test data set. We tested our methods' performance when each prior was specified correctly and when each prior was misspecified by~${\pm20\%}$ for a total of ${9 = 3 \times 3}$ conditions per learner.  \PUc~estimates~$\prd{\te}$ as part of its density-ratio estimation.  As such, we only report three bias conditions for~\PUc, all over training prior~$\prd{\tr}$.  Like all previous experiments, performance was evaluated using the inductive (test) misclassification rate, and all methods saw identical datasets splits in each trial.

\paragraph{Analysis} Tables~\ref{tab:App:PriorShift:MNIST},~\ref{tab:App:PriorShift:20NG}, and~\ref{tab:App:PriorShift:CIFAR10} contain the results for MNIST, 20~Newsgroups, and CIFAR10 respectively.  Each learner's results are presented in a ${3 \times 3}$~grid with~$\prd{\tr}$ changing row by row and $\prd{\te}$~changing column by column.  Each cell is shaded red, with a darker background denoting worse performance, i.e.,~a greater misclassification rate.  In all but one setup, our methods outperformed~\PUc.

Similar to Section~\ref{sec:App:AddRes:BiasSweep}'s experiments, the MNIST results were most affected by bias. The 20~Newsgroups and CIFAR10 results were more immune due to the richer feature representations generated through transfer learning.

Of our three methods, \pupnu{}~was the least affected by misspecified priors.  For the two-step methods, the worst performing misestimation profile was dataset specific.  In contrast, \nnpurpl{}'s performance was always worst when the train and test priors were misspecified in opposite directions.  To understand why this is, recall that \nnpurpl{}'s definition in Eq.~\eqref{eq:PURPL:Complete} includes prior ratio~\eqsmall{$\frac{1 -\prd{\te}}{1 - \prd{\tr}}$}. This ratio compounds prior misestimations with opposite signs.

As an example, consider the MNIST experiment below with true priors ${\prd{\tr} = \prd{\te} = 0.5}$, making \nnpurpl{}'s ideal prior ratio~\eqsmall{${\frac{1-0.5}{1-0.5} = 1}$}.  This ratio remains~1 even if the priors are misspecified as ${\prd{\tr} = \prd{\te} = 0.6}$ or ${\prd{\tr} = \prd{\te} = 0.4}$.  In contrast, if ${\prd{\tr} = 0.4}$ and ${\prd{\te} = 0.6}$, \nnpurpl{}'s (erroneous) prior ratio is~\eqsmall{${\frac{1 - 0.6}{1-0.4} \approx 0.67}$} --~a 33\% error.  Furthermore, when the priors are misspecified as ${\prd{\tr} = 0.6}$ and ${\prd{\te} = 0.4}$, the prior ratio jumps to~\eqsmall{${\frac{1 - 0.4}{1 - 0.6} = 1.5}$} -- a 50\%~error.  This is why over-estimation of the training prior and underestimation of the test prior is always \nnpurpl{}'s worst performing configuration.

\begin{table}[ht]
  \centering
  \caption{Positive train~(\pDsTrain), positive test~(\pDsTest), and negative~(N) class definitions and actual prior probabilities for the experiments examining the effect of misspecified prior(s) on our algorithms' performance.}\label{tab:App:PriorShift:ExpSetup}
\newcommand{\VertCentered}[1]{\begin{tabular}{l} #1 \end{tabular}}
\newcommand{\PriorDsName}[1]{\VertCentered{#1}}
\newcommand{\PriorNDef}[1]{\makecell[l]{#1}}
\newcommand{\PriorPDef}[1]{\PriorNDef{#1}}

\begin{tabular}{@{}llllll@{}}
  \toprule
   & N   &   \pDsTrain{}     & \pDsTest{}     & $\prd{\tr}$  & $\prd{\te}$ \\\midrule
  \PriorDsName{MNIST}    & \PriorNDef{0, 2, 4,\\ 6, 8} & \PriorPDef{1, 3, 5,\\ 7, 9}        & 7, 9      & 0.5  & 0.5  \\\hdashline
  \PriorDsName{20~News.} & \PriorNDef{sci, soc,\\talk} & \PriorPDef{alt, comp,\\ misc, rec} & \PriorPDef{misc,\\ rec} & 0.37 & 0.56 \\\hdashline
  \PriorDsName{CIFAR10}  & \PriorNDef{Bird, Cat,\\Deer, Dog,\\Frog, Horse} & \PriorPDef{Plane, \\ Auto, Ship, \\ Truck} & Plane & 0.4  & 0.4   \\
  \bottomrule
\end{tabular}
 \end{table}

\clearpage
\begin{table}[ht]
  \centering
  \caption{\PriorTableCaptionText{MNIST}}\label{tab:App:PriorShift:MNIST}
\begin{tabular}{ r r R R R r R R R r R R R r R}%
  \toprule%
  \multicolumn{1}{c}{} &
  \multicolumn{1}{c}{} &%
  \multicolumn{3}{c}{\nnpurpl} &%
  \multicolumn{1}{c}{} &%
  \multicolumn{3}{c}{\pupnu} &%
  \multicolumn{1}{c}{} &%
  \multicolumn{3}{c}{\punu} &%
  \multicolumn{1}{c}{} &%
  \multicolumn{1}{c}{\multirow{2}{*}{\PUc}} \\\cmidrule(lr){3-5}\cmidrule(lr){7-9}\cmidrule(lr){11-13}%
  \multicolumn{1}{r} {} &%
  \multicolumn{1}{c}{} &%
  \multicolumn{1}{r} {$0.8\prd{\te}$} &%
  \multicolumn{1}{r} {$\prd{\te}$} &%
  \multicolumn{1}{r} {$1.2\prd{\te}$} &%
  \multicolumn{1}{c}{} &%
  \multicolumn{1}{r} {$0.8\prd{\te}$} &%
  \multicolumn{1}{r} {$\prd{\te}$} &%
  \multicolumn{1}{r} {$1.2\prd{\te}$} &%
  \multicolumn{1}{c}{} &%
  \multicolumn{1}{r} {$0.8\prd{\te}$} &%
  \multicolumn{1}{r} {$\prd{\te}$} &%
  \multicolumn{1}{r} {$1.2\prd{\te}$} &%
  \multicolumn{1}{c}{} &%
  \multicolumn{1}{c} {} \\\midrule%
  $0.8\prd{\tr}$   &    & 17.4   & 16.6   & 19.8   &    & 7.4   & 9.5   & 12.2   &    & 10.3   & 13.2   & 18.7   &    & 29.6 \\
$\prd{\tr}$   &    & 12.9   & 9.2   & 13.6   &    & 6.6   & 7.4   & 10.1   &    & 8.5   & 10.3   & 13.9   &    & 26.7 \\
$1.2\prd{\tr}$   &    & 25.3   & 15.8   & 12.7   &    & 18.0   & 7.7   & 7.5   &    & 16.9   & 8.9   & 10.3   &    & 26.3 \\
  \toprule%
\end{tabular}%

\end{table}

\begin{table}[ht]
  \centering
  \caption{\PriorTableCaptionText{20~Newsgroups}}\label{tab:App:PriorShift:20NG}
\begin{tabular}{ r r R R R r R R R r R R R r R}%
  \toprule%
  \multicolumn{1}{c}{} &
  \multicolumn{1}{c}{} &%
  \multicolumn{3}{c}{\nnpurpl} &%
  \multicolumn{1}{c}{} &%
  \multicolumn{3}{c}{\pupnu} &%
  \multicolumn{1}{c}{} &%
  \multicolumn{3}{c}{\punu} &%
  \multicolumn{1}{c}{} &%
  \multicolumn{1}{c}{\multirow{2}{*}{\PUc}} \\\cmidrule(lr){3-5}\cmidrule(lr){7-9}\cmidrule(lr){11-13}%
  \multicolumn{1}{r} {} &%
  \multicolumn{1}{c}{} &%
  \multicolumn{1}{r} {$0.8\prd{\te}$} &%
  \multicolumn{1}{r} {$\prd{\te}$} &%
  \multicolumn{1}{r} {$1.2\prd{\te}$} &%
  \multicolumn{1}{c}{} &%
  \multicolumn{1}{r} {$0.8\prd{\te}$} &%
  \multicolumn{1}{r} {$\prd{\te}$} &%
  \multicolumn{1}{r} {$1.2\prd{\te}$} &%
  \multicolumn{1}{c}{} &%
  \multicolumn{1}{r} {$0.8\prd{\te}$} &%
  \multicolumn{1}{r} {$\prd{\te}$} &%
  \multicolumn{1}{r} {$1.2\prd{\te}$} &%
  \multicolumn{1}{c}{} &%
  \multicolumn{1}{c} {} \\\midrule%
  $0.8\prd{\tr}$   &    & 15.2   & 14.9   & 16.5   &    & 12.3   & 12.5   & 13.3   &    & 13.4   & 14.3   & 16.7   &    & 34.1 \\
$\prd{\tr}$      &    & 15.9   & 13.8   & 15.1   &    & 12.6   & 12.8   & 13.5   &    & 13.4   & 14.2   & 16.1   &    & 28.6 \\
$1.2\prd{\tr}$   &    & 18.7   & 13.3   & 14.0   &    & 13.4   & 14.2   & 15.0   &    & 14.4   & 15.6   & 17.2   &    & 24.9 \\

  \toprule%
\end{tabular}%

\end{table}

\begin{table}[h]
  \centering
  \caption{\PriorTableCaptionText{CIFAR10}}\label{tab:App:PriorShift:CIFAR10}
\begin{tabular}{ r r R R R r R R R r R R R r R}%
  \toprule%
  \multicolumn{1}{c}{} &
  \multicolumn{1}{c}{} &%
  \multicolumn{3}{c}{\nnpurpl} &%
  \multicolumn{1}{c}{} &%
  \multicolumn{3}{c}{\pupnu} &%
  \multicolumn{1}{c}{} &%
  \multicolumn{3}{c}{\punu} &%
  \multicolumn{1}{c}{} &%
  \multicolumn{1}{c}{\multirow{2}{*}{\PUc}} \\\cmidrule(lr){3-5}\cmidrule(lr){7-9}\cmidrule(lr){11-13}%
  \multicolumn{1}{r} {} &%
  \multicolumn{1}{c}{} &%
  \multicolumn{1}{r} {$0.8\prd{\te}$} &%
  \multicolumn{1}{r} {$\prd{\te}$} &%
  \multicolumn{1}{r} {$1.2\prd{\te}$} &%
  \multicolumn{1}{c}{} &%
  \multicolumn{1}{r} {$0.8\prd{\te}$} &%
  \multicolumn{1}{r} {$\prd{\te}$} &%
  \multicolumn{1}{r} {$1.2\prd{\te}$} &%
  \multicolumn{1}{c}{} &%
  \multicolumn{1}{r} {$0.8\prd{\te}$} &%
  \multicolumn{1}{r} {$\prd{\te}$} &%
  \multicolumn{1}{r} {$1.2\prd{\te}$} &%
  \multicolumn{1}{c}{} &%
  \multicolumn{1}{c} {} \\\midrule%
  $0.8\prd{\tr}$   &    & 16.4   & 15.6   & 18.0   &    & 13.9   & 15.1   & 17.5   &    & 18.2   & 19.3   & 21.1   &    & 23.4 \\
$\prd{\tr}$      &    & 16.0   & 13.7   & 15.3   &    & 13.3   & 14.4   & 16.6   &    & 14.2   & 14.9   & 16.9   &    & 20.1 \\
$1.2\prd{\tr}$   &    & 20.8   & 15.7   & 14.7   &    & 16.8   & 16.4   & 17.5   &    & 15.6   & 15.9   & 17.9   &    & 19.7 \\

  \toprule%
\end{tabular}%

\end{table}
   \end{document}